\documentclass{article}

    \PassOptionsToPackage{numbers, compress,square}{natbib}

     \usepackage[final]{neurips_2019}

\usepackage[utf8]{inputenc} %
\usepackage[T1]{fontenc}    %
\usepackage[colorlinks,citecolor=blue,urlcolor=black,linkcolor=blue,pdfborder={0 0 0}]{hyperref}
\usepackage{url}            %
\usepackage{booktabs}       %
\usepackage{amsfonts}       %
\usepackage{nicefrac}       %
\usepackage{microtype}      %

\usepackage{amsmath}
\usepackage{amssymb}
\usepackage{amsthm}
\usepackage{xcolor}
\usepackage{enumitem}
\usepackage{graphicx}
\usepackage{algpseudocode}
\usepackage{algorithm}
\usepackage{scalerel}
\usepackage{mathtools}
\mathtoolsset{showonlyrefs}  %

\usepackage{tabularx,relsize}
\usepackage{tablefootnote}
\usepackage{ragged2e}  %
\usepackage{todonotes}
\usepackage{booktabs}
\newcolumntype{Y}{>{\RaggedRight\arraybackslash}X}
\algrenewcommand\textproc{}

\newtheorem{defi}{Definition}
\newtheorem*{defi*}{Definition}
\numberwithin{defi}{section}

\newtheorem{lemm}{Lemma}
\newtheorem{theo}[lemm]{Theorem}
\newtheorem{coro}[lemm]{Corollary}

\theoremstyle{definition}
\newtheorem*{remark}{\textbf{Remark}}

\theoremstyle{definition}
\newtheorem{exam}{\textbf{Example}}

\newcommand{\E}{\mathbb{E}}

\newcommand{\R}{\mathbb{R}}

\allowdisplaybreaks

\newcommand*{\A}{\mathcal{A}}
\newcommand*{\B}{\mathcal{B}}
\newcommand*{\F}{\mathcal{F}}

\newcommand*{\N}{\mathcal{N}}

\newcommand*{\LL}{\left}
\newcommand*{\RR}{\right}

\newcommand*{\Exp}[1]{\mathbb{E} \left[#1\right]}

\newcommand*{\bracks}[1]{\left(#1\right)}  %
\newcommand*{\abracks}[1]{\left\langle#1\right\rangle}  %
\newcommand*{\sbracks}[1]{\left[#1\right]}  %

\newcommand{\dee}{\mathop{\mathrm{d}\!}}
\newcommand{\dt}{\,\dee t}

\newcommand{\ds}{\,\dee s}

\newcommand{\dX}{\,\dee X}

\newcommand{\dV}{\,\dee V}
\newcommand{\du}{\,\dee u}

\newcommand{\dB}{\,\dee B} %
\newcommand{\dtau}{\,\dee \tau}

\newcommand{\dzeta}{\,\dee \zeta}

\newcommand{\tX}{\tilde{X}}
\newcommand{\tH}{\tilde{H}}
\newcommand{\tV}{\tilde{V}}
\newcommand{\tY}{\tilde{Y}}
\newcommand{\tZ}{\tilde{Z}}

\newcommand{\barX}{\bar{X}}

\newcommand{\barR}{\Bar{R}}

\newcommand{\lyap}{\mathcal{U}}
\newcommand{\uyap}{\mathcal{V}}
\newcommand{\wyap}{\mathcal{W}}

\newcommand{\norm}[1]{\left\lVert#1\right\rVert}
\newcommand{\normf}[1]{\left\lVert#1\right\rVert_{\mathrm{F}}}
\newcommand{\normop}[1]{\left\lVert#1\right\rVert_{\mathrm{op}}}

\newcommand{\normtwo}[1]{\left\lVert#1\right\rVert_2}
\newcommand{\eq}[1]{\begin{align}#1\end{align}}
\newcommand{\eqn}[1]{\begin{align*}#1\end{align*}}

\newcommand{\Tr}[1]{\mathrm{Tr}\left(#1\right)}

\usepackage[super]{nth}
\newcommand{\lap}[1]{\Vec{\Delta} \left(#1\right)}
\def\ssum{\mathsmaller{\sum}}
\def\sint{\mathsmaller{\int}}
\def\wtwo{$W_2$}

\newcommand\numberthis{\addtocounter{equation}{1}\tag{\theequation}}

\title{Stochastic Runge-Kutta Accelerates \\Langevin Monte Carlo and Beyond}

\author{
    Xuechen Li\textsuperscript{1, 2},\, 
    Denny Wu\textsuperscript{1, 2},\, 
    Lester Mackey\textsuperscript{3},\, 
    Murat A. Erdogdu\textsuperscript{1, 2}\\
    University of Toronto\textsuperscript{1}, 
    Vector Institute\textsuperscript{2},
    Microsoft Research\textsuperscript{3}\\
    \texttt{\{lxuechen, dennywu, erdogdu\}@cs.toronto.edu},
    \texttt{lmackey@microsoft.com}
}

\begin{document}

\maketitle
\begin{abstract}
Sampling with Markov chain Monte Carlo methods often amounts to discretizing some continuous-time dynamics with numerical integration.
In this paper, we establish the convergence rate of sampling algorithms obtained by discretizing smooth It\^o diffusions exhibiting fast Wasserstein-$2$ contraction, based on local deviation properties of the integration scheme. 
In particular, we study a sampling algorithm constructed by discretizing the overdamped Langevin diffusion with 
the method of stochastic Runge-Kutta.
For strongly convex potentials that are smooth up to a certain order, its iterates converge to the target distribution in $2$-Wasserstein distance in $\tilde{\mathcal{O}}(d\epsilon^{-2/3})$ iterations.
This improves upon the best-known rate for strongly log-concave sampling based on the overdamped Langevin equation using only the gradient oracle without adjustment. 
In addition, we extend our analysis of stochastic Runge-Kutta methods to uniformly dissipative diffusions with possibly non-convex potentials and show they achieve better rates compared to the Euler-Maruyama scheme in terms of the dependence on tolerance $\epsilon$. 
Numerical studies show that these algorithms lead to better stability and lower asymptotic errors. 
\end{abstract}

\section{Introduction}
Sampling from a probability distribution is a fundamental problem that arises
in machine learning, statistics, and optimization. In many situations,
the goal is to obtain samples from a target distribution given only
the unnormalized density~\cite{brooks2011handbook,gelman2013bayesian,mackay2003information}.
A prominent approach to this problem is
the method of Markov chain Monte Carlo (MCMC), where an ergodic Markov chain
is simulated so that
iterates converge exactly or approximately to the distribution of interest~\cite{metropolis1953equation,brooks2011handbook}.

MCMC samplers based on numerically integrating continuous-time dynamics have proven very useful due to their ability to accommodate a stochastic gradient oracle~\cite{welling2011bayesian}.
Moreover, when used as optimizations algorithms, these methods can deliver strong theoretical guarantees in non-convex settings~\cite{raginsky2017non}.
A popular example in this regime is the unadjusted Langevin Monte Carlo (LMC) algorithm~\cite{roberts1996exponential}.
Fast mixing of LMC is inherited from exponential Wasserstein
decay of the Langevin diffusion, and numerical integration using the Euler-Maruyama scheme with a sufficiently small step size ensures the Markov chain tracks the diffusion.
Asymptotic guarantees of this algorithm are well-studied~\cite{roberts1996exponential,gelfand1991recursive,mattingly2002ergodicity},
and non-asymptotic analyses specifying explicit constants in convergence bounds were recently conducted~\cite{dalalyan2012sparse,dalalyan2017theoretical,durmus2016high,cheng2017convergence,durmus2018efficient,cheng2018sharp}.

To the best of our knowledge, the best known rate of LMC in $2$-Wasserstein distance is due to~\citet{durmus2016high} -- $\tilde{\mathcal{O}}({d}{\epsilon^{-1}})$ iterations are required to reach $\epsilon$ accuracy to $d$-dimensional target distributions with strongly convex potentials under the additional Lipschitz Hessian assumption, where $\tilde{\mathcal{O}}$ hides insubstantial poly-logarithmic factors.
Due to its simplicity and well-understood theoretical properties, LMC and its derivatives have
found numerous applications in statistics and machine learning~\cite{welling2011bayesian,ding2014bayesian}. 
However, from the numerical integration point of view,
the Euler-Maruyama scheme is usually less preferred for many problems due to its inferior stability compared to implicit schemes~\cite{anderson2009weak} and large integration error compared to high-order
schemes~\cite{milstein2013stochastic}.

In this paper, we study the convergence rate of MCMC samplers devised from
discretizing It\^o diffusions with exponential Wasserstein-2 contraction. 
Our result provides a general framework for establishing convergence rates of existing
numerical schemes in the SDE literature when used as sampling algorithms.
In particular, we establish \textit{non-asymptotic} convergence bounds for sampling with \textit{stochastic Runge-Kutta}~(SRK) methods.
For strongly convex potentials,
iterates of a variant of SRK applied to the overdamped Langevin diffusion
has a convergence rate of $\tilde{\mathcal{O}}({d}{\epsilon^{-{2}/{3}}})$.
Similar to LMC, the algorithm only queries the gradient oracle of the potential during
each update and improves upon the best known rate of $\tilde{\mathcal{O}}({d}{\epsilon^{-1}})$
for strongly log-concave sampling 
based on the overdamped Langevin diffusion
without Metropolis adjustment, under the mild extra assumption that the potential is smooth up to the third order. 
In addition, we extend our analysis to \textit{uniformly dissipative} diffusions, which 
enables sampling from non-convex potentials by choosing a non-constant 
diffusion coefficient. We study a different variant of SRK
and obtain the convergence rate of 
$\tilde{\mathcal{O}}(d^{3/4} m^2 {\epsilon^{-1}})$
for general It\^o diffusions, where $m$ is the dimensionality of the Brownian motion.
This improves upon the convergence rate of $\tilde{\mathcal{O}}(d {\epsilon^{-2}})$
for the Euler-Maruyama scheme
in terms of the tolerance $\epsilon$, while potentially trading off dimension dependence.

Our contributions can be summarized as follows:
\vspace{-2mm}
\begin{itemize}[leftmargin=*,noitemsep]
\item We provide a broadly applicable theorem for establishing convergence rates of
  sampling algorithms based on discretizing It\^o diffusions exhibiting exponential Wasserstein-2 contraction to
  the target invariant measure.
  The convergence rate is explicitly expressed in terms of the contraction rate of the diffusion and 
  local properties of the numerical scheme, both of which can be easily derived.
\item We show for strongly convex potentials,
  a variant of SRK applied to the overdamped Langevin diffusion
  achieves the improved convergence rate of $\tilde{\mathcal{O}}({d}{\epsilon^{-{2}/{3}}})$
  by accessing only the gradient oracle, under mild additional smoothness conditions on the potential. 
\item We establish the convergence rate of a different variant of SRK applied
  to uniformly dissipative diffusions.
  By choosing an appropriate diffusion coefficient,
  we show the corresponding algorithm can sample from certain non-convex potentials
  and achieves the rate of $\tilde{\mathcal{O}}({d^{3/4} m^2}{\epsilon^{-1}})$.
\item We provide examples and numerical studies of sampling from both convex and non-convex potentials 
      with SRK methods and show they lead to better stability and lower asymptotic errors.
\end{itemize}

\vspace{-2mm}
\subsection{Additional Related Work}
\vspace{-1mm}

\paragraph{High-Order Schemes.}
Numerically solving SDEs has been a research area for decades~\cite{milstein2013stochastic,kloeden2013numerical}.
We refer the reader to~\cite{burrage2004numerical} for a review and to~\cite{kloeden2013numerical} for technical foundations.
\citet{chen2015convergence} studied the convergence of smooth functions evaluated at iterates of sampling algorithms obtained by discretizing the Langevin diffusion with high-order numerical schemes. 
Their focus was on convergence rates of function evaluations under a stochastic gradient oracle using asymptotic arguments. 
This convergence assessment pertains to analyzing numerical schemes in the weak sense. 
By contrast, we establish non-asymptotic convergence bounds in the $2$-Wasserstein metric, which covers a broader class of functions by the Kantorovich duality~\cite{gentil2015analogy,villani2008optimal}, and our techniques are based on the mean-square convergence analysis of numerical schemes. 
Notably, a key ingredient in the proofs by~\citet{chen2015convergence}, i.e. moment bounds in the guise of a Lyapunov function argument, is assumed without justification, whereas we derive this formally and obtain convergence bounds with explicit dimension dependent constants. 
\citet{durmus2016stochastic} considered convergence of function evaluations of schemes obtained using Richardson-Romberg extrapolation.
\citet{sabanis2018higher} introduced a numerical scheme that queries the gradient of the Laplacian based on an integrator that accommodates superlinear drifts~\cite{sabanis2019explicit}. In particular, for potentials with a Lipschitz gradient, they obtained the convergence rate of $\mathcal{\tilde{O}}({d^{4/3}}{\epsilon^{-2/3}})$. 
In optimization, high-order ordinary differential equation (ODE) integration schemes were introduced to discretize a second-order ODE and achieved acceleration~\cite{zhang2018direct}.

\vspace{-1mm}
\paragraph{Non-Convex Learning.}
The convergence analyses of sampling using the overdamped and underdamped Langevin diffusion 
were extended to the non-convex setting~\cite{cheng2018sharp,ma2019there}. 
For the Langevin diffusion, the most common assumption on the potential is strong convexity outside a ball of finite radius, in addition to Lipschitz smoothness and twice differentiability~\cite{cheng2018sharp,ma2018sampling,ma2019there}. 
More generally, \citet{vempala2019rapid} showed that convergence in the KL divergence of LMC can be derived assuming a log-Sobolev inequality on the target measure. 
For general It\^o diffusions, the notion of \textit{distant dissipativity}~\cite{gorham2016measuring,eberle2016reflection,eberle2017couplings} is used to study convergence to target measures with non-convex potentials in the $1$-Wasserstein distance.
Different from these works, our non-convex convergence analysis, due to conducted in $W_2$, requires the slightly stronger uniform dissipativity condition~\cite{gorham2016measuring}.
In optimization, non-asymptotic results for stochastic gradient Langevin dynamics and its variants have been established for non-convex objectives~\cite{raginsky2017non,xu2018global,erdogdu2018global,zou2019sampling}. 

\renewcommand{\arraystretch}{1.4}
\begin{table}
\begin{center}
\caption{Convergence rates in $W_2$ for algorithms sampling from strongly convex potentials by discretizing the overdamped Langevin diffusion. ``Oracle'' refers to highest derivative used in the update. ``Smoothness'' refers to Lipschitz conditions. Note that faster algorithms exist by discretizing high-order Langevin equations~\cite{dalalyan2018sampling,cheng2017underdamped,cheng2018sharp,mou2019high,shen2019randomized} or applying Metropolis adjustment~\cite{dwivedi2018log,chen2019fast}.}
\vspace{-6mm}
\begin{tabular}[t]{ c c c c }
 \toprule
 \textbf{Method} & \textbf{Convergence Rate} & \textbf{Oracle} & \textbf{Smoothness} \\ 
 \hline
Euler-Maruyama~\cite{durmus2016high} & $\tilde{\mathcal{O}}(d\epsilon^{-2})$ & $1$st order & gradient\\
\hline
  Euler-Maruyama~\cite{durmus2016high} & $\tilde{\mathcal{O}}(d\epsilon^{-1})$ & $1$st order & gradient \& Hessian \\ 
\hline
Ozaki's~\cite{dalalyan2017theoretical}
  \tablefootnote{
  We obtain a rate in $W_2$ from the discretization analysis in KL~\cite{dalalyan2017theoretical} via standard techniques~\cite{raginsky2017non,vempala2019rapid}. }
  & $\tilde{\mathcal{O}}(d\epsilon^{-1})$ & $2$nd order & gradient \& Hessian \\
\hline
{Tamed Order 1.5~\cite{sabanis2018higher}
\tablefootnote{
\citet{sabanis2018higher} use the Frobenius norm for matrices and the Euclidean norm of Frobenius norms for 3-tensors. For a fair comparison, we convert their Lipschitz constants to be based on the operator norm.
}
} & $\tilde{\mathcal{O}}(d^{4/3}\epsilon^{-2/3})$ & $3$rd order & $1$st to $3$rd derivatives \\ 
\hline
\textbf{Stochastic Runge-Kutta~(this work)} & $ \tilde{\mathcal{O}}(d\epsilon^{-2/3})$ & $1$st order & $1$st to $3$rd derivatives \\ 
\bottomrule
\end{tabular}
\end{center}
\vspace{-6mm}
\label{table:rates}
\end{table}

\paragraph{Notation.}
We denote the $p$-norm of a real vector $x \in \R^d$ by $\norm{x}_p$. 
For a function $f: \R^d \to \R$, we denote its $i$th derivative by $\nabla^i f(x)$ and
its Laplacian by $ \Delta f = \sum_{i=1}^d {\partial^2 f_i(x)}/{\partial x_i^2}$.
For a vector-valued function $g: \R^d \to \R^m$,
we denote its vector Laplacian by $\Vec{\Delta} (g)$, i.e. $\Vec{\Delta} (g)_i = \Delta (g_i)$.
For a tensor $ T \in \R^{d_1 \times d_2 \times \cdots \times d_m}$,
we define its operator norm recursively as $ \normop{T} = \sup_{\norm{u}_2\le1} \normop{T[u]}$,
where $T[u]$ denotes the tensor-vector product.
For $f$ sufficiently differentiable, we denote the Lipschitz and polynomial coefficients of its $i$th order derivative as
\eqn{
  \mu_{0}(f) \!=\! 
               \sup_{x \in \mathbb{R}^{d}}\|f(x)\|_{\mathrm{op}},\ \
               \mu_{i}(f)\! = 
               \!\!\!\!\!\!
               \sup_{ x, y \in \mathbb{R}^{d}, x \neq y} \!\!\!\!\!
               \tfrac{ \normop{\nabla^{i-1} f(x)-\nabla^{i-1} f(y)} }{ \normtwo{x - y}},
               \ \text{and} \ \pi_{i, n} (f) \!=\! \sup_{ x \in \R^d} \tfrac{\normop{\nabla^{i-1} f(x)}^n}{ 1 + \normtwo{x}^n},
}
with the exception in Theorem~\ref{theo:srk_id}, where $\pi_{1, n}(\sigma)$ is used for a sublinear growth condition.
We denote Lipschitz and growth coefficients under the Frobenius norm $\normf{\cdot}$ as $\mu_1^{\mathrm{F}}(\cdot)$ and $\pi_{1,n}^{\mathrm{F}}(\cdot)$, respectively.

\vspace{-1mm}
\paragraph{Coupling and Wasserstein Distance.}
We denote by $\B(\R^d)$ the Borel $\sigma$-field of $\R^d$. Given probability measures $\nu$ and $\nu'$ on $(\R^d, \B(\R^d))$, we define a coupling (or transference plan) $\zeta$ between $\nu$ and $\nu'$ as a probability measure on $(\R^d \times \R^d, \B(\R^d \times \R^d)) $ such that $\zeta(A \times \R^d) = \nu(A)$ and $\zeta(\R^d \times A) = \nu'(A)$ for all $A \in \B(\R^d)$. Let $\mathrm{couplings}(\nu, \nu')$ denote the set of all such couplings. We define the $2$-Wasserstein distance between a pair of probability measures $\nu$ and $\nu'$ as 
\eqn{
  W_2(\nu, \nu') = 
    \inf_{\zeta \in \mathrm{couplings}(\nu, \nu')} \bracks{ \sint \normtwo{x - y}^2 \dzeta(\nu, \nu') }^{1/2}.
}

\vspace{-4mm}
\section{Sampling with Discretized Diffusions}
We study the problem of sampling from a target distribution $p(x)$
with the help of a candidate It\^o diffusion~\cite{ma2015complete,meyn2012markov} given as the solution to the following stochastic differential equation (SDE):
\begin{align}
  \dX_t = b(X_t) \dt + \sigma(X_t) \dB_t, \quad  \text{with} \quad X_0 = x_0, \label{eq:continuous_general}
\end{align}
where $b: \R^d \to \R^d$ and $\sigma: \R^d \to \R^{d \times m}$ are termed as
the drift and diffusion coefficients, respectively.
Here, $\{B_t\}_{t\ge0}$ is an $m$-dimensional Brownian motion adapted to the filtration $\{\F_t\}_{t \ge 0}$, whose $i$th dimension we denote by $
\text{
  \footnotesize
  $
    \{B_t^{(i)}\}_{t\ge0}
  $
}$.
A candidate diffusion should be chosen so that
(i) its invariant measure is
the target distribution $p(x)$ and (ii) it exhibits fast mixing properties.
Under mild conditions, one can design a diffusion with the target invariant measure
by choosing the drift coefficient as (see e.g. \cite[Thm. 2]{ma2015complete})
\begin{align}
  b(x) = \tfrac{1}{2p(x)}\abracks{
  \nabla , p(x) w(x)
  } , \quad \text{where}\quad w(x) = \sigma(x) \sigma(x)^\top + c(x),  \label{eq:drift_coeff}
\end{align}
$c(x)\in \R^{d \times d}$ is any skew-symmetric matrix and $\abracks{\nabla, \cdot}$
is the divergence operator for a matrix-valued function, i.e. 
$ \abracks{\nabla, w(x)}_i = \sum_{j=1}^{d} {\partial w_{i,j}(x)}/{\partial x_j}$ for $w: \R^d \to \R^{d\times d}$.
To guarantee that this diffusion has fast convergence properties, we will require certain dissipativity conditions
to be introduced later.
For example, if the target is the Gibbs measure of a strongly convex potential $f: \R^d \to \R$,
i.e., $p(x) \propto \exp{ \bracks{- f(x)}}$, a popular candidate diffusion is
the (overdamped) Langevin diffusion which is the solution to the following SDE:
\begin{align}
    \dX_t = -\nabla f(X_t) \dt + \sqrt{2} \dB_t, \quad \text{with}\quad X_0 = x_0 \label{eq:continuous}.
\end{align}
It is straightforward to verify \eqref{eq:drift_coeff} for the above diffusion which implies that the target $p(x)$
is its invariant measure.
Moreover, strong convexity of $f$ implies uniform dissipativity and ensures that the diffusion achieves fast convergence.

\subsection{Numerical Schemes and the It\^o-Taylor Expansion}
In practice, the It\^o diffusion \eqref{eq:continuous_general} (similarly \eqref{eq:continuous})
cannot be simulated in continuous time and is instead approximated by
a discrete-time numerical integration scheme. Owing to its simplicity,
a common choice is the Euler-Maruyama (EM) scheme \cite{kloeden2013numerical},
which relies on the following update rule,
\begin{align} \label{eq:em_general}
  \tX_{k+1} = \tX_k + h\, b(\tX_k) + \sqrt{h}\, \sigma(\tX_k) \xi_{k+1}, \quad k=0, 1, \dots,
\end{align}
where $h$ is the step size and $\xi_{k+1} \overset{\text{i.i.d.}}{\sim} \N(0, I_d)$ is independent of $\tX_k$ for all $k \in \mathbb{N}$.
The above iteration defines a Markov chain and due to discretization error,
its invariant measure $\tilde{p}(x)$ is different from the target distribution $p(x)$; yet,
for a sufficiently small step size, the difference between $\tilde{p}(x)$
and $p(x)$ can be characterized (see e.g. \cite[Thm. 7.3]{mattingly2002ergodicity}).

Analogous to ODE solvers, numerical schemes such as the EM scheme and SRK schemes are derived based on approximating the continuous-time dynamics locally.
Similar to the standard Taylor expansion, It\^o's lemma induces a stochastic
version of the Taylor expansion of a smooth function evaluated at a stochastic process at time $t$.
This is known as the It\^o-Taylor (or Wagner-Platen) expansion~\cite{milstein2013stochastic}, and
one can also interpret the expansion as recursively applying It\^o's lemma to terms in the integral form of an SDE. 
Specifically, for $g:\R^d \to \R^d$, we define the operators:
\eq{
  L (g) (x) =
  \nabla g(x) \cdot b(x) \!+\! 
  \tfrac{1}{2} \ssum_{i=1}^m \nabla^2 g(x) [\sigma_i(x), \, \sigma_i(x)], 
  \quad
  \Lambda_j (g) (x) = \nabla g(x) \cdot \sigma_j (x), \label{eq:ito_taylor_ops}
}
where $\sigma_i(x)$ denotes the $i$th column of $\sigma(x)$.
Then, applying It\^o's lemma to the integral form of the SDE~\eqref{eq:continuous_general} with the starting point $X_0$ yields the following expansion around $X_0$~\cite{kloeden2013numerical,milstein2013stochastic}:
\eqn{
  \text{
    \footnotesize
    $X_t $
  }
  \!=\!&\!\! 
  \, 
  \overbrace{
    \underbrace{
      \text{
        \footnotesize
        $X_0+
        t\, b(X_0)  
        + \sigma(X_0) B_t
        $ 
      }
    }_{\text{Euler-Maruyama update}}
    +
    \text{\footnotesize
      $\ssum_{i,j=1}^m \sint_0^t \sint_0^s \Lambda_j (\sigma_i) (X_u) \!\dB_u^{(j)}\! \dB_u^{(i)}$
    }
  }^{\text{mean-square order 1.0 stochastic Runge-Kutta update}}
  + 
  \text{\footnotesize
    $\sint_0^t \sint_0^s L(b) (X_u)\!\du\! \ds$
  }
  \\&
  +
  \text{\footnotesize
    $\ssum_{i=1}^m \sint_0^t \sint_0^s L (\sigma_i) (X_u) \du \dB_s^{(i)}$
  }
  + 
  \text{
    \footnotesize
    $\ssum_{i=1}^m \sint_0^t \sint_0^s \Lambda_i (b) (X_u) \dB_u^{(i)} \ds.$
  } 
  \numberthis \label{eq:ito_taylor_expansion}
}
The expansion justifies the update rule of the EM scheme, since the discretization is 
nothing more than taking the first three terms on the right hand side of~\eqref{eq:ito_taylor_expansion}.
Similarly, a mean-square order 1.0 SRK scheme for general It\^o diffusions -- introduced in Section~\ref{subsubsec:srk-nonconvex} -- approximates the first four terms.
In principle, one may recursively apply It\^o's lemma
to terms in the expansion to obtain a more fine-grained approximation.
However, the appearance of non-Gaussian terms in the guise of iterated Brownian integrals 
presents a challenge for simulation.
Nevertheless, it is clear that the above SRK scheme will be a more accurate local approximation than the EM scheme,
due to accounting more terms in the expansion. As a result, the local deviation
between the continuous-time process and Markov chain will be smaller.
We characterize this property of a numerical scheme as follows.

\begin{defi}[Uniform Local Deviation Orders]
  Let $\{\tX_k\}_{k\in \mathbb{N}}$ denote the discretization of an It\^o diffusion $\{X_t\}_{t\ge 0}$
  based on a numerical integration scheme with constant step size $h$,
  and its governing Brownian motion $\{B_t\}_{t\ge 0}$ be adapted to the filtration $\{\F_t\}_{t \ge 0}$.
  Suppose 
  $
  \text{
    \footnotesize
    $
    \{X_s^{(k)}\}_{s\ge 0}
    $
  }
  $ is another instance of the same diffusion starting from $\tX_{k-1}$ at $s=0$
  and governed by the Brownian motion $\{B_{s+h(k-1)}\}_{s\ge 0}$. 
  Then, the numerical integration scheme has local deviation 
  $
  \text{
    \footnotesize
    $
    D_{h}^{(k)} = \tX_{k} - X_h^{(k)}
    $
  }
  $
  with uniform orders $(p_1, p_2)$ if 
  \begin{align}
    \mathcal{E}_k^{(1)}
      = \Exp{
        \mathbb{E} \big[
          \|
          D_{h}^{(k)}
          \|_2^2 | \F_{t_{k-1}}
        \big]
      } \le \lambda_1 h^{2 p_1}, \quad
    \mathcal{E}^{(2)}_k
      = \Exp{ \big\|
      \mathbb{E}\big[
      D_{h}^{(k)} | \F_{t_{k-1}}
      \big]
    \big\|_2^2 } \le 
    \lambda_2 h^{2 p_2}, \label{eq:uniform_local_deviation_orders}
  \end{align}
  for all $k \in \mathbb{N}_{+}$ and $0 \le h < C_h$, where constants $0 < \lambda_1, \lambda_2, C_h < \infty$. 
  We say that $\scriptstyle \mathcal{E}_k^{(1)}$ and $\scriptstyle \mathcal{E}_k^{(2)}$ are  
  the local mean-square deviation and the local mean deviation at iteration $k$, respectively. 
\end{defi}
In the SDE literature, local deviation orders are defined to derive the mean-square 
order (or strong order) of numerical schemes~\cite{milstein2013stochastic}, 
where the mean-square order is defined as the maximum half-integer $p$ 
such that 
$
\E [
    \| X_{t_k} - \tX_{k} \|_2^2
  ]
  \le
    C h^{2 p}
$
for 
a constant $C$ independent of step size $h$ and 
all $k \in \mathbb{N}$ where $t_k < T$.
Here, $\{X_{t}\}_{t\ge0}$ is the continuous-time process, $\tX_k (k=0,1,\dots)$  is the Markov chain 
with the same Brownian motion as the continuous-time process, and $T < \infty$ is the 
terminal time.
The key difference between our definition of \textit{uniform} local deviation orders and
local deviation orders in the SDE literature is we require the extra step of 
ensuring the expectations of $\scriptstyle \mathcal{E}_k^{(1)}$ 
and $\scriptstyle \mathcal{E}_k^{(2)}$ are bounded across all iterations,
instead of merely requiring the two deviation variables to be bounded
by a function of the previous iterate.

\section{Convergence Rates of Numerical Schemes for Sampling}
We present a user-friendly and broadly applicable theorem
that establishes the convergence rate of a diffusion-based sampling algorithm.
We develop our explicit bounds in the $2$-Wasserstein distance based on two crucial steps.
We first verify that the candidate diffusion exhibits exponential Wasserstein-2 contraction and thereafter 
compute the uniform local deviation orders of the scheme.
\begin{defi}[Wasserstein-$2$ rate]
A diffusion $X_t$ has Wasserstein-2 (\wtwo) rate $r:\R_{\ge 0} \to \R$ if for two instances of 
the diffusion $X_t$ initiated respectively from $x$ and $y$, we have
\eqn{
  W_2(\delta_xP_t,\delta_yP_t) 
\le r(t) \normtwo{x - y}, \quad \text{for all} \; x, y\in \R^d, t \ge 0,
}
where $\delta_xP_t$ denotes the distribution of the diffusion $X_t$ starting from $x$.
Moreover, if $r(t) = e^{- \alpha t}$ for some $\alpha >0 $, then 
we say the diffusion has exponential \wtwo-contraction. 
\end{defi}
The above condition guarantees fast mixing of the sampling algorithm.
For It\^o diffusions, uniform dissipativity suffices to ensure exponential \wtwo-contraction
$r(t) = e^{-\alpha t}$ \cite[Prop. 3.3]{erdogdu2018global}.
\begin{defi}[Uniform Dissipativity]
A diffusion defined by \eqref{eq:continuous_general} is $\alpha$-uniformly dissipative if
\eq{
    \abracks{
        b(x) - b(y), x - y
    }
    + 
    \tfrac{1}{2}\normf{\sigma(x) - \sigma(y)}^2
    \le
        -\alpha \normtwo{x-y}^2, \quad \text{for all}\; x, y \in \R^d.
}
\end{defi}
For It\^o diffusions with a constant diffusion coefficient,
uniform dissipativity is equivalent to one-sided Lipschitz continuity of the drift with coefficient $-2\alpha$.
In particular, for the overdamped Langevin diffusion \eqref{eq:continuous},
this reduces to strong convexity of the potential.
Moreover, for this special case, exponential \wtwo-contraction of the diffusion and strong convexity of the potential are equivalent~\cite{calogero2012exponential}.
We will ultimately verify uniform dissipativity for the candidate diffusions,
but we first use \wtwo-contraction to derive the convergence rate of a diffusion-based sampling algorithm.

\begin{theo}[\wtwo-rate of a numerical scheme]\label{theo:master}
  For a diffusion with invariant measure $\nu^*$, exponentially contracting $W_2$-rate $r(t) = e^{-\alpha t}$, and Lipschitz
  drift and diffusion coefficients,
  suppose its discretization based on a numerical integration scheme has uniform local deviation orders $(p_1, p_2)$ 
where $p_1 \ge {1}/{2}$ and $p_2 \ge p_1 + {1}/{2}$. 
Let $\nu_k$ be the measure associated with the Markov chain obtained from 
the discretization after $k$ steps starting from the dirac measure $\nu_0 = \delta_{x_0}$.
Then, for constant step size $h$ satisfying
\eqn{
    h < 1 
    \wedge C_h 
    \wedge \frac{1}{2\alpha}
    \wedge \frac{1}{8\mu_1(b)^2 + 8 \mu_1^{\mathrm{F}}(\sigma)^2}
    ,
}
where $C_h$ is the step size constraint for obtaining the uniform local deviation orders, 
we have
{
    \footnotesize
    \eq{
        W_2(\nu_k, \nu^*) \le 
            \bracks{
                1 - \frac{\alpha h}{2}
            }^k
            W_2(\nu_0, \nu^*)
            +
            \bracks{
                \frac{
                    8\bracks{
                        16 \mu_1(b) \lambda_1 + \lambda_2
                    }
                }{\alpha^2} +
                \frac{2 \lambda_1}{\alpha}
            }^{1/2}
            h^{p_1 - 1/2}. \label{eq:main_bound}
    }
}
Moreover, if $p_1 > 1/2$ and the step size additionally satisfies
{
\footnotesize
  \eqn{
      h <
        \bracks{
          \frac{2}{\epsilon}
              \sqrt{
                  \frac{64 ( 16 \lambda_1 \mu_1(b) + \lambda_2 ) }{\alpha^2}
                  + 
                  \frac{2 \lambda_1}{\alpha}
              }
      }^{-1 / (p_1 - 1/2)},
  }
}
then $W_2(\nu_k, \nu^*)$ converges in $\mathcal{\tilde{O}}({\epsilon^{-{1}/{(p_1 - 1/2)}}})$ iterations within a sufficiently small positive error $\epsilon$.
\end{theo}

Theorem~\ref{theo:master} directly translates mean-square order
results in the SDE literature to convergence rates of sampling algorithms in $W_2$. 
The proof deferred to Appendix~\ref{proof:master} follows from
an inductive argument over the local deviation at each step (see e.g. \cite{milstein2013stochastic}),
and the convergence is provided by the exponential \wtwo-contraction of the diffusion.
To invoke the theorem and obtain convergence rates of a sampling algorithm,
it suffices to (i) show that the candidate diffusion is uniformly dissipative
and (ii) derive the local deviation orders for the underlying discretization.
Below, we demonstrate this on both the overdamped Langevin and general It\^o diffusions when the EM scheme is used for discretization, as well as the underdamped Langevin diffusion when a linearization is used for discretization~\cite{cheng2017underdamped}. 
For these schemes, local deviation orders are either well-known or straightforward to derive. 
Thus, convergence rates for corresponding sampling algorithms can be easily obtained using Theorem~\ref{theo:master}.

\begin{exam}
  Consider sampling from a target distribution whose potential is strongly convex using
  the overdamped Langevin diffusion \eqref{eq:continuous} discretized by the EM scheme.
  The scheme has local deviation of orders $(1.5, 2.0)$ for It\^o diffusions with
  constant diffusion coefficients
  and drift coefficients that are sufficiently smooth
  \footnote{In fact, it suffices to ensure the drift is three-times differentiable with Lipschitz gradient and Hessian.} 
  (see e.g.~\cite[Sec. 1.5.4]{milstein2013stochastic}).
  Since the potential is strongly convex,
  the Langevin diffusion is uniformly dissipative and achieves exponential \wtwo-contraction~\cite[Prop. 1]{durmus2016high}.
  Elementary algebra shows that Markov chain moments are bounded~\cite[Lem. A.2]{erdogdu2018global}. 
  Therefore, Theorem~\ref{theo:master} 
  implies that the rate of the sampling is
  $\tilde{\mathcal{O}}(d \epsilon^{-1})$, where the dimension dependence can be extracted
  from the explicit bound. This recovers the result by~\citet[Thm. 8]{durmus2016high}. 
\end{exam}

\begin{exam} \label{exam:em_unif_diss}
  If a general It\^o diffusion \eqref{eq:continuous_general}
  with Lipschitz smooth drift and diffusion coefficients
  is used for the sampling task,
  local deviation orders of the EM scheme reduce to $(1.0, 1.5)$
  due to the approximation of the diffusion term \cite{milstein2013stochastic} --
  this term is exact for Langevin diffusion.
  If we further have uniform dissipativity, 
  it can be shown that Markov chain moments are bounded \cite[Lem. A.2]{erdogdu2018global}. 
  Hence, Theorem~\ref{theo:master} concludes that the convergence rate is $\tilde{\mathcal{O}}(d \epsilon^{-2})$.
  We note that for the diffusion coefficient, we use the Frobenius norm for the 
  Lipschitz and growth constants which potentially hides dimension dependence factors. 
  The dimension dependence worsens if one were to convert all bounds to be based on the operator norm 
  using the pessimistic inequality 
  $
  \text{
    \footnotesize
    $
      \normf{\sigma(x)} \le ( d^{1/2} + m^{1/2} ) \normop{\sigma(x)}
    $
  }$.
  Appendix~\ref{app:em_unif_diss} provides a convergence bound with explicit constants.
\end{exam}

\begin{exam}
  Consider sampling from a target distribution whose potential is strongly convex using the underdamped Langevin diffusion:
  \eq{
    \dX_t &= V_t \dt, \quad
    \dV_t = - \gamma V_t \dt - u \nabla f (X_t) \dt + \sqrt{2 \gamma u} \dB_t.
  }
  \citet{cheng2017underdamped} show that the continuous-time process $\{(X_t, X_t + V_t)\}_{t 
  \ge 0}$ exhibits exponential \wtwo-contraction when the coefficients $\gamma$ and $u$ are appropriately chosen~\cite[Thm. 5]{cheng2017underdamped}. Moreover, the scheme devised by linearizing the degenerate SDE for the augmented state has uniform local deviation orders $(1.5, 2.0)$~\footnote{\citet{cheng2017underdamped} derive the uniform local mean-square deviation order. Jensen's inequality implies that the local mean deviation is of the same uniform order. This entails uniform local deviation orders are $(2.0, 2.0)$ and hence also $(1.5, 2.0)$ when step size constraint $C_h \le 1$; note $p_2 \ge p_1 + 1/2$ is required to invoke Theorem~\ref{theo:master}.}~\cite[Thm. 9]{cheng2017underdamped}. 
  Theorem~\ref{theo:master} implies that the convergence rate is $\mathcal{O}(d^{1/2} \epsilon^{-1})$, where the dimension dependence is extracted from explicit bounds. This recovers the result by~\citet[Thm. 1]{cheng2017underdamped}.
\end{exam}

Since the first version of this paper appeared on arXiv, several new schemes were devised in the literature. 
We include the example of deriving the convergence rate for the recently proposed randomized midpoint method~\cite{shen2019randomized}. This example demonstrates that Theorem~\ref{theo:master} can also be applied to schemes that include additional randomness which is independent of that of the Brownian motion.
\begin{exam}
\citet{shen2019randomized} discretize the underdamped Langevin diffusion with a variant of the midpoint method, where the midpoint is computed with the linearization scheme~\cite{cheng2017underdamped} at a random time $\tau$ uniformly selected in $[0, h]$.
It can be shown that the local mean-square deviation order is the same as that of the one-step linearization scheme~\cite[Lem. 2]{shen2019randomized}. 
However, one sees that the local mean deviation order improves by inspecting the following bound
\eqn{
\text{
\footnotesize
$
  \big\|
    \E \big [ \tV_k - V_h^{(k)} | \F_{t_{k-1}} \big]
  \big\|_2^2
  \le
  \E_{\not \tau} \big [
    \big\| 
      \E_{\tau} [ \tV_k - V_h^{(k)}]
    \big\|_2^2
    | \F_{t_{k-1}}
  \big ]
  \le
  \mathcal{O}\bracks{
    \| \tV_{k-1} \|_2^2 h^{8}
    \! + \!
    \| \nabla f(\tX_{k-1}) \|_2^2  h^{10}
    \! + \!
    h^{9}
  },
$
}
}
where $\E_\tau$ and $\E_{\not \tau}$ denote taking expectation over the random time and Brownian motion, respectively.
The first and second inequalities are respectively by Jensen's and Lemma 2 of \cite{shen2019randomized}. 
A similar improvement holds for the position variable. Combined with moment bounds, this gives uniform deviation orders of $(2, 4)$ for the augmented SDE $\{(X_t, X_t + V_t)\}_{t\ge 0}$, which implies a convergence rate of $\tilde{\mathcal{O}}(\epsilon^{-2/3})$ without high-order smoothness for sufficiently small step size.
\end{exam}

While computing the local deviation orders of a numerical scheme for a single step is often straightforward,
it is not immediately clear how one might verify them uniformly for each iteration. This
requires a uniform bound on moments of the Markov chain defined by the numerical scheme. 
As our second principal contribution, we explicitly bound the Markov chain moments of SRK schemes
which, combined with Theorem~\ref{theo:master}, leads to improved rates by only accessing the first-order oracle.

\section{Sampling with Stochastic Runge-Kutta and Improved Rates} \label{sec:srk}
We show that convergence rates of sampling can be significantly improved if
an It\^o diffusion with exponential \wtwo-contraction is discretized using SRK methods.
Compared to the EM scheme, SRK schemes
we consider query the same order oracle and improve on the deviation orders.

Theorem~\ref{theo:master} hints that one may expect the convergence rate of sampling to improve
as more terms of the It\^o-Taylor expansion are incorporated in the numerical integration scheme. 
However, in practice, 
a challenge for simulation is the appearance of non-Gaussian terms in the form of iterated It\^o integrals.  
Fortunately, since the overdamped Langevin diffusion has a constant diffusion coefficient, efficient 
SRK methods can still be applied to accelerate convergence.

\subsection{Sampling from Strongly Convex Potentials with the Langevin Diffusion}
We provide a non-asymptotic analysis for integrating the overdamped Langevin diffusion based on a
mean-square order 1.5 SRK scheme for SDEs with constant diffusion coefficients~\cite{milstein2013stochastic}. 
We refer to the sampling algorithm as SRK-LD. Specifically, given a sample from the previous iteration $\tX_k$, 
\begin{align}
    \tH_1 &= \tX_k + \sqrt{2h} \sbracks{
        \bracks{
            \frac{1}{2} + \frac{1}{\sqrt{6}}
        } \xi_{k+1} + 
        \frac{1}{\sqrt{12}} \eta_{k+1}
    } , \notag \\
    \tH_2 &= \tX_k - h \nabla f(\tX_k) + \sqrt{2h} \sbracks{
        \bracks{
            \frac{1}{2} - \frac{1}{\sqrt{6}}
        } \xi_{k+1} + 
        \frac{1}{\sqrt{12}} \eta_{k+1}
    }, \notag \\
    \tX_{k+1} &= \tX_k -
    \frac{h}{2} \bracks{
        \nabla f(\tH_1) +
        \nabla f(\tH_2)
    } 
    + \sqrt{2h} \xi_{k+1} 
    \label{eq:srk},
\end{align}
where $h$ is the step size and $\xi_{k+1}, \eta_{k+1} \overset{\text{i.i.d.}}{\sim} \mathcal{N}(0, I_d)$ are independent of $\tX_k$ for all $k\in\mathbb{N}$. We refer the reader to~\cite[Sec. 1.5]{milstein2013stochastic} for a detailed derivation of the scheme and other background information.

\begin{theo}[SRK-LD] \label{theo:srk_ld}
Let $\nu^*$ be the target distribution
with a strongly convex potential that is four-times differentiable with Lipschitz continuous
first three derivatives.
Let $\nu_k$ be the distribution of the $k$th Markov chain iterate
defined by \eqref{eq:srk} starting from the dirac measure $\nu_0 = \delta_{x_0}$.
Then, for a sufficiently small step size,
1.5 SRK scheme has uniform local deviation orders $(2.0, 2.5)$, and 
$W_2(\nu_k, \nu^*)$ converges within $\epsilon$ error in $\tilde{\mathcal{O}}(d \epsilon^{-2/3})$
iterations.
\end{theo}
The proof of this theorem is given in Appendix~\ref{app:srk_ld} where we provide explicit constants. 
The basic idea of the proof
is to match up the terms in the It\^o-Taylor expansion to terms in the Taylor expansion 
of the discretization scheme. However, extreme care is needed to ensure a tight dimension dependence.

\begin{remark}
For large-scale Bayesian inference, computing the full gradient of the potential can be costly. 
Fortunately, for SRK-LD, the convergence rate is retained when we replace the first-order oracle with an unbiased stochastic one, provided queries of the latter have a variance not overly large. 
We provide an informal discussion in Appendix~\ref{app:srk-ld-stochasticgrad}.
\end{remark}

We emphasize that the 1.5 SRK scheme \eqref{eq:srk} only queries the gradient of the potential and improves 
the best available \wtwo-rate of LMC in the same setting
from $\tilde{\mathcal{O}}(d \epsilon^{-1})$ to $\tilde{\mathcal{O}}(d \epsilon^{-2/3})$,
with merely two extra gradient evaluations per iteration.
Remarkably, the dimension dependence stays the same.

\subsection{Sampling from Non-Convex Potentials with It\^o Diffusions}
\label{subsubsec:srk-nonconvex}
For the Langevin diffusion, the conclusions of Theorem~\ref{theo:master} only apply to distributions with strongly convex potentials, as exponential \wtwo-contraction of the Langevin diffusion is equivalent to strong convexity of the potential.
This shortcoming can be addressed using a non-constant diffusion coefficient which allows us to
sample from non-convex potentials using uniformly dissipative candidate diffusions. 
Below, we use a mean-square order 1.0 SRK scheme for general diffusions~\cite{rossler2010runge} and achieve an improved convergence rate compared to sampling with the EM scheme. 

We refer to the sampling algorithm as SRK-ID, which has the following update rule:
\eqn{
\tH_1^{(i)} &= \tX_k + \ssum_{j=1}^m \sigma_l(\tX_k) \frac{I_{(j,i)}}{\sqrt{h}}, \quad\quad
\tH_2^{(i)} = \tX_k - \ssum_{j=1}^m \sigma_l(\tX_k) \frac{I_{(j,i)}}{\sqrt{h}}, \\
\tX_{k+1} &= 
    \tX_k + h b(\tX_k)  + \ssum_{i=1}^m \sigma_i(\tX_k) I_{(i)} + 
            \frac{\sqrt{h}}{2} \ssum_{i=1}^m
        \big(
            \sigma_i (\tH_1^{(i)}) - \sigma_i (\tH_2^{(i)})
        \big), \numberthis \label{eq:srk_id}
}
where 
$
\text{
  \footnotesize
  $
    I_{(i)} = \sint_{t_k}^{t_{k+1}} \dB_s^{(i)}
  $
}$,
$
\text{
  \footnotesize
  $I_{(j,i)} = \sint_{t_k}^{t_{k+1}} \sint_{t_k}^{s} \dB_u^{(j)} \dB_s^{(i)}$
}
$.
We note that schemes of higher order exist for general diffusions, but they typically require 
advanced approximations of iterated It\^o integrals of the form 
$
\text{
  \footnotesize
  $
    \int_0^{t_0} \cdots \int_0^{t_{n-1}} \dB_{t_n}^{(k_n)} \cdots \dB_{t_1}^{(k_1)}
  $
}$.
\begin{theo}[SRK-ID] \label{theo:srk_id}
  For a uniformly dissipative diffusion with invariant measure $\nu_*$, Lipschitz drift and diffusion coefficients that have Lipschitz gradients, 
  assume that the diffusion coefficient further satisfies the sublinear growth condition
$
  \text{
    \footnotesize
    $
    \normop{\sigma(x)} \le \pi_{1,1}(\sigma) \bigl(1 + \normtwo{x}^{1/2} \bigr)
    $
  }
$
for all $x \in \R^d$.
  Let $\nu_k$ be the distribution of the $k$th Markov chain iterate
  defined by \eqref{eq:srk_id} starting from the dirac measure $\nu_0 = \delta_{x_0}$.
  Then for a sufficiently small step size, iterates of
  the 1.0 SRK scheme have uniform local deviation orders $(1.5, 2.0)$, and
  $W_2(\nu_k, \nu^*)$ converges within $\epsilon$ error in $\tilde{\mathcal{O}}(d^{3/4} m^{2} \epsilon^{-1})$
iterations.
\end{theo}
The proof is given in Appendix~\ref{app:srk_id} where 
we present explicit constants. We note that the dimension dependence in this case is only better than that of EM 
due to the extra growth condition on the diffusion. The extra $m$-dependence comes from the $2m$ evaluations of the 
diffusion coefficient at $\scriptstyle \tH_1^{(i)}$ and $\scriptstyle \tH_2^{(i)}$ ($i=1, \dots, m$). 
In the above theorem, we use the Frobenius norm for the 
Lipschitz and growth constants for the diffusion coefficient 
which potentially hides dimension dependence. One may 
convert all bounds to be based on the operator norm 
with our constants given in the Appendix. 

In practice, accurately simulating both the iterated It\^o integrals $I_{(j,i)}$
and the Brownian motion increments $I_{(i)}$ simultaneously is difficult.
We comment on two possible approximations based on truncating an infinite series in Appendix~\ref{app:srk-id}.

\section{Examples and Numerical Studies} \label{sec:exp}
We provide examples of our theory and numerical studies showing SRK methods achieve lower asymptotic errors, are stable under large step sizes, and hence converge faster to a prescribed tolerance.
We sample from strongly convex potentials with SRK-LD and non-convex potentials with SRK-ID. 
Since our theory is in $W_2$, we compare with EM on $W_2$ and mean squared error (MSE) between iterates of the Markov chain and the target. We do not compare to schemes that require computing derivatives of the drift and diffusion coefficients.
Since directly computing $W_2$ is infeasible, we estimate it using samples instead. However, sample-based estimators have a bias of order $\Omega(n^{-1/d})$~\cite{weed2017sharp}, so we perform a heuristic correction whose description is in Appendix~\ref{app:bias_correction}.

\begin{figure}[ht]
\begin{minipage}[t]{0.33\linewidth}
\centering
{\includegraphics[width=1.0\textwidth]{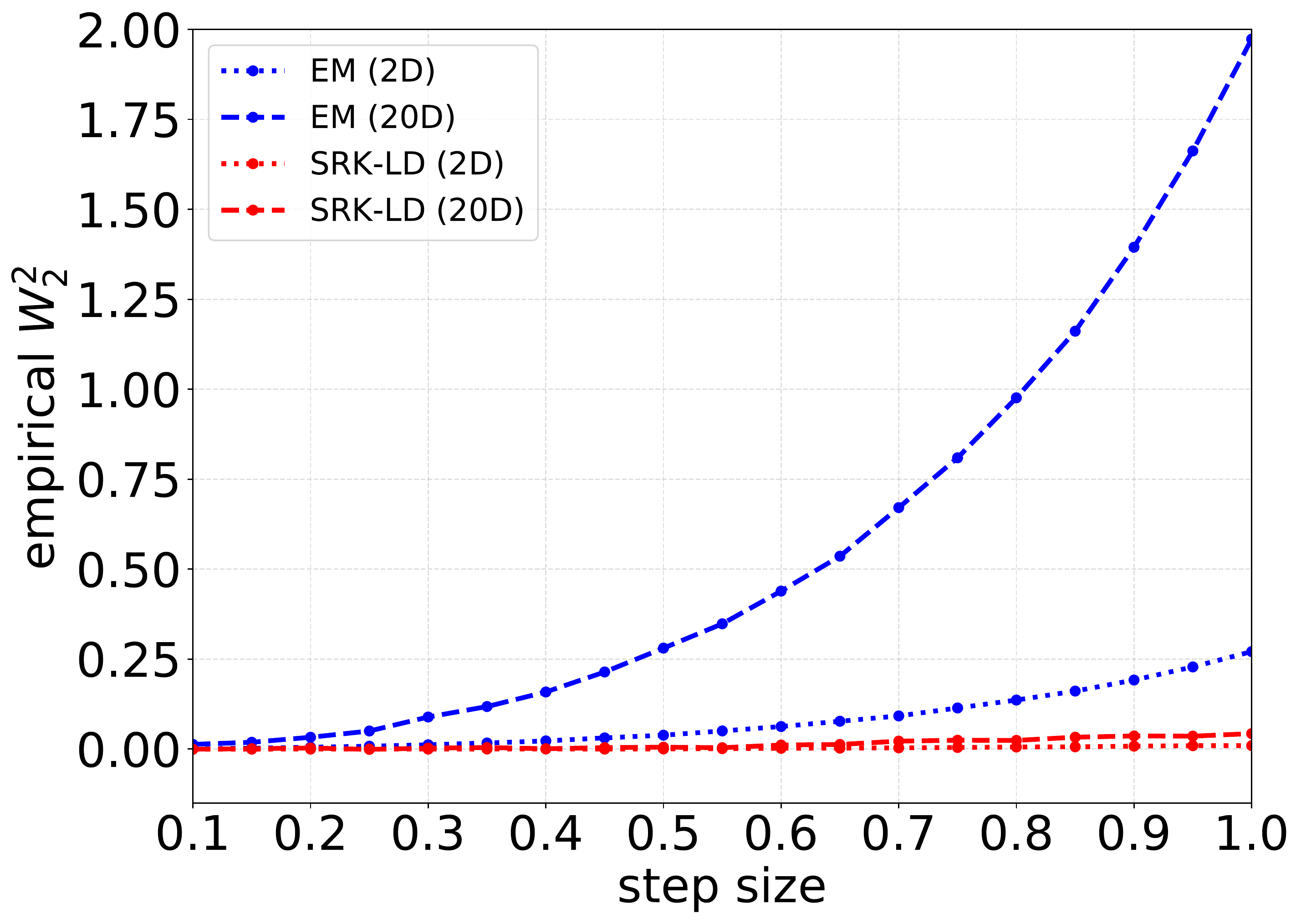}} 
{
    \footnotesize
    (a) Gaussian mixture
}
\end{minipage}
\begin{minipage}[t]{0.33\linewidth}
\centering
{\includegraphics[width=1.0\textwidth]{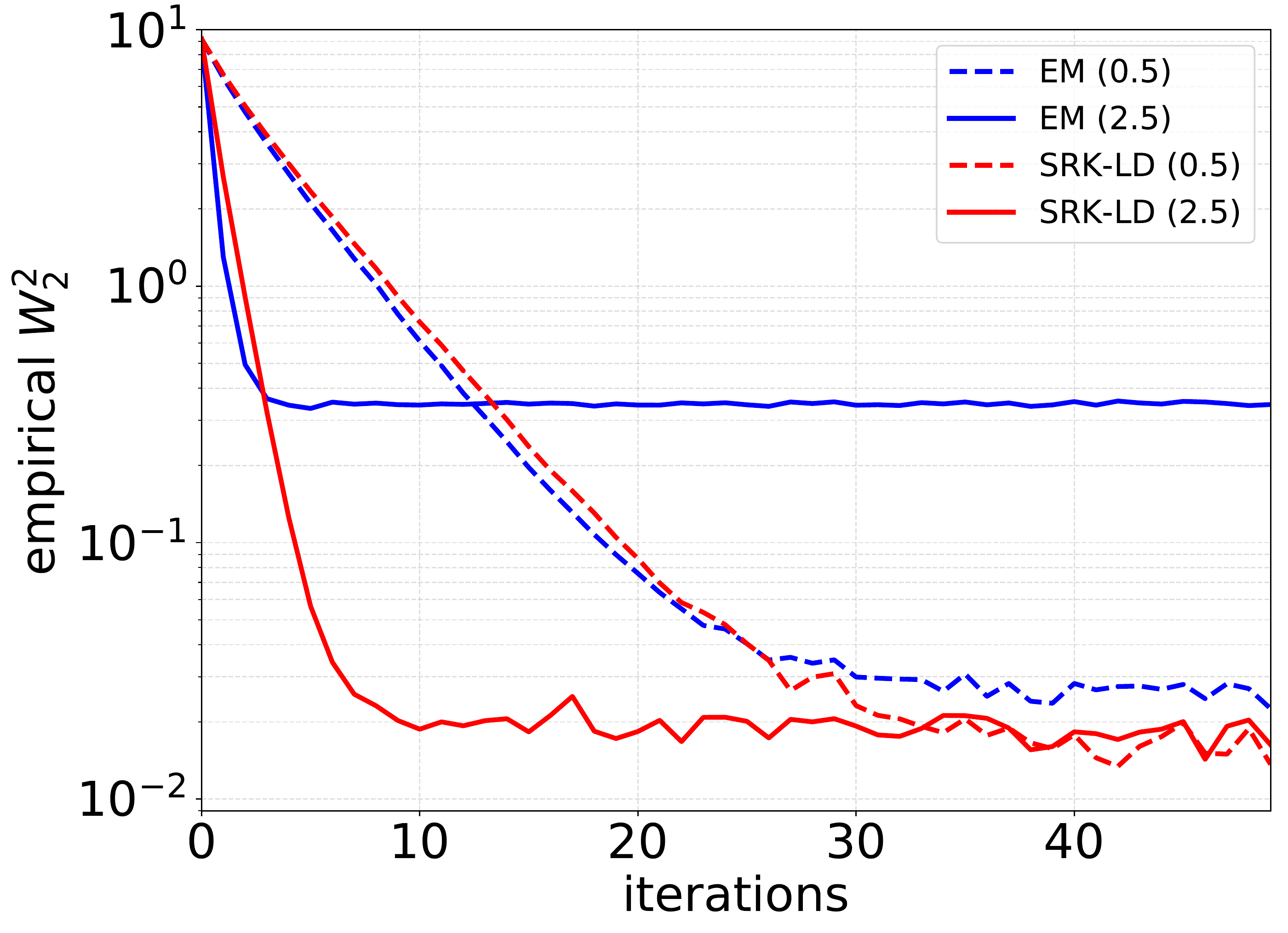}} 
{
    \footnotesize
    (b) Bayesian logistic regression
}
\end{minipage}
\begin{minipage}[t]{0.33\linewidth}
\centering
{\includegraphics[width=1.0\textwidth]{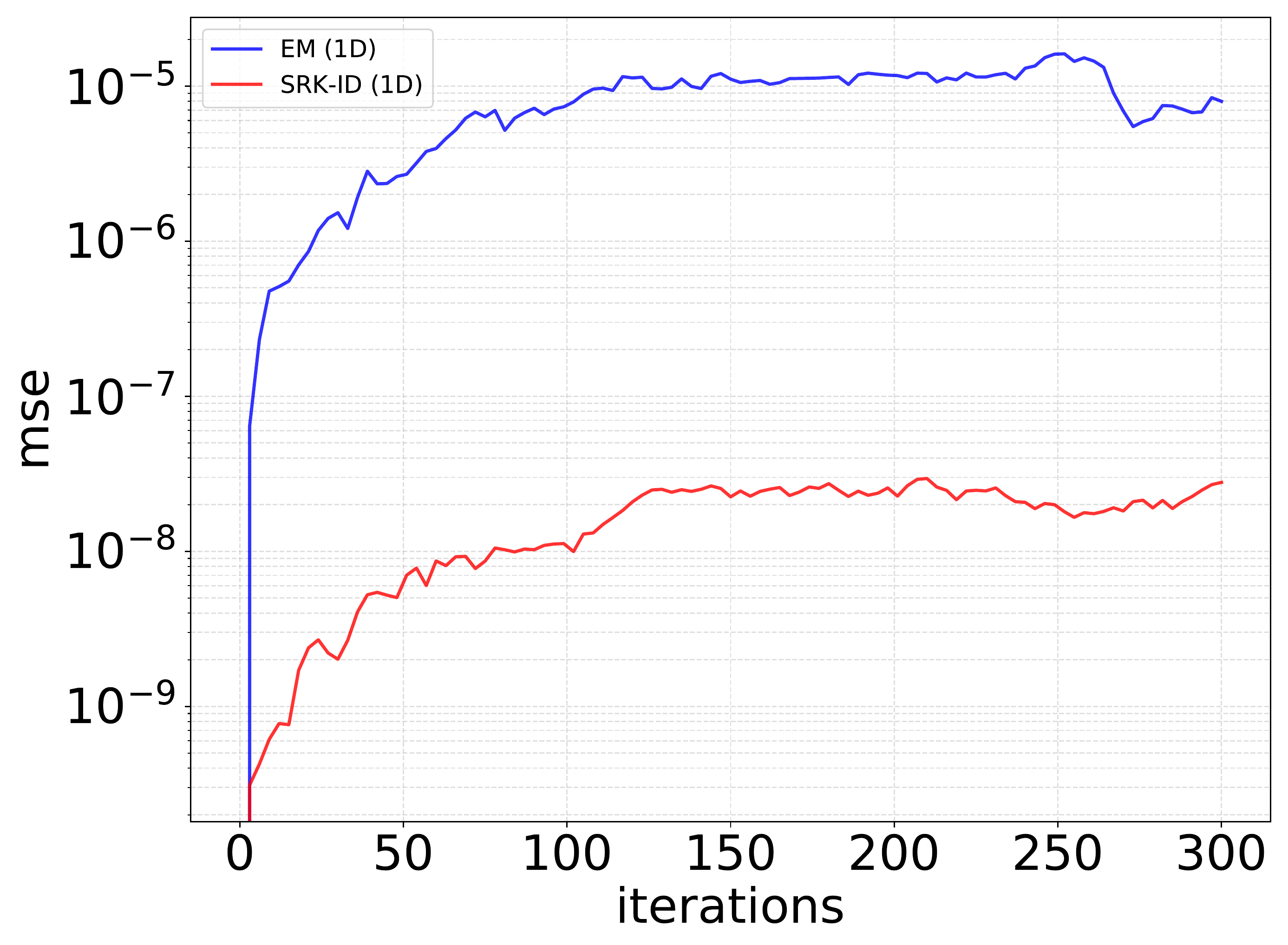}}
{
    \footnotesize
    (c) non-convex potential
}
\end{minipage}
\caption{
(a) Estimated asymptotic error against step size.
(b) Estimated error against number of iterations. 
(c) MSE against number of iterations.
Legends of (a) and (c) denote ``scheme (dimensionality)''.
Legend of (b) denotes ``scheme (step size)''.
}
\label{fig:strongly_convex}
\end{figure}

\subsection{Strongly Convex Potentials} \label{subsec:strongly_convex_exp}

\paragraph{Gaussian Mixture.}
We consider sampling from a multivariate Gaussian mixture with density
\eqn{
    \pi(\theta) \propto 
        \exp \bracks{- \tfrac{1}{2} \|\theta - a\|_{2}^{2} } + 
        \exp \bracks{- \tfrac{1}{2} \|\theta + a\|_{2}^{2} }, \quad \theta \in \mathbb{R}^{d},
}
where $a \in \R^d$ is a parameter that measures the separation of two modes. 
The potential is strongly convex when $\normtwo{a} < 1$ and has Lipschitz gradient and Hessian~\cite{dalalyan2017theoretical}. 
Moreover, one can verify that its third derivative is also Lipschitz.

\paragraph{Bayesian Logistic Regression.}
We consider Bayesian logistic regression (BLR)~\cite{dalalyan2017theoretical}. Given data samples $\mathrm{X} = \{x_i\}_{i=1}^n \in \R^{n \times d}$, $\mathrm{Y} = \{y_i\}_{i=1}^n \in \R^n$,
and parameter $\theta\in\R^d$, logistic regression models the Bernoulli conditional distribution with probability
$\Pr(y_i\!=\!1|x_i) = 1 / (1 + \exp(-\theta^\top x_i))$. 
We place a Gaussian prior on $\theta$ with mean zero and covariance proportional to $\Sigma^{-1}_{\mathrm{X}}$, 
where $\Sigma_{\mathrm{X}} = \mathrm{X}^\top \mathrm{X} /n$ is the sample covariance matrix. 
We sample from the posterior density
\eqn{
    \pi(\theta)\propto \exp (-f(\theta)) =
  \exp \Big(
    \mathrm{Y}^\top \mathrm{X} \theta - 
    \ssum_{i=1}^n \log( 1+\exp(-\theta^\top x_i) ) - 
    \tfrac{\alpha}{2}
        \| \Sigma_{\mathrm{X}}^{1/2}\theta \|_2^2 \Big).
}
The potential is strongly convex and has Lipschitz gradient and Hessian~\cite{dalalyan2017theoretical}. 
One can also verify that it has a Lipschitz third derivative.

To obtain the potential, we generate data from the model with the parameter $\theta_* = \mathbf{1}_d$ following~\cite{dalalyan2017theoretical,dwivedi2018log}. To obtain each $x_i$, we sample a vector whose components are independently drawn from the Rademacher distribution and normalize it by the Frobenius norm of the sample matrix $X$ times $d^{-1/2}$. 
Note that our normalization scheme is different from that adopted in~\cite{dalalyan2017theoretical,dwivedi2018log}, where each $x_i$ is normalized by its Euclidean norm.
We sample the corresponding $y_i$ from the model and fix the regularizer $\alpha = 0.3d / \pi^2$.

To characterize the true posterior, we sample 50k particles driven by EM with a step size of $0.001$ until convergence. We subsample from these particles 5k examples to represent the true posterior each time we intend to estimate squared $W_2$. We monitor the kernel Stein discrepancy~\footnote{Unfortunately, there appear to be two definitions for KSD and the energy distance in the literature, differing in whether a square root is taken or not. We adopt the version with the square root taken.} (KSD)~\cite{gorham2017measuring,chwialkowski2016kernel,liu2016kernelized} using the inverse multiquadratic kernel~\cite{gorham2017measuring} with hyperparameters $\beta=-1/2$ and $c=1$ to measure the distance between the 100k particles and the true posterior. We confirm that these particles faithfully approximate the true posterior with the squared KSD being less than $0.002$ in all settings. 

When sampling from a Gaussian mixture and the posterior of BLR, we observe that SRK-LD leads to a
consistent improvement in the asymptotic error compared to the EM scheme when the same step size is used. 
In particular, Figure~\ref{fig:strongly_convex} (a) plots the estimated asymptotic error in squared $W_2$ of different step sizes for 2D and 20D Gaussian mixture problems and shows that SRK-LD is surprisingly stable for exceptionally large step sizes. 
Figure~\ref{fig:strongly_convex} (b) plots the estimated error in squared $W_2$ as the number of iterations increases for 2D BLR. 
We include additional results on problems in 2D and 20D with error estimates in squared $W_2$ and the energy distance~\cite{szekely2003statistics} along with a wall time analysis in Appendix~\ref{app:additional_exp}.

\subsection{Non-Convex Potentials} \label{subsec:exp_nonconvex}
We consider sampling from the non-convex potential
\eqn{
    f(x) = \bigl(
            \beta + \| x \|_2^2
        \bigr)^{1/2}
        + \gamma  \log \bigl(
            \beta + \| x \|_2^2
        \bigr), \quad x \in \R^d,
}
where $\beta, \gamma > 0$ are scalar parameters of the distribution.
The corresponding density is a simplified abstraction for the posterior distribution of Student's t regression with a pseudo-Huber prior~\cite{gorham2016measuring}.
One can verify that when 
$
\text{
    \scriptsize
    $\beta + \normtwo{x}^2 < 1$
}
$ and
$
\text{
    \scriptsize
    $(4\gamma + 1) \normtwo{x}^2 < (2\gamma + 1) \sqrt{\beta + \normtwo{x}^2}$
}$, the Hessian has a negative eigenvalue. 
The candidate diffusion, where the drift coefficient is given by~\eqref{eq:drift_coeff} and diffusion coefficient
$
\text{ 
\footnotesize 
$\sigma(x) =  g(x)^{1/2} I_d$
}
$ 
with 
$
\text{
    \footnotesize
    $g(x) = \bigl( \beta + \| x \|_2^2 \bigr)^{1/2}$,
}
$ 
is uniformly dissipative if 
$
\tfrac{1}{2} -
| \gamma - \tfrac{1}{2} | \tfrac{2}{\beta^{1/2}} - 
\tfrac{d}{8 \beta^{1/2}} > 0$.
Indeed, one can verify that $\mu_1(g) \le 1$, $\mu_2(g)\le \tfrac{2}{\beta^{1/2}}$, and $\mu_1(\sigma) \le \tfrac{1}{2\beta^{1/4}}$. 
Therefore,
\eqn{
    \langle b(x) - b(y), x - y \rangle + 
    \tfrac{1}{2} \normf{\sigma(x) - \sigma(y)}^2
    \le&
        -\bigl(
            \tfrac{1}{2} -
            |\gamma - \tfrac{1}{2}| \mu_2(g) -
            \tfrac{d}{2} \mu_1(\sigma)^2
        \bigr)
        \normtwo{x - y}^2
        \\
    \le&
        -(
            \tfrac{1}{2} -
            |\gamma - \tfrac{1}{2}| \tfrac{2}{\beta^{1/2}} - 
            \tfrac{d}{8 \beta^{1/2}}
        ) \normtwo{x - y}^2
    .
}
Moreover, $b$ and $\sigma$ have Lipschitz first two derivatives, and the latter 
satisfies the sublinear growth condition in Theorem~\ref{theo:srk_id}. 

To study the behavior of SRK-ID, we simulate using both SRK-ID and EM. 
For both schemes, we simulate with a step size of $10^{-3}$ initiated from the same 50k particles approximating the stationary distribution obtained by simulating EM with a step size of $10^{-6}$ until convergence.
We compute the MSE between the continuous-time process and the Markov chain with the same Brownian motion for $300$ iterations when we observe the MSE curve plateaus. We approximate the continuous-time process by simulating using the EM scheme with a step size of $10^{-6}$ similar to the setting in~\cite{rossler2010runge}. To obtain final results, we average across ten independent runs.
We note that the MSE upper bounds $W_2$ due to the latter being an infimum over all couplings. Hence, the MSE value serves as an indication of the convergence performance in $W_2$.

Figure~\ref{fig:strongly_convex} (c) shows that for $\beta=0.33$, $\gamma=0.5$ and $d=1$, when simulating from a good approximation to the target distribution with the same step size, the MSE of SRK-ID remains small, whereas the MSE of EM converges to a larger value. 
However, this improvement diminishes as the dimensionality of the sampling problem increases. 
We report additional results with other parameter settings in Appendix~\ref{app:srk-id_diminish}.
Notably, we did not observe significant differences in the estimated squared $W_2$ values. 
We suspect this is due to the discrepancy being dominated by the bias of our estimator.

\section{Discussion}
We established convergence rates of samplings algorithm obtained by discretizing It\^o diffusions with exponential \wtwo-contraction based on local properties of numerical schemes. 
The user-friendly conditions promote one to derive rates based on the uniform orders of the local deviation.
In addition, we showed that discretizing diffusions with SRK schemes leads to improved rates in $W_2$ for both strongly convex potentials and a certain class of non-convex potentials.

Despite focusing on SRK methods, Theorem~\ref{theo:master} can be used to obtain convergence rates for other classes of schemes. 
For the underdamped Langevin diffusion, quasi-symplectic schemes that rely on Runge-Kutta-type updates can achieve mean-square order 2.0 and beyond~\cite{milstein2003quasi}. 
For general It\^o diffusions, there exist schemes of mean-square order 1.5 and beyond, using the Fourier-Legendre series to approximate the L\'evy area~\cite{kuznetsov2018explicit}. 

Compared to some existing proofs for convergence rates in $W_2$ (e.g. \cite[Thm. 1]{cheng2017underdamped}), our Theorem~\ref{theo:master} requires two conditions (the uniform local mean and mean-square deviation bounds), neither of which can be eliminated in order to obtain a tight convergence bound. The uniform local mean deviation appears in our proof due to a direct expansion of the squared $2$-norm. This can be thought of a natural consequence of our convergence bounds being based on $W_2$. An avenue of interest is to see whether additional conditions can be identified to obtain refined bounds in $W_p$ for even integer $p > 2$.

Another direction of interest is to relax the $W_2$-contraction condition on the diffusion to $W_1$-contraction or $W_1$-decay. 
This would enable us to leverage results based on distant dissipativity, and consequently allow us to sample from a wider class of non-convex potentials~\cite{eberle2017couplings}.
Orthogonally, for the overdamped Langevin diffusion, the $W_2$-contraction condition may be relaxed to a log-Sobolev inequality condition on the target measure, if the discretization analysis is adapted to be based on the KL divergence~\cite{vempala2019rapid,raginsky2017non}. 
This would also broaden the class of non-convex potentials from which we can sample with theoretical guarantees.

Parallel to studying sampling from a mean-square convergence aspect, works in numerical analysis have established convergence results in the weak sense for SRK schemes applied to ergodic SDEs with techniques as aromatic trees and B-series~\cite{laurent2017exotic,vilmart2015postprocessed}. However, moment bounds in these works are proven by generic arguments (see e.g. \cite[Lem. 2.2.2]{milstein2013stochastic}), and reasoning about the rate's dimension dependence becomes less obvious. Refined non-asymptotic convergence bounds would provide more insight for these algorithms' performance on practical problems.

Lastly, the convergence results in $W_2$ for SRK-LD and SRK-ID can be augmented to yield generalization bounds for optimization when the excess risk is characterized using the Gibbs distribution~\cite{raginsky2017non}.

\section*{Acknowledgments}
We thank Mufan Li for insightful discussions, and
Jimmy Ba, David Duvenaud and Taiji Suzuki for helpful comments on an early draft of this work.
We also thank anonymous reviewers for helpful suggestions.
MAE is partially funded by NSERC and CIFAR AI Chairs program at Vector Institute.

\bibliography{main}

\begin{thebibliography}{69}
\providecommand{\natexlab}[1]{#1}
\providecommand{\url}[1]{\texttt{#1}}
\expandafter\ifx\csname urlstyle\endcsname\relax
  \providecommand{\doi}[1]{doi: #1}\else
  \providecommand{\doi}{doi: \begingroup \urlstyle{rm}\Url}\fi

\bibitem[Anderson and Mattingly(2009)]{anderson2009weak}
David~F Anderson and Jonathan~C Mattingly.
\newblock A weak trapezoidal method for a class of stochastic differential
  equations.
\newblock \emph{arXiv preprint arXiv:0906.3475}, 2009.

\bibitem[Brooks et~al.(2011)Brooks, Gelman, Jones, and
  Meng]{brooks2011handbook}
Steve Brooks, Andrew Gelman, Galin Jones, and Xiao-Li Meng.
\newblock \emph{Handbook of {Markov chain Monte Carlo}}.
\newblock CRC press, 2011.

\bibitem[Burrage et~al.(2004)Burrage, Burrage, and Tian]{burrage2004numerical}
Kevin Burrage, PM~Burrage, and Tianhai Tian.
\newblock Numerical methods for strong solutions of stochastic differential
  equations: an overview.
\newblock \emph{Proceedings of the Royal Society of London. Series A:
  Mathematical, Physical and Engineering Sciences}, 460\penalty0
  (2041):\penalty0 373--402, 2004.

\bibitem[Calogero(2012)]{calogero2012exponential}
Simone Calogero.
\newblock Exponential convergence to equilibrium for kinetic {Fokker-Planck}
  equations.
\newblock \emph{Communications in Partial Differential Equations}, 37\penalty0
  (8):\penalty0 1357--1390, 2012.

\bibitem[Chen et~al.(2015)Chen, Ding, and Carin]{chen2015convergence}
Changyou Chen, Nan Ding, and Lawrence Carin.
\newblock On the convergence of stochastic gradient {MCMC} algorithms with
  high-order integrators.
\newblock In \emph{Advances in Neural Information Processing Systems}, pages
  2278--2286, 2015.

\bibitem[Chen et~al.(2019)Chen, Dwivedi, Wainwright, and Yu]{chen2019fast}
Yuansi Chen, Raaz Dwivedi, Martin~J Wainwright, and Bin Yu.
\newblock Fast mixing of {Metropolized Hamiltonian Monte Carlo}: Benefits of
  multi-step gradients.
\newblock \emph{arXiv preprint arXiv:1905.12247}, 2019.

\bibitem[Cheng and Bartlett(2017)]{cheng2017convergence}
Xiang Cheng and Peter Bartlett.
\newblock Convergence of {Langevin} {MCMC} in {KL-divergence}.
\newblock \emph{arXiv preprint arXiv:1705.09048}, 2017.

\bibitem[Cheng et~al.(2017)Cheng, Chatterji, Bartlett, and
  Jordan]{cheng2017underdamped}
Xiang Cheng, Niladri~S Chatterji, Peter~L Bartlett, and Michael~I Jordan.
\newblock Underdamped {Langevin} {MCMC}: A non-asymptotic analysis.
\newblock \emph{arXiv preprint arXiv:1707.03663}, 2017.

\bibitem[Cheng et~al.(2018)Cheng, Chatterji, Abbasi-Yadkori, Bartlett, and
  Jordan]{cheng2018sharp}
Xiang Cheng, Niladri~S Chatterji, Yasin Abbasi-Yadkori, Peter~L Bartlett, and
  Michael~I Jordan.
\newblock Sharp convergence rates for {Langevin} dynamics in the nonconvex
  setting.
\newblock \emph{arXiv preprint arXiv:1805.01648}, 2018.

\bibitem[Chwialkowski et~al.(2016)Chwialkowski, Strathmann, and
  Gretton]{chwialkowski2016kernel}
Kacper Chwialkowski, Heiko Strathmann, and Arthur Gretton.
\newblock A kernel test of goodness of fit.
\newblock JMLR: Workshop and Conference Proceedings, 2016.

\bibitem[Dalalyan(2017)]{dalalyan2017theoretical}
Arnak~S Dalalyan.
\newblock Theoretical guarantees for approximate sampling from smooth and
  log-concave densities.
\newblock \emph{Journal of the Royal Statistical Society: Series B (Statistical
  Methodology)}, 79\penalty0 (3):\penalty0 651--676, 2017.

\bibitem[Dalalyan and Karagulyan(2019)]{dalalyan2019user}
Arnak~S Dalalyan and Avetik Karagulyan.
\newblock User-friendly guarantees for the {Langevin Monte Carlo} with
  inaccurate gradient.
\newblock \emph{Stochastic Processes and their Applications}, 2019.

\bibitem[Dalalyan and Riou-Durand(2018)]{dalalyan2018sampling}
Arnak~S Dalalyan and Lionel Riou-Durand.
\newblock On sampling from a log-concave density using kinetic {Langevin}
  diffusions.
\newblock \emph{arXiv preprint arXiv:1807.09382}, 2018.

\bibitem[Dalalyan and Tsybakov(2012)]{dalalyan2012sparse}
Arnak~S Dalalyan and Alexandre~B Tsybakov.
\newblock Sparse regression learning by aggregation and {Langevin Monte-Carlo}.
\newblock \emph{Journal of Computer and System Sciences}, 78\penalty0
  (5):\penalty0 1423--1443, 2012.

\bibitem[Ding et~al.(2014)Ding, Fang, Babbush, Chen, Skeel, and
  Neven]{ding2014bayesian}
Nan Ding, Youhan Fang, Ryan Babbush, Changyou Chen, Robert~D Skeel, and Hartmut
  Neven.
\newblock Bayesian sampling using stochastic gradient thermostats.
\newblock In \emph{Advances in neural information processing systems}, pages
  3203--3211, 2014.

\bibitem[Dobri{\'c} and Yukich(1995)]{dobric1995asymptotics}
V~Dobri{\'c} and Joseph~E Yukich.
\newblock Asymptotics for transportation cost in high dimensions.
\newblock \emph{Journal of Theoretical Probability}, 8\penalty0 (1):\penalty0
  97--118, 1995.

\bibitem[Dudley(1969)]{dudley1969speed}
RM~Dudley.
\newblock The speed of mean {Glivenko-Cantelli} convergence.
\newblock \emph{The Annals of Mathematical Statistics}, 40\penalty0
  (1):\penalty0 40--50, 1969.

\bibitem[Durmus and Moulines(2016)]{durmus2016high}
Alain Durmus and Eric Moulines.
\newblock High-dimensional {Bayesian} inference via the unadjusted {Langevin}
  algorithm.
\newblock \emph{arXiv preprint arXiv:1605.01559}, 2016.

\bibitem[Durmus et~al.(2016)Durmus, Simsekli, Moulines, Badeau, and
  Richard]{durmus2016stochastic}
Alain Durmus, Umut Simsekli, Eric Moulines, Roland Badeau, and Ga{\"e}l
  Richard.
\newblock Stochastic gradient {Richardson-Romberg} {Markov chain Monte Carlo}.
\newblock In \emph{Advances in Neural Information Processing Systems}, pages
  2047--2055, 2016.

\bibitem[Durmus et~al.(2018)Durmus, Moulines, and Pereyra]{durmus2018efficient}
Alain Durmus, Eric Moulines, and Marcelo Pereyra.
\newblock Efficient {Bayesian} computation by proximal {Markov chain Monte
  Carlo}: when {Langevin} meets {moreau}.
\newblock \emph{SIAM Journal on Imaging Sciences}, 11\penalty0 (1):\penalty0
  473--506, 2018.

\bibitem[Dwivedi et~al.(2018)Dwivedi, Chen, Wainwright, and Yu]{dwivedi2018log}
Raaz Dwivedi, Yuansi Chen, Martin~J Wainwright, and Bin Yu.
\newblock Log-concave sampling: {Metropolis-Hastings} algorithms are fast!
\newblock \emph{arXiv preprint arXiv:1801.02309}, 2018.

\bibitem[Eberle(2016)]{eberle2016reflection}
Andreas Eberle.
\newblock Reflection couplings and contraction rates for diffusions.
\newblock \emph{Probability theory and related fields}, 166\penalty0
  (3-4):\penalty0 851--886, 2016.

\bibitem[Eberle et~al.(2017)Eberle, Guillin, and Zimmer]{eberle2017couplings}
Andreas Eberle, Arnaud Guillin, and Raphael Zimmer.
\newblock Couplings and quantitative contraction rates for {Langevin} dynamics.
\newblock \emph{arXiv preprint arXiv:1703.01617}, 2017.

\bibitem[Erdogdu et~al.(2018)Erdogdu, Mackey, and Shamir]{erdogdu2018global}
Murat~A Erdogdu, Lester Mackey, and Ohad Shamir.
\newblock Global non-convex optimization with discretized diffusions.
\newblock In \emph{Advances in Neural Information Processing Systems}, pages
  9694--9703, 2018.

\bibitem[Flamary and Courty(2017)]{flamary2017pot}
R{'e}mi Flamary and Nicolas Courty.
\newblock {POT} python optimal transport library, 2017.
\newblock URL \url{https://github.com/rflamary/POT}.

\bibitem[Gelfand and Mitter(1991)]{gelfand1991recursive}
Saul~B Gelfand and Sanjoy~K Mitter.
\newblock Recursive stochastic algorithms for global optimization in r\^{}d.
\newblock \emph{SIAM Journal on Control and Optimization}, 29\penalty0
  (5):\penalty0 999--1018, 1991.

\bibitem[Gelman et~al.(2013)Gelman, Stern, Carlin, Dunson, Vehtari, and
  Rubin]{gelman2013bayesian}
Andrew Gelman, Hal~S Stern, John~B Carlin, David~B Dunson, Aki Vehtari, and
  Donald~B Rubin.
\newblock \emph{Bayesian data analysis}.
\newblock Chapman and Hall/CRC, 2013.

\bibitem[Gentil et~al.(2015)Gentil, L{\'e}onard, and Ripani]{gentil2015analogy}
Ivan Gentil, Christian L{\'e}onard, and Luigia Ripani.
\newblock About the analogy between optimal transport and minimal entropy.
\newblock \emph{arXiv preprint arXiv:1510.08230}, 2015.

\bibitem[Gorham and Mackey(2017)]{gorham2017measuring}
Jackson Gorham and Lester Mackey.
\newblock Measuring sample quality with kernels.
\newblock In \emph{Proceedings of the 34th International Conference on Machine
  Learning-Volume 70}, pages 1292--1301. JMLR. org, 2017.

\bibitem[Gorham et~al.(2016)Gorham, Duncan, Vollmer, and
  Mackey]{gorham2016measuring}
Jackson Gorham, Andrew~B Duncan, Sebastian~J Vollmer, and Lester Mackey.
\newblock Measuring sample quality with diffusions.
\newblock \emph{arXiv preprint arXiv:1611.06972}, 2016.

\bibitem[Gretton et~al.(2012)Gretton, Borgwardt, Rasch, Sch{\"o}lkopf, and
  Smola]{gretton2012kernel}
Arthur Gretton, Karsten~M Borgwardt, Malte~J Rasch, Bernhard Sch{\"o}lkopf, and
  Alexander Smola.
\newblock A kernel two-sample test.
\newblock \emph{Journal of Machine Learning Research}, 13\penalty0
  (Mar):\penalty0 723--773, 2012.

\bibitem[Kloeden and Platen(2013)]{kloeden2013numerical}
Peter~E Kloeden and Eckhard Platen.
\newblock \emph{Numerical solution of stochastic differential equations},
  volume~23.
\newblock Springer Science \& Business Media, 2013.

\bibitem[Kloeden et~al.(1992)Kloeden, Platen, and
  Wright]{kloeden1992approximation}
Peter~E Kloeden, Eckhard Platen, and IW~Wright.
\newblock The approximation of multiple stochastic integrals.
\newblock \emph{Stochastic analysis and applications}, 10\penalty0
  (4):\penalty0 431--441, 1992.

\bibitem[Kuznetsov(2018)]{kuznetsov2018explicit}
Dmitriy~F Kuznetsov.
\newblock Explicit one-step strong numerical methods of order 2.5 for {It\^o}
  stochastic differential equations, based on the unified {Taylor-It\^o} and
  {Taylor-Stratonovich} expansions.
\newblock \emph{arXiv preprint arXiv:1802.04844}, 2018.

\bibitem[Laurent and Vilmart(2017)]{laurent2017exotic}
Adrien Laurent and Gilles Vilmart.
\newblock Exotic aromatic b-series for the study of long time integrators for a
  class of ergodic {SDEs}.
\newblock \emph{arXiv preprint arXiv:1707.02877}, 2017.

\bibitem[Liu et~al.(2016)Liu, Lee, and Jordan]{liu2016kernelized}
Qiang Liu, Jason Lee, and Michael Jordan.
\newblock A kernelized stein discrepancy for goodness-of-fit tests.
\newblock In \emph{International Conference on Machine Learning}, pages
  276--284, 2016.

\bibitem[Ma et~al.(2015)Ma, Chen, and Fox]{ma2015complete}
Yi-An Ma, Tianqi Chen, and Emily Fox.
\newblock A complete recipe for stochastic gradient {MCMC}.
\newblock In \emph{Advances in Neural Information Processing Systems}, pages
  2917--2925, 2015.

\bibitem[Ma et~al.(2018)Ma, Chen, Jin, Flammarion, and Jordan]{ma2018sampling}
Yi-An Ma, Yuansi Chen, Chi Jin, Nicolas Flammarion, and Michael~I Jordan.
\newblock Sampling can be faster than optimization.
\newblock \emph{arXiv preprint arXiv:1811.08413}, 2018.

\bibitem[Ma et~al.(2019)Ma, Chatterji, Cheng, Flammarion, Bartlett, and
  Jordan]{ma2019there}
Yi-An Ma, Niladri Chatterji, Xiang Cheng, Nicolas Flammarion, Peter Bartlett,
  and Michael~I Jordan.
\newblock Is there an analog of {Nesterov} acceleration for {MCMC}?
\newblock \emph{arXiv preprint arXiv:1902.00996}, 2019.

\bibitem[MacKay and Mac~Kay(2003)]{mackay2003information}
David~JC MacKay and David~JC Mac~Kay.
\newblock \emph{Information theory, inference and learning algorithms}.
\newblock Cambridge university press, 2003.

\bibitem[Mao(2007)]{mao2007stochastic}
Xuerong Mao.
\newblock \emph{Stochastic differential equations and applications}.
\newblock Elsevier, 2007.

\bibitem[Mattingly et~al.(2002)Mattingly, Stuart, and
  Higham]{mattingly2002ergodicity}
Jonathan~C Mattingly, Andrew~M Stuart, and Desmond~J Higham.
\newblock Ergodicity for {SDEs} and approximations: locally lipschitz vector
  fields and degenerate noise.
\newblock \emph{Stochastic processes and their applications}, 101\penalty0
  (2):\penalty0 185--232, 2002.

\bibitem[Metropolis et~al.(1953)Metropolis, Rosenbluth, Rosenbluth, Teller, and
  Teller]{metropolis1953equation}
Nicholas Metropolis, Arianna~W Rosenbluth, Marshall~N Rosenbluth, Augusta~H
  Teller, and Edward Teller.
\newblock Equation of state calculations by fast computing machines.
\newblock \emph{The journal of chemical physics}, 21\penalty0 (6):\penalty0
  1087--1092, 1953.

\bibitem[Meyn and Tweedie(2012)]{meyn2012markov}
Sean~P Meyn and Richard~L Tweedie.
\newblock \emph{Markov chains and stochastic stability}.
\newblock Springer Science \& Business Media, 2012.

\bibitem[Milstein and Tretyakov(2003)]{milstein2003quasi}
GN~Milstein and Michael~V Tretyakov.
\newblock Quasi-symplectic methods for {Langevin}-type equations.
\newblock \emph{IMA journal of numerical analysis}, 23\penalty0 (4):\penalty0
  593--626, 2003.

\bibitem[Milstein and Tretyakov(2013)]{milstein2013stochastic}
Grigori~Noah Milstein and Michael~V Tretyakov.
\newblock \emph{Stochastic numerics for mathematical physics}.
\newblock Springer Science \& Business Media, 2013.

\bibitem[Mou et~al.(2019)Mou, Ma, Wainwright, Bartlett, and
  Jordan]{mou2019high}
Wenlong Mou, Yi-An Ma, Martin~J Wainwright, Peter~L Bartlett, and Michael~I
  Jordan.
\newblock High-order {Langevin} diffusion yields an accelerated {MCMC}
  algorithm.
\newblock \emph{arXiv preprint arXiv:1908.10859}, 2019.

\bibitem[{\O}ksendal(2003)]{oksendal2003stochastic}
Bernt {\O}ksendal.
\newblock Stochastic differential equations.
\newblock In \emph{Stochastic differential equations}, pages 65--84. Springer,
  2003.

\bibitem[Peyr{\'e} et~al.(2019)Peyr{\'e}, Cuturi,
  et~al.]{peyre2019computational}
Gabriel Peyr{\'e}, Marco Cuturi, et~al.
\newblock Computational optimal transport.
\newblock \emph{Foundations and Trends{\textregistered} in Machine Learning},
  11\penalty0 (5-6):\penalty0 355--607, 2019.

\bibitem[Raginsky et~al.(2017)Raginsky, Rakhlin, and
  Telgarsky]{raginsky2017non}
Maxim Raginsky, Alexander Rakhlin, and Matus Telgarsky.
\newblock Non-convex learning via stochastic gradient {Langevin} dynamics: a
  nonasymptotic analysis.
\newblock \emph{arXiv preprint arXiv:1702.03849}, 2017.

\bibitem[Roberts et~al.(1996)Roberts, Tweedie, et~al.]{roberts1996exponential}
Gareth~O Roberts, Richard~L Tweedie, et~al.
\newblock Exponential convergence of {Langevin} distributions and their
  discrete approximations.
\newblock \emph{Bernoulli}, 2\penalty0 (4):\penalty0 341--363, 1996.

\bibitem[R{\"o}{\ss}ler(2010)]{rossler2010runge}
Andreas R{\"o}{\ss}ler.
\newblock Runge--kutta methods for the strong approximation of solutions of
  stochastic differential equations.
\newblock \emph{SIAM Journal on Numerical Analysis}, 48\penalty0 (3):\penalty0
  922--952, 2010.

\bibitem[Sabanis and Zhang(2018)]{sabanis2018higher}
Sotirios Sabanis and Ying Zhang.
\newblock Higher order {Langevin Monte Carlo} algorithm.
\newblock \emph{arXiv preprint arXiv:1808.00728}, 2018.

\bibitem[Sabanis and Zhang(2019)]{sabanis2019explicit}
Sotirios Sabanis and Ying Zhang.
\newblock On explicit order 1.5 approximations with varying coefficients: the
  case of super-linear diffusion coefficients.
\newblock \emph{Journal of Complexity}, 50:\penalty0 84--115, 2019.

\bibitem[Sejdinovic et~al.(2013)Sejdinovic, Sriperumbudur, Gretton, Fukumizu,
  et~al.]{sejdinovic2013equivalence}
Dino Sejdinovic, Bharath Sriperumbudur, Arthur Gretton, Kenji Fukumizu, et~al.
\newblock Equivalence of distance-based and {RKHS}-based statistics in
  hypothesis testing.
\newblock \emph{The Annals of Statistics}, 41\penalty0 (5):\penalty0
  2263--2291, 2013.

\bibitem[Shen and Lee(2019)]{shen2019randomized}
Ruoqi Shen and Yin~Tat Lee.
\newblock The randomized midpoint method for log-concave sampling.
\newblock \emph{arXiv preprint arXiv:1909.05503}, 2019.

\bibitem[Simon(2007)]{simon2007probability}
Marvin~K Simon.
\newblock \emph{Probability distributions involving Gaussian random variables:
  A handbook for engineers and scientists}.
\newblock Springer Science \& Business Media, 2007.

\bibitem[Sz{\'e}kely(2003)]{szekely2003statistics}
G{\'a}bor~J Sz{\'e}kely.
\newblock E-statistics: The energy of statistical samples.
\newblock \emph{Bowling Green State University, Department of Mathematics and
  Statistics Technical Report}, 3\penalty0 (05):\penalty0 1--18, 2003.

\bibitem[Sz{\'e}kely and Rizzo(2013)]{szekely2013energy}
G{\'a}bor~J Sz{\'e}kely and Maria~L Rizzo.
\newblock Energy statistics: A class of statistics based on distances.
\newblock \emph{Journal of statistical planning and inference}, 143\penalty0
  (8):\penalty0 1249--1272, 2013.

\bibitem[Varadarajan(1958)]{varadarajan1958convergence}
Veeravalli~S Varadarajan.
\newblock On the convergence of sample probability distributions.
\newblock \emph{Sankhy{\=a}: The Indian Journal of Statistics (1933-1960)},
  19\penalty0 (1/2):\penalty0 23--26, 1958.

\bibitem[Vempala and Wibisono(2019)]{vempala2019rapid}
Santosh~S Vempala and Andre Wibisono.
\newblock Rapid convergence of the unadjusted {Langevin} algorithm: Log-sobolev
  suffices.
\newblock \emph{arXiv preprint arXiv:1903.08568}, 2019.

\bibitem[Villani(2008)]{villani2008optimal}
C{\'e}dric Villani.
\newblock \emph{Optimal transport: old and new}, volume 338.
\newblock Springer Science \& Business Media, 2008.

\bibitem[Vilmart(2015)]{vilmart2015postprocessed}
Gilles Vilmart.
\newblock Postprocessed integrators for the high order integration of ergodic
  {SDEs}.
\newblock \emph{SIAM Journal on Scientific Computing}, 37\penalty0
  (1):\penalty0 A201--A220, 2015.

\bibitem[Weed and Bach(2017)]{weed2017sharp}
Jonathan Weed and Francis Bach.
\newblock Sharp asymptotic and finite-sample rates of convergence of empirical
  measures in wasserstein distance.
\newblock \emph{arXiv preprint arXiv:1707.00087}, 2017.

\bibitem[Welling and Teh(2011)]{welling2011bayesian}
Max Welling and Yee~W Teh.
\newblock Bayesian learning via stochastic gradient {Langevin} dynamics.
\newblock In \emph{Proceedings of the 28th international conference on machine
  learning (ICML-11)}, pages 681--688, 2011.

\bibitem[Wiktorsson et~al.(2001)]{wiktorsson2001joint}
Magnus Wiktorsson et~al.
\newblock Joint characteristic function and simultaneous simulation of iterated
  {It\^o} integrals for multiple independent {Brownian} motions.
\newblock \emph{The Annals of Applied Probability}, 11\penalty0 (2):\penalty0
  470--487, 2001.

\bibitem[Xu et~al.(2018)Xu, Chen, Zou, and Gu]{xu2018global}
Pan Xu, Jinghui Chen, Difan Zou, and Quanquan Gu.
\newblock Global convergence of {Langevin} dynamics based algorithms for
  nonconvex optimization.
\newblock In \emph{Advances in Neural Information Processing Systems}, pages
  3122--3133, 2018.

\bibitem[Zhang et~al.(2018)Zhang, Mokhtari, Sra, and
  Jadbabaie]{zhang2018direct}
Jingzhao Zhang, Aryan Mokhtari, Suvrit Sra, and Ali Jadbabaie.
\newblock Direct {Runge-Kutta} discretization achieves acceleration.
\newblock In \emph{Advances in Neural Information Processing Systems}, pages
  3904--3913, 2018.

\bibitem[Zou et~al.(2019)Zou, Xu, and Gu]{zou2019sampling}
Difan Zou, Pan Xu, and Quanquan Gu.
\newblock Sampling from non-log-concave distributions via variance-reduced
  gradient {Langevin} dynamics.
\newblock In \emph{The 22nd International Conference on Artificial Intelligence
  and Statistics}, pages 2936--2945, 2019.

\end{thebibliography}
\bibliographystyle{plainnat}

\newpage
\appendix

\newpage
\section{Proof of Theorem~\ref{theo:master}} \label{proof:master}
\begin{proof}
Let $\{X_t\}_{t\ge0}$ denote the continuous-time process defined by the SDE~\eqref{eq:continuous_general} 
initiated from the target stationary distribution, driven by the Brownian motion $\{B_t\}_{t\ge0}$. 
Since the continuous-time transition kernel preserves the stationary distribution,
the marginal distribution of $\{X_t\}_{t\ge0}$ remains to be the stationary distribution for all $t \ge 0$. 

We denote by $t_k \; (k =0,1,\dots)$ the timestamps of the Markov chain obtained by discretizing the 
continuous-time process with a numerical integration scheme and assume the Markov chain has a constant step size $h$ that 
satisfies the conditions in the theorem statement. 
We denote by $\tX_k$ the $k$th iterate of the Markov chain.
In the following, we derive a recursion for the quantity 
\eqn{
    A_k = \Exp{
        \normtwo{X_{t_k} - \tX_k}^2
    }^{1/2}.
}
Fix $k \in \mathbb{N}$. 
We define the process $\{\bar{X}_t\}_{t \ge 0}$ such that it is the Markov chain until $t_k$, starting from which it follows the continuous-time process defined by the SDE~\eqref{eq:continuous_general}. 
We let $\{\bar{X}_t\}_{t \ge 0}$ and the Markov chain $\tX_k\; (k =0,1,\dots)$ share the same Brownian motion $\{\bar{B}_t\}_{t\ge 0}$. 
Suppose $\{\F_t\}_{t \ge 0}$ is a filtration to which both $\{B_t\}_{t\ge0}$ and $\{\bar{B}_t\}_{t\ge 0}$ are adapted.
Conditional on $\F_{t_k}$, let 
$X_{t_{k+1}}$ and $\bar{X}_{t_{k+1}}$ be coupled such that 
\eq{
    \Exp{
        \normtwo{ X_{t_{k+1}} - \barX_{t_{k+1}} }^2
        | \F_{t_k}
    }
    \le&
        e^{-2 \alpha h} \normtwo{ X_{t_k} - \barX_{t_k} }^2
    .\label{eq:exp_w2_decay}
}
This we can achieve due to exponential \wtwo-contraction. 
We define the process $\{Z_s\}_{s\ge t_k}$ as follows
\eqn{
    Z_s = 
        \bracks{
            X_{s} - \barX_s
        } 
        - 
        \bracks{
            X_{t_k} - \barX_{t_k}
        }.
}
Note $\int_{t_k}^{t_k + t} \sigma(X_s) \dB_s - \int_{t_k}^{t_k + t} \sigma(\barX_s) \dee\bar{B}_s$ is a Martingale w.r.t. $\{\F_{t_k + t}\}_{t \ge 0}$, since it is adapted and the two component It\^o integrals are Martingales w.r.t. the considered filtration. 
By Fubini's theorem, we switch the order of integrals and obtain
\eq{
    \Exp{ Z_{t_{k+1}} | \F_{t_k} } = \int_{t_k}^{t_{k+1}} \Exp{ 
        b(X_s) - b(\barX_s) 
        | \F_{t_k}
    } \ds.
}
By Jensen's inequality,
\eq{
    \normtwo{ \Exp{ Z_{t_{k+1}} | \F_{t_k} } }^2
    \le&
        h \int_{t_k}^{t_{k+1}} \Exp{
            \normtwo{ b(X_s) - b(\barX_s) }^2
            | \F_{t_k}
        } \ds \\
    \le&
        \mu_1(b)^2 h \int_{t_k}^{t_{k+1}}
        \Exp{
            \normtwo{X_s - \barX_s}^2
            | \F_{t_k}
        } \ds. \label{eq:zs_ineq}
}
For $s \in [t_k , t_k + h]$, 
by Young's inequality, Jensen's inequality, and It\^o isometry,
\eqn{
    &\Exp{
        \normtwo{ X_s - \barX_s }^2
        | \F_{t_k}
    }\\
    =&
        \Exp{
            \normtwo{ 
                X_{t_k} - \barX_{t_k} + 
                \int_{t_k}^s \bracks{
                    b(X_u) - b(\barX_u)
                } \du + 
                \int_{t_k}^s \bracks{
                    \sigma(X_u) - \sigma(\barX_u)
                } \dB_u
            }^2
            | \F_{t_k}
        } \\
    \le&
        4 \normtwo{ X_{t_k} - \barX_{t_k} }^2
        + 4 (s - t_k) \int_{t_k}^s \Exp{
            \normtwo{
                b(X_u) - b(\barX_u)
            }^2
            | \F_{t_k}
        } \du
        \\&
        + 4 \int_{t_k}^s \Exp{
            \normf{
                \sigma(X_u) - \sigma(\barX_u)
            }^2
            | \F_{t_k}
        } \du \\
    \le&
        4 \normtwo{ X_{t_k} - \barX_{t_k} }^2
        + 
        4 (s - t_k) \mu_1(b)^2 \int_{t_k}^s \Exp{
            \normtwo{ X_u - \barX_u }^2
            | \F_{t_k}
        } \du
        \\&
        +
        4 \mu_1^{\mathrm{F}}(\sigma)^2 \int_{t_k}^s \Exp{
            \normtwo{ X_u - \barX_u }^2
            | \F_{t_k}
        } u \\
    \le&
        4 \normtwo{ X_{t_k} - \barX_{t_k} }^2
        + 
        4 \bracks{ \mu_1(b)^2 + \mu_1^{\mathrm{F}}(\sigma)^2 }
        \int_{t_k}^s \Exp{
            \normtwo{
                X_u - \barX_u
            }^2
            | \F_{t_k}
        } \du. 
}
By the integral form of Gr\"onwall's inequality for continuous functions,
\eqn{
    \Exp{
        \normtwo{ X_s - \barX_s }^2
        | \F_{t_k}
    }
    \le&
        4 \exp \bracks{
            4 \bracks{ \mu_1(b)^2 + \mu_1^{\mathrm{F}}(\sigma)^2 } (s - t_k) 
        }
        \normtwo{ X_{t_k} - \barX_{t_k} }^2.
}
Plugging this result into~\eqref{eq:zs_ineq}, by 
$h < 1 / \bracks{ 8 \mu_1(b)^2 + 8 \mu_1^{\mathrm{F}}(\sigma)^2 }$, 
\eqn{
    \normtwo{
        \Exp{ 
            Z_{t_{k+1}} | \F_{t_k}
        }
    }^2
    \le&
        \frac{\mu_1(b)^2 h}{
            \mu_1(b)^2 +
            \mu_1^{\mathrm{F}}(\sigma)^2
        }
        \sbracks{
            \exp\bracks{
                4 \bracks{ \mu_1(b)^2 + \mu_1^{\mathrm{F}}(\sigma)^2 } h
            } - 1
        } 
        \normtwo{ X_{t_k} - \barX_{t_k} }^2 \\
    \le&
        \frac{8 \mu_1(b)^2 h^2}{
            \mu_1(b)^2 +
            \mu_1^{\mathrm{F}}(\sigma)^2
        }
        \bracks{
            \mu_1(b)^2 +
            \mu_1^{\mathrm{F}}(\sigma)^2
        }
        \normtwo{ X_{t_k} - \barX_{t_k} }^2 \\
    \le&
        8 \mu_1(b)^2 h^2
        \normtwo{ X_{t_k} - \barX_{t_k} }^2
    . \numberthis \label{eq:norm_of_exp}
}
By direct expansion, 
{
\footnotesize
\eq{
    \Exp{
        \normtwo{
            X_{t_{k+1}} - \barX_{t_{k+1}}
        }^2
        | \F_{t_k}
    }
    =&
        \normtwo{ X_{t_k} - \barX_{t_k} }^2
        + \Exp{
            \normtwo{ Z_{t_{k+1}} }^2
            | \F_{t_k}
        }
        + 2 \abracks{
            X_{t_k} - \barX_{t_k}, 
            \Exp{
                Z_{t_{k+1}}
                | \F_{t_k}
            }
        }. \label{eq:trivial_expansion}
}
}
Combining~\eqref{eq:exp_w2_decay} \eqref{eq:norm_of_exp} and \eqref{eq:trivial_expansion}, by the Cauchy-Schwarz inequality, 
\eqn{
    \Exp{
        \normtwo{ Z_{t_{k+1}} }^2
        | \F_{t_k}
    }
    \le&
        \bracks{
            e^{-2 \alpha h} - 1
        } 
        \normtwo{
            X_{t_k} - \barX_{t_k}
        }^2 
        - 
        2 \abracks{
            X_{t_k} - \barX_{t_k},
            \Exp{ Z_{t_{k+1}} | \F_{t_k} }
        }
        \\
    \le&
        2 \normtwo{ X_{t_k} - \barX_{t_k} }
        \normtwo{ \Exp{ Z_{t_{k+1}} | \F_{t_k} } } \\
    \le&
        8 \mu_1(b) h \normtwo{ X_{t_k} - \barX_{t_k} }^2 \\
    =&
        8 \mu_1(b) h \normtwo{ X_{t_k} - \tX_k }^2.
}
Hence,
\eq{
    \Exp{
        \normtwo{Z_{t_{k+1}}}^2
    }
    =
        \Exp{
            \Exp{
                \normtwo{ Z_{t_{k+1}} }^2
                | \F_{t_k}
            }
        }
    \le
        8 \mu_1(b) h \Exp{
            \normtwo{ X_{t_k} - \tX_k }^2
        }
    =
        8 \mu_1(b) h A_k^2.
}
Let $\lambda_3 = 8 \lambda_1^{1/2} \mu_1(b)^{1/2} + 2 \lambda_2^{1/2} $. 
Then, by the Cauchy–Schwarz inequality, we obtain a recursion 
\eq{
    A_{k+1}^2
    =&
        \Exp{ \normtwo{X_{t_{k+1}} - \tX_{k+1} }^2 } \notag \\
    =&
        \Exp{
            \normtwo{X_{t_{k+1}} - \barX_{t_{k+1}} + \barX_{t_{k+1}} - \tX_{k+1} }^2
        } \\
    =&
        \Exp{
            \normtwo{
                X_{t_{k+1}} - \barX_{t_{k+1}}
            }^2 + 
            \normtwo{
                \barX_{t_{k+1}} - \tX_{k+1}
            }^2 + 
            2 \abracks{
                X_{t_{k+1}} - \barX_{t_{k+1}}, \barX_{t_{k+1}} - \tX_{k+1}
            }
        } \\
    =&
        \Exp{
            \Exp{
                \normtwo{
                    X_{t_{k+1}} - \barX_{t_{k+1}}
                }^2 | \F_{t_k}
            }
        } 
        + \Exp{
            \Exp{
                \normtwo{
                    \barX_{t_{k+1}} - \tX_{k+1}
                }^2 | \F_{t_k}
            }
        }  \\
        &+
        2 \Exp{
            \Exp{
                \abracks{
                    X_{t_{k+1}} - \barX_{t_{k+1}}, \barX_{t_{k+1}} - \tX_{k+1}
                } | \F_{t_k}
            }
        } \\
    =&
        \Exp{
            \Exp{
                \normtwo{
                    X_{t_{k+1}} - \barX_{t_{k+1}}
                }^2 | \F_{t_k}
            }
        } +
        \Exp{
            \Exp{
                \normtwo{
                    \barX_{t_{k+1}} - \tX_{k+1}
                }^2 | \F_{t_k}
            }
        } \\
        &+
        2 \Exp{
            \abracks{
                X_{t_{k}} - \barX_{t_{k}},
                \Exp{
                    \barX_{t_{k+1}} - \tX_{k+1} | \F_{t_k} 
                }
            }
        }
        \\ &
        +
        2 \Exp{
            \abracks{
                Z_{t_{k+1}}, \barX_{t_{k+1}} - \tX_{k+1}
            }
        } \\
    \le&
        \Exp{
            \Exp{
                \normtwo{
                    X_{t_{k+1}} - \barX_{t_{k+1}}
                }^2 | \F_{t_k}
            }
        } +
        \Exp{
            \Exp{
                \normtwo{
                    \barX_{t_{k+1}} - \tX_{k+1}
                }^2 | \F_{t_k}
            }
        } \\
        &+
        2 \Exp{
            \normtwo{X_{t_k} - \barX_{t_k}}^2
        }^{1/2}
        \Exp{
            \normtwo{ 
                \Exp{
                    \barX_{t_{k+1}} - \tX_{k+1} | \F_{t_k} 
                }
            }^2
        }^{1/2}  \\
        &+
        2 \Exp{
            \normtwo{Z_{t_{k+1}}}^2
        }^{1/2}
        \Exp{
            \normtwo{\barX_{t_{k+1}} - \tX_{k+1}}^2
        }^{1/2} \\
    \le&
        e^{-2 \alpha h} A_k^2 + 
        \lambda_1 h^{2 p_1} + 
        2 \lambda_2^{1/2} h^{p_2} A_k + 
        8 \lambda_1^{1/2} \mu_1(b)^{1/2} h^{p_1 + 1/2} A_k
        \\
    \le&
        \bracks{1 - \alpha h} A_k^2 + 
        \lambda_3 h^{p_1 + 1/2} A_k + 
        \lambda_1 h^{2 p_1} \\
    \le&
        \bracks{1 - \alpha h } A_k^2 + 
        \frac{\alpha h}{2} A_k^2 + 
        \frac{8}{\alpha} \lambda_3^2 h^{2p_1} + \lambda_1 h^{2p_1} \\
    \le&
        \bracks{1 - \alpha h/2} A_k^2 + 
        \bracks{8\lambda_3^2 / \alpha + \lambda_1} h^{2 p_1}
    ,
    \label{eq:master_recursion}
}
where the third to last inequality follows from $e^{-2 \alpha h} < 1 - \alpha h$ when $\alpha h < 1/2$, and
the second to last inequality follows from the elementary relation below with the choice of $\kappa = \alpha/2$
\eqn{
    A_k h^{1/2} \cdot \lambda_3 h^{p_1}
    \le 
        \kappa A_k^2 h +
        \frac{4}{\kappa} \lambda_3^2 h^{2p_1}.
}
Let $\eta = 1 - \alpha h /2 \le e^{-\alpha h/2} \le 1$. 
By unrolling the recursion,
\eqn{
A_k^2
\le&
    \bracks{1 - \alpha h/2} A_{k-1}^2 + 
    \bracks{8\lambda_3^2 / \alpha + \lambda_1} h^{2 p_1} \notag \\
\le&
    \eta^k A_0^2 + 
    \bracks{1 + \eta + \dots + \eta^{k-1}} 
        \bracks{8\lambda_3^2 / \alpha + \lambda_1} h^{2 p_1} \notag \\
\le&
    \eta^k A_0^2 + 
    \bracks{8\lambda_3^2 / \alpha + \lambda_1} h^{2 p_1} / (1 - \eta) \notag \\
=&
    \eta^k A_0^2 + (16 \lambda_3^2/ \alpha^2  + 2 \lambda_1/\alpha)  h^{2p_1 - 1}.
}
Let $\nu_k$ and $\nu^*$ be the measures associated with the $k$th
iterate of the Markov chain
and the target distribution, respectively.
Since $W_2$ is defined as an infimum over all couplings,
\eqn{
W_2(\nu_k, \nu^*)
\le&
    A_k
\le
    e^{-\alpha h k / 4} A_0 + (16 \lambda_3^2/\alpha^2  + 2 \lambda_1 / \alpha)^{1/2} h^{p_1 - 1/2}.
}
To ensure $W_2$ is less than some small positive tolerance $\epsilon$,
we need only ensure the two terms in the above inequality are each less than $\epsilon /2$. Some simple calculations 
show that it suffices that
\eqn{
h <& 
\text{
    \footnotesize
    $
    \bracks{
        \frac{2}{\epsilon}
            \sqrt{
                \frac{64 ( 16 \lambda_1 \mu_1(b) + \lambda_2 ) }{\alpha^2}
                + 
                \frac{2 \lambda_1}{\alpha}
            }
    }^{-1 / (p_1 - 1/2)}
    \wedge
        \frac{1}{2\alpha}
    \wedge
        \frac{1}{8 \mu_1(b)^2 + 8 \mu_1^{\mathrm{F}}(\sigma)^2 }
    ,
    $
}  \numberthis \label{eq:master_proof_stepsize}\\
k >& 
\text{
    \footnotesize
    $
    \sbracks{
        \bracks{
            \frac{2}{\epsilon}
                \sqrt{ 
                    \frac{64 ( 16 \lambda_1 \mu_1(b) + \lambda_2 ) }{\alpha^2}
                    + 
                    \frac{2 \lambda_1}{\alpha}
                }
        }^{1 / (p_1 - 1/2) }
    \vee
        2\alpha
    \vee
        \bracks{
            8 \mu_1(b)^2 + 8\mu_1^{\mathrm{F}} (\sigma)^2
        }
    }
    \frac{4}{\alpha} \log \bracks{\frac{2A_0}{\epsilon} }
    $
}
.
}
Note that for small enough positive tolerance $\epsilon$, when the step size satisfies~\eqref{eq:master_proof_stepsize}, it suffices that 
\eqn{
    k = 
        \left\lceil
            \bracks{
                \frac{2}{\epsilon}
                    \sqrt{
                        \frac{64 ( 16 \lambda_1 \mu_1(b) + \lambda_2 ) }{\alpha^2 }
                        + 
                        \frac{2 \lambda_1}{\alpha}
                    }
            }^{1 / (p_1 - 1/2)}
            \frac{4}{\alpha} \log \bracks{
                \frac{2A_0}{\epsilon}
            }
        \right\rceil
        =
        \tilde{\mathcal{O}} (
            \epsilon^{ -1 / (p_1 - 1/2) }
        ).
}
\end{proof}
\section{Proof of Theorem~\ref{theo:srk_ld}} \label{app:srk_ld}
\subsection{Moment Bounds}
Verifying the order conditions in Theorem~\ref{theo:master} for SRK-LD requires bounding the second, fourth, and sixth moments of the Markov chain. 
In principle, one may employ an exponential moment bound argument using a Lyapunov function. 
However, in this case, the tightness of the final convergence bound may depend on the selection of the Lyapunov function, and reasoning about the dimension dependence can become less obvious. 
Here, we directly bound all the even moments by expanding the expression. 
Intuitively, one expects the $2n$th moments of the Markov chain iterates to be $\mathcal{O}(d^n)$.
The following proofs assume Lipschitz smoothness of the potential to a certain order and dissipativity.
\begin{defi}[Dissipativity] For constants 
$\alpha, \beta > 0$,
the diffusion satisfies the following
\eqn{
    \abracks{ \nabla f(x), x} \ge \frac{\alpha}{2} \norm{x}_2^2 - \beta, \quad \forall x \in \R^d.
}
\end{defi}
For the Langevin diffusion, dissipativity directly follows from strong convexity of the potential~\cite{erdogdu2018global}.
Here, $\alpha$ can be chosen as the strong convexity parameter, provided $\beta$ is an appropriate constant of order $\mathcal{O}(d)$.

Additionally, we assume the discretization has a constant step size $h$ and the timestamp
of the $k$th iterate is $t_k$ as per the proof of Theorem~\ref{theo:master}. 
To simplify notation, we define the following
\eqn{
\tilde{\nabla} f =& \frac{1}{2} \bracks{\nabla f(\tH_1) + \nabla f(\tH_2)}, \\
v_1 =& \sqrt{2} \bracks{ \frac{1}{2} + \frac{1}{\sqrt{6}} } \xi_{k+1} \sqrt{h}, \\ 
v_1' =& \sqrt{2} \bracks{ \frac{1}{2} - \frac{1}{\sqrt{6}} } \xi_{k+1} \sqrt{h}, \\
v_2 =& \frac{1}{\sqrt{6}} \eta_{k+1} \sqrt{h},
}
where $\xi_{k+1}, \eta_{k+1} \overset{\text{i.i.d.}}{\sim} \mathcal{N}(0, I_d)$ 
independent of $\tX_k$ for all $k \in \mathbb{N}$. 
We rewrite $\tH_1$ and $\tH_2$ as
\eqn{
\tH_1 & = \tX_k + \Delta \tH_1 = \tX_k + v_1 + v_2, \\
\tH_2 &= \tX_k + \Delta \tH_2 = \tX_k + v_1' + v_2 - \nabla f(\tX_k) h.
}

\subsubsection{Second Moment Bound}
\begin{lemm} \label{lemm:second_moment}
If the second moment of the initial iterate is finite, then
the second moments of Markov chain iterates defined in~\eqref{eq:srk} are 
uniformly bounded by a constant of order $\mathcal{O}(d)$, i.e.
\eqn{
    \Exp{\norm{\tX_k}_2^{2}} \le \lyap_{2}, \quad \text{for all} \; {k \in \mathbb{N}},
}
where $\lyap_2 = \Exp{ \norm{\tX_0}_2^2 } + N_6$, 
and constants $N_1$ to $N_6$ are given in the proof, if the step size
\eqn{
h < 1 
\wedge \frac{2d}{\pi_{2,2}(f)} 
\wedge \frac{2\pi_{2,1}(f)}{\pi_{2,2}(f)} 
\wedge \frac{\alpha}{4\mu_2(f) \pi_{2,2}(f)}
\wedge \frac{3\alpha}{2N_1 + 4}.
}
\end{lemm}
\begin{proof}
By direct computation,
\begin{align*}
    \norm{\tX_{k+1}}_2^2 
    =& 
        \norm{
            \tX_k - 
            \bracks{
                \nabla f(\tH_1) + \nabla f(\tH_2)
            } \frac{h}{2} + 
            2^{1/2} \xi_{k+1} h^{1/2}
        }_2^2 \notag \\
    =&
        \norm{\tX_k}_2^2 + 
        \norm{
            \nabla f(\tH_1) + \nabla f(\tH_2)
        }_2^2 \frac{h^2}{4} + 
        2 \norm{\xi_{k+1}}_2^2 h
        \\&
        - \abracks{
            \tX_k, 
            \nabla f(\tH_1 ) +\nabla f(\tH_2 )
        } h 
        \\&
        + 
        2^{3/2} \abracks{
            \tX_k,
            \xi_{k+1}
        } h^{1/2} \\
        &-
        2^{1/2}
        \abracks{
            \nabla f(\tH_1) + \nabla f(\tH_2),
            \xi_{k+1}
        } h^{3/2}.
\end{align*}
In the following, we bound each term in the expansion separately and obtain a recursion. 
To achieve this, we first upper bound the second moments of $\tH_1$ and $\tH_2$ for $ h < 2d \wedge 2\pi_{2,1}(f) / \pi_{2,2}(f) $,
\begin{align*}
    \Exp{
        \normtwo{\tH_1}^2 | \F_{t_k}
    }
    =&
        \normtwo{\tX_k}^2 + 
        \Exp{ \normtwo{v_1}^2 | \F_{t_k} } + 
        \Exp{ \normtwo{v_2}^2 | \F_{t_k} } 
    \le
        \normtwo{\tX_k}^2 + 
        3dh, \\
    \Exp{
        \normtwo{\tH_2}^2 | \F_{t_k}
    }
    =&
        \normtwo{\tX_k}^2 + \normtwo{\nabla f(\tX_k)}^2 h^2 + 
        \Exp{
            \normtwo{v_1'}^2 | \F_{t_k}
        } + 
        \Exp{
            \normtwo{v_2}^2 | \F_{t_k}
        }
        \\&+
        2\abracks{
            \tX_k,
            \nabla f(\tX_k)
        } h \\
    \le&
        \normtwo{\tX_k}^2 +
        \pi_{2,2}(f) \bracks{1 + \normtwo{\tX_k}^2 } h^2 + 
        dh + 
        2 \pi_{2,1}(f) \normtwo{\tX_k}^2 h \\
    \le&
        \normtwo{\tX_k}^2 +
        4 \pi_{2,1}(f) h \normtwo{\tX_k}^2 + 
        3d h. 
\end{align*}
Thus, 
\begin{align*}
    \Exp{
        \normtwo{ \nabla f(\tH_1) +  \nabla f(\tH_2) }^2
        | \F_{t_k}
    }
    \le&
    2 \Exp{
        \normtwo{ \nabla f(\tH_1) }^2 +
        \normtwo{ \nabla f(\tH_2) }^2 
        | \F_{t_k}
    } \\
    \le&
        2 \pi_{2,2} (f)\Exp{
            2 + \normtwo{ \tH_1 }^2  + \normtwo{ \tH_2 }^2
            | \F_{t_k}
        }\\
    =&
        N_1\normtwo{\tX_k}^2 + N_2, 
\end{align*}
where $N_1 = 2\pi_{2,2}(f) \bracks{2 + 4 \pi_{2,1}(f)}$ and $N_2 = 2\pi_{2,2} (f) \bracks{6d + 2}$. 

Additionally, by the Cauchy-Schwarz inequality,
\begin{align*}
    -\Exp{
        \abracks{
            \nabla f(\tH_1),
            \xi_{k+1}
        }
        | \F_{t_k}
    }
    \le&
        \Exp{
            \norm{ \nabla f(\tH_1 ) }_2
            \norm{\xi_{k+1}}_2
            | \F_{t_k}
        } \\
    \le&
        \Exp{
            \norm{ \nabla f(\tH_1)}_2^2
            | \F_{t_k}
        }^{{1}/{2}}
        \Exp{
            \norm{\xi_{k+1}}_2^2
        }^{{1}/{2}} \\
    \le&
        \sqrt{d \pi_{2,2} (f) } \bracks{
            1 + \Exp{ 
                \norm{\tH_1}_2^2 | \F_{t_k}
            }^{{1}/{2}}
        } \\
    \le&
        \sqrt{d \pi_{2,2} (f) }  \bracks{
            1 + \norm{\tX_k}_2 + \sqrt{3dh}
        }. \numberthis \label{eq:tilde_h_1}
\end{align*}
Similarly,
\begin{align*}
   -\Exp{
        \abracks{
            \nabla f(\tH_2),
            \xi_{k+1}
        }
        | \F_{t_k}
    }
    \le&
        \Exp{
            \normtwo{ \nabla f(\tH_2 ) }
            \normtwo{\xi_{k+1}}
            | \F_{t_k}
        } \\
    \le&
        \Exp{ \normtwo{ \nabla f(\tH_2)}^2 | \F_{t_k} }^{{1}/{2}}
        \Exp{ \normtwo{\xi_{k+1}}^2 | \F_{t_k} }^{{1}/{2}} \\
    \le&
        \sqrt{d \pi_{2,2} (f)}  \bracks{
            1 + \Exp{ 
                \norm{\tH_2}_2^2 | \F_{t_k}
            }^{{1}/{2}}
        } \\
    \le&
        \sqrt{d \pi_{2,2}  (f) } \Bigl(
            1 + \norm{\tX_k}_2 + 2 \sqrt{
                \pi_{2,1} (f) h
            } \norm{\tX_k}_2 + \sqrt{3dh}
        \Bigr). \numberthis \label{eq:tilde_h_2}
\end{align*}
Combining~\eqref{eq:tilde_h_1} and \eqref{eq:tilde_h_2}, we obtain the following using AM–GM,
\begin{align*}
    -2^{1/2} \Exp{
        \abracks{
            \nabla f(\tH_1) + \nabla f(\tH_2),
            \xi_{k+1}
        } 
        | \F_{t_k}
    } h^{3/2}
    \le&
        N_3 \norm{\tX_k}_2 h^{3/2} + N_4 \\
    \le&
        \frac{1}{2} \norm{\tX_k}_2^2 h^2 + \frac{N_3^2}{2} h + N_4 h^{3/2}.
\end{align*}
where $N_3 = 2 \sqrt{2 d\pi_{2,2} (f) } \bracks{1 + \sqrt{\pi_{2,1}(f) }} $ and 
$N_4 = 2\sqrt{2 d \pi_{2,2}(f) } \bracks{1 + \sqrt{3d}}$. 

Now, we lower bound the second moments of $\tH_1$ and $\tH_2$ by dissipativity,
\begin{align*}
    \Exp{
        \norm{\tH_1}_2^2 | \F_{t_k}
    }
    =& 
        \Exp{
            \normtwo{ \tX_k + v_1 + v_2 }^2 | \F_{t_k}
        } \\
    =&
        \normtwo{\tX_k}^2 + 
        \Exp{ \normtwo{v_1}^2 | \F_{t_k} } + 
        \Exp{ \normtwo{v_2}^2 | \F_{t_k} } 
    \ge
        \normtwo{\tX_k}^2, \numberthis \label{eq:lower_bound_tH_1}\\
    \Exp{
        \normtwo{\tH_2}^2 | \F_{t_k}
    }
    =& 
        \Exp{
            \normtwo{ \tX_k - \nabla f(\tX_k) h + v_1' + v_2 }^2
            | \F_{t_k}
        }  \\
    =&
        \normtwo{\tX_k}^2 + \normtwo{\nabla f(\tX_k)}^2 h^2 + 
        \Exp{
            \normtwo{v_1'}^2 | \F_{t_k}
        } + 
        \Exp{
            \normtwo{v_2}^2 | \F_{t_k}
        }
        \\ &+
        2\abracks{
            \tX_k,
            \nabla f(\tX_k)
        } h \\
    \ge&
        \normtwo{\tX_k}^2 + 2 \bracks{
            \frac{\alpha}{2} \normtwo{\tX_k}^2 - \beta
        } h \\
    \ge&
        \normtwo{\tX_k}^2 - 2\beta h.
\end{align*}
Additionally, by Stein's lemma for multivariate Gaussians, 
\eqn{
    \Exp{
        \abracks{
            \nabla f(\tH_1),
            v_1
        }
        | \F_{t_k}
    }
    =&
        2h \bracks{\frac{1}{2} + \frac{1}{\sqrt{6}}}^2 \Exp{
            \Delta (f) (\tH_1) | \F_{t_k}
        }
    \le
        2 d \mu_3(f) h, \\
    \Exp{
        \abracks{
            \nabla f(\tH_1),
            v_2
        }
        | \F_{t_k}
    }
    =&
        \frac{1}{6} h \Exp{
            \Delta(f) (\tH_1)
            | \F_{t_k}
        }
    \le
        \frac{1}{6} d \mu_3(f) h, \\
    \Exp{
        \abracks{
            \nabla f(\tH_2),
            v_1'
        }
        | \F_{t_k}
    }
    =&
        2 h \bracks{\frac{1}{2} - \frac{1}{\sqrt{6}}}^2 \Exp{
            \Delta(f) (\tH_2)
            | \F_{t_k}
        }
    \le
        d\mu_3(f) h, \\
    \Exp{
        \abracks{
            \nabla f(\tH_2),
            v_2
        }
        | \F_{t_k}
    }
    =&
        \frac{1}{6} h \Exp{
            \Delta(f) (\tH_2)
            | \F_{t_k}
        }
    \le
        d\mu_3(f) h.
}
Therefore, by dissipativity and the lower bound~\eqref{eq:lower_bound_tH_1},
\begin{align*}
    - \Exp{
        \abracks{
            \nabla f(\tH_1),
            \tX_k
        }
        | \F_{t_k}
    }
    =&
        -\Exp{
            \abracks{
                \nabla f(\tH_1),
                \tH_1
            }
            | \F_{t_k}
        }
        + \Exp{ 
            \abracks{
                \nabla f(\tH_1),
                v_1 + v_2
            }
            | \F_{t_k}
        } \\
    \le&
        - \frac{\alpha}{2} \Exp{
            \normtwo{ \tH_1 }^2
            | \F_{t_k}
        }+ \beta + 
        \Exp{
            \abracks{
                \nabla f(\tH_1),
                v_1 + v_2
            }
            | \F_{t_k}
        } \\
    \le&
        -\frac{\alpha}{2} \normtwo{\tX_k}^2 + 
        \beta + 
        3 d \mu_3(f) h. \numberthis \label{eq:diss_1}
\end{align*}
To bound the expectation of $-\big\langle \nabla f(\tH_2) , \tX_k \big\rangle$, we first bound the second moment of $\Delta \tH_2$,
\begin{align*}
    \Exp{
        \norm{\Delta \tH_2}_2^2 | \F_{t_k}
    }
    =&
        \Exp{
            \norm{
                -\nabla f(\tX_k) h + v_1' + v_2
            }_2^2
            | \F_{t_k}
        } \\
    =&
        \norm{\nabla f(\tX_k) }_2^2
        h^2 + 
        \Exp{
            \norm{v_1'}_2^2 
            | \F_{t_k}
        } + 
        \Exp{
            \norm{v_2}_2^2
            | \F_{t_k}
        } \\
    \le&
        \pi_{2, 2} (f)\bracks{1 + \norm{\tX_k}_2^2  } h^2 + d h. \numberthis \label{eq:norm_delta_h2}
\end{align*}
Notice the second equality above also implies
\eqn{
\normtwo{ \nabla f(\tX_k) } h \le \Exp { \normtwo{ \Delta \tH_2 }^2 | \F_{t_k}}^{1/2}. \numberthis \label{eq:norm_grad_f}
}
By Taylor's Theorem with the remainder in integral form,
\eqn{
    \nabla f(\tH_2)
    = \nabla f(\tX_k) + R(t_{k+1})
    = 
        \nabla f(\tX_k) + 
        \int_0^1 \nabla^2 f\bracks{\tX_k + \tau \Delta \tH_2} \Delta \tH_2 \dtau.
}
Since $\nabla f$ is Lipschitz, $\nabla^2 f$ is bounded, and 
\eqn{
\normtwo{R(t_{k+1})}
\le 
    \int_0^1 \normop{\nabla^2 f\bracks{\tX_k + \tau \Delta \tH_2}} \normtwo{\Delta \tH_2} \dtau
\le 
    \mu_2(f) \normtwo{\Delta \tH_2}.
}
By \eqref{eq:norm_delta_h2} and \eqref{eq:norm_grad_f},
\begin{align*}
    -\Exp{
        \abracks{
            \nabla f(\tH_2),
            \nabla f(\tX_k)
        }
        | \F_{t_k}
    }
    =&
    -
        \norm{ \nabla f(\tX_k) }_2^2 
    -
        \abracks{
            \Exp{ R(t_{k+1}) | \F_{t_k} }, \nabla f(\tX_k)
        }\\
    \le&
        \norm{
            \Exp{ R(t_{k+1}) | \F_{t_k} }
        }_2
        \norm{ \nabla f(\tX_k)}_2 \\
    \le&
        \Exp{
            \norm{ R(t_{k+1})  }_2
            | \F_{t_k}
        }
        \norm{ \nabla f(\tX_k)}_2 \\
    \le&
       \mu_2(f) \Exp{ \norm{\Delta \tH_2}_2 | \F_{t_k}} \norm{ \nabla f(\tX_k)}_2 \\
    \le&
        \mu_2(f) \Exp{ \norm{\Delta \tH_2}_2^2 | \F_{t_k}}^{1/2} \norm{ \nabla f(\tX_k)}_2 \\
    \le&
        \mu_2(f) \Exp{ \norm{\Delta \tH_2}_2^2 | \F_{t_k}} h^{-1} \\
    \le&
        \mu_2(f) \pi_{2, 2} (f)\bracks{1 + \norm{\tX_k}_2^2  } h + d.
\end{align*}
Therefore, for $h < 1 \wedge \alpha / (4\mu_2(f) \pi_{2,2}(f))$, 
\begin{align*}
    &-\Exp{
        \abracks{
            \nabla f(\tH_2),
            \tX_k
        }
        | \F_{t_k}
    } \\
    =&
        -\E \Bigl[
            \abracks{
                \nabla f(\tH_2),
                \tH_2
            } 
            + \abracks{ \nabla f(\tH_2), \nabla f(\tX_k) }  h 
            - \abracks{ \nabla f(\tH_2), v_1' + v_2 } | \F_{t_k}
        \Bigr] \\
    \le&
        - \frac{\alpha}{2} \Exp{ \norm{ \tH_2 }_2^2 | \F_{t_k} }+ \beta 
        - \Exp{
            \abracks{
                \nabla f(\tH_2),
                \nabla f(\tX_k)
            } 
            | \F_{t_k}
        } h 
        +  \Exp{
            \abracks{
                \nabla f(\tH_2),
                v_1' + v_2
            }
            | \F_{t_k}
        } \\
    \le&
        -\frac{\alpha}{2} \norm{\tX_k}_2^2 + \alpha \beta h + \beta 
        + \mu_2(f) \pi_{2, 2} (f) \bracks{1 + \norm{\tX_k}_2^2  } h^2 + d h
        + 2 d \mu_3(f) h \\
    \le&
        -\frac{\alpha}{4} \norm{\tX_k}_2^2 +
        \bracks{
            \alpha \beta +
            \mu_2(f) \pi_{2, 2} (f) + 
             d + 
             2 d \mu_3(f)}  h +
        \beta. \numberthis \label{eq:diss_2}
\end{align*}
Combining \eqref{eq:diss_1} and \eqref{eq:diss_2}, we have
\eqn{
- \Exp{
    \abracks{
        \nabla f(\tH_1) + \nabla f(\tH_2), \tX_k
    }
    | \F_{t_k}
}
\le 
-\frac{3}{4} \alpha \norm{\tX_k}_2^2 + N_5, \numberthis \label{eq:diss_result}
}
where 
$N_5 = \bracks{
    \alpha \beta +
    \mu_2(f) \pi_{2, 2} (f) + 
     d + 
     5 d \mu_3(f)} +
2\beta$.

Putting things together, for $h < 3\alpha / (2N_1 + 4)$, we obtain
\begin{align*}
    \Exp{
        \normtwo{\tX_{k+1}}^2 | \F_{t_k}
    }
    =& 
        \normtwo{\tX_k}^2 +
        \Exp{
            \normtwo{
                \nabla f(\tH_1) + 
                \nabla f(\tH_2)
            }^2
            | \F_{t_k}
        }  \frac{h^2}{4} + 
        2d h
        \\&
        - \Exp{
            \abracks{
                \tX_k, \nabla f(\tH_1) + \nabla f(\tH_2)
            } | \F_{t_k}
        } h 
        \\&
        -
            2^{1/2} \Exp{
                \abracks{
                    \nabla f(\tH_1) + 
                    \nabla f(\tH_2), \xi_{k+1}
                }
                | \F_{t_k}
            } h^{3/2} \\
    \le&
        \normtwo{\tX_k}^2 
        +
            \frac{N_1}{4} \normtwo{\tX_k}^2 h^2 + \frac{N_2}{4} h^2
        + 2dh
        \\&
        -\frac{3}{4} \alpha h \normtwo{ \tX_k }^2
        + N_5 h 
        \\&
        +
            \frac{1}{2} \normtwo{\tX_k}^2 h^2 + \frac{N_3^2}{2} h + N_4 h^{3/2} \\
    \le&
        \bracks{1 - \frac{3}{4}\alpha h + \frac{N_1 + 2}{4} h^2} \normtwo{\tX_k}^2 
        \\&
        + N_2 h^2 / 4 + 2dh + N_5 h + N_3^2 h /2 + N_4 h^{3/2} \\
    \le&
        \bracks{1 - \frac{3}{8}\alpha h} \normtwo{\tX_k}^2 + 
        N_2 h^2 / 4 + 2dh + N_5 h + N_3^2 h /2 + N_4 h^{3/2}, 
\end{align*}
For $h < 1$, by unrolling the recursion, we obtain the following
\eqn{
    \Exp{
        \normtwo{\tX_k}^2
    } 
    \le&
        \Exp{ \normtwo{\tX_0}^2 } + N_6,
    \quad \text{for all} \; k \in \mathbb{N}
    ,
}
where 
\eqn{
N_6 = \frac{1}{3\alpha} 
\bracks{
    2 N_2 + 16d + 8N_5 + 4 N_3^2 + 8N_4
} = \mathcal{O}(d).
}
\end{proof}

\subsubsection{\texorpdfstring{$2n$}{}th Moment Bound}
\begin{lemm} \label{lemm:2nth_moment}
For $n \in \mathbb{N}_{+}$, 
if the $2n$th moment of the initial iterate is finite, then
the $2n$th moments of Markov chain iterates defined in~\eqref{eq:srk} are uniformly bounded by a constant of order $\mathcal{O}(d^n)$, i.e.
\eqn{
    \Exp{\norm{\tX_k}_2^{2n}} \le \lyap_{2n}, \quad \text{for all} \; {k \in \mathbb{N}},
}
where 
\eqn{
\lyap_{2n} = 
\Exp{
    \norm{\tX_0}_2^{2n}
} + 
\frac{8}{3\alpha n} \bracks{N_{7,n} + N_{12,n} },
}
and constants $N_{7,n}$ to $N_{12, n}$ are given in the proof, if the step size 
{
\footnotesize
\eqn{
h < 1 
\wedge \frac{2d}{\pi_{2,2}(f)} 
\wedge \frac{2\pi_{2,1}(f)}{\pi_{2,2}(f)} 
\wedge \frac{\alpha}{4\mu_2(f) \pi_{2,2}(f)}
\wedge \frac{3\alpha}{2N_1 + 4}
\wedge \min \Biggl\{ \bracks{\frac{3\alpha l }{8N_{11,l} } }^2 : l=2,\dots ,n \Biggr\}.
}
}
\end{lemm}
\begin{proof}
Our proof is by induction. The base case is given in Lemma~\ref{lemm:second_moment}. For the inductive
case, we prove that the $2n$th moment is uniformly bounded by a constant of order $\mathcal{O}(d^n)$, assuming the $2(n\text{-}1)$th moment is uniformly bounded by a constant of order $\mathcal{O}(d^{n-1})$.

By the multinomial theorem, 
\begin{align*}
    \Exp{
        \norm{\tX_{k+1}}_2^{2n}
    }
    =& 
        \Exp{
            \norm{
                X_{k} - \tilde{\nabla} f h + 2^{1/2} \xi_{k+1} h^{1/2}
            }_2^{2n}
        } \\
    =&
        \E\Bigl[
            \Bigl(
                \norm{X_{k}}_2^2
                + \norm{
                    \tilde{\nabla} f
                }_2^2 h^2
                + 2 \norm{\xi_{k+1}}_2^2 h
                \\& 
                \phantom{111} -2 \abracks{
                    \tX_k, \tilde{\nabla} f
                } h  
                +  2^{3/2} \abracks{
                    \tX_k,
                    \xi_{k+1}
                } h^{1/2}
                - 2^{3/2} \abracks{
                    \tilde{\nabla} f, 
                    \xi_{k+1}
                } h^{3/2}
            \Bigr)^{n}
        \Bigr] \\
    =&
        \E \Biggl[
            \sum_{\scriptscriptstyle k_1 + \dots + k_6= n }
            (-1)^{\scriptscriptstyle k_4 +k_6} 
            \binom{n}{k_1\,\dots\,k_6}
            2^{\scriptscriptstyle k_3 + k_4 + \frac{3k_5}{2} + \frac{3 k_6}{2}}
            h^{\scriptscriptstyle 2k_2 + k_3 + k_4 + \frac{k_5}{2} + \frac{3k_6}{2}}
            \\
            &\phantom{111}
            \norm{\tX_k}_2^{2k_1}
            \norm{\tilde{\nabla} f}_2^{2k_2}
            \norm{\xi_{k+1}}_2^{2k_3}
            \abracks{
                \tX_k, \tilde{\nabla} f
            }^{k_4}
            \abracks{
                \tX_k, \xi_{k+1}
            }^{k_5}
            \abracks{
                \tilde{\nabla } f, \xi_{k+1}
            }^{k_6}
        \Biggr] \\
    =&
        \Exp{
            \norm{\tX_k}_2^{2n} + A h + B h^{3/2}
        }, 
\end{align*}
where
\eqn{
    \text{
    \footnotesize
        $
        A =
            {
                2n \norm{\tX_k}_2^{2(n-1)} \norm{\xi_{k+1}}_2^2
                -2n \norm{\tX_k}_2^{2(n-1)} \abracks{
                    \tX_k, 
                    \tilde{\nabla } f
                }
                + 4 n(n-1) \norm{\tX_k}_2^{2(n-2)} \abracks{
                    \tX_k, \xi_{k+1}
                }^2
            },
        $
    }
}
\eqn{
    B \le
        \sum_{
            \mathclap{
                \substack{
                    \scriptscriptstyle k_1 + \cdots + k_6 = n\\
                    \scriptscriptstyle 2k_2 + k_3 + k_4 + \frac{k_5}{2} + \frac{3k_6}{2} > 1
                }
            }
        }
        2^{\frac{3n}{2}}
        \binom{n}{k_1\,\dots\,k_6}
        \norm{\tX_k}_2^{
            \scriptscriptstyle 2k_1 + k_4 + k_5
        }
        \norm{\tilde{\nabla} f}_2^{
            \scriptscriptstyle 2k_2 + k_4 + k_6
        }
        \norm{\xi_{k+1}}_2^{
            \scriptscriptstyle 2k_3 + k_5 + k_6
        }
        .
}
Now, we bound the expectation of $A$ using~\eqref{eq:diss_result},
\begin{align*}
    \Exp { A |\F_{t_k}}
    \le&
        2d n \norm{\tX_k}_2^{2(n-1)}
        +2n \norm{\tX_k}_2^{2(n-1)} \bracks{
            -\frac{3}{8} \alpha \norm{\tX_k}_2^2 + \frac{N_5}{2}
        }
        +4d n(n-1) \norm{\tX_k}_2^{2(n-1)} \\
    \le&
        -\frac{3}{4} \alpha n \norm{\tX_k}_2^{2n} + 
        \bracks{
            2d n + 
            n N_5 + 
            4d n(n-1)
        } \norm{\tX_k}_2^{2(n-1)}. 
\end{align*}
Moreover, by the inductive hypothesis,
\begin{align*}
\Exp{
    A
}
=& 
    \Exp{
        \Exp{
            A | \F_{t_k}
        }
    }
\le
    -\frac{3}{4} \alpha n \Exp{
        \norm{\tX_k}_2^{2n}
    } +
    N_{7,n}, \numberthis \label{eq:a}
\end{align*}
where $N_{7,n} =\bracks{
            2d n + 
            n N_5 + 
            4d n(n-1)
    } \lyap_{2(n-1)} = \mathcal{O}(d^n)$.

Next, we bound the expectation of $B$. By the Cauchy–Schwarz inequality,
\begin{align*}
\Exp{B | \F_{t_k}}
=&
    \sum_{
        \mathclap{
            \substack{
                \scriptscriptstyle k_1 + \cdots + k_6 = n\\
                \scriptscriptstyle 2k_2 + k_3 + k_4 + \frac{k_5}{2} + \frac{3k_6}{2} > 1
            }
        }
    }
    2^{\frac{3n}{2}}
    \binom{n}{k_1\,\dots\,k_6}
    \norm{\tX_k}_2^{\scriptstyle 2k_1 + k_4 + k_5} 
    \Exp{
        \norm{\tilde{\nabla} f}_2^{2k_2 + k_4 + k_6}
        \norm{\xi_{k+1}}_2^{2k_3 + k_5+k_6}
        | \F_{t_k}
    } \\
\le&
    \sum_{
        \mathclap{
            \substack{
                \scriptscriptstyle k_1 + \cdots + k_6 = n\\
                \scriptscriptstyle 2k_2 + k_3 + k_4 + \frac{k_5}{2} + \frac{3k_6}{2} > 1
            }
        }
    }
    2^{\frac{3n}{2}}
    \binom{n}{k_1\,\dots\,k_6}
    \norm{\tX_k}_2^{\scriptstyle 2k_1 + k_4 + k_5} 
    \\& \phantom{1111111} \times
    \Exp{
        \norm{\tilde{\nabla} f}_2^{4k_2 + 2k_4 + 2k_6} |\F_{t_k}
    }^{{1}/{2}}
    \Exp {
        \norm{\xi_{k+1}}_2^{4k_3 + 2k_5+2k_6} | \F_{t_k}
    }^{{1}/{2}}.
\end{align*}
Let $\chi(d)^2$ be a chi-squared random variable with $d$ degrees of freedom. 
Recall its $n$th moment has a closed form solution and is of order $\mathcal{O}(d^n)$~\cite{simon2007probability}.
Now, we bound the $2p$th moments of $\tH_1$ and $\tH_2$ for positive integer $p$. 
To achieve this, we first expand the expressions,
\begin{align*}
\norm{\tH_1}_2^{2p}
=&
\norm{\tX_k + v_1 + v_2}_2^{2p} \\
=&
\bracks{
    \norm{\tX_k}_2^2 + \norm{v_1}_2^2 + \norm{v_2}_2^2
    + 2 \abracks{\tX_k, v_1} + 2\abracks{\tX_k, v_2} + 2\abracks{v_1, v_2}
}^p\\
\le&
    \sum_{
        \mathclap{
            \substack{
                \scriptscriptstyle j_1 + \cdots + j_6 = p\\
            }
        }
    }
    2^{j_4 + j_5 + j_6} 
    \binom{p}{j_1\,\dots\,j_6}
    \norm{\tX_k}_2^{2j_1 + j_4 + j_5}
    \norm{v_1}_2^{2j_2 + j_4 + j_6}
    \norm{v_2}_2^{2j_3 + j_5 + j_6} \\
\le&
    \sum_{
        \mathclap{
            \substack{
                \scriptscriptstyle j_1 + \cdots + j_6 = p\\
            }
        }
    }
    2^{\scriptscriptstyle  j_2 + \frac{3}{2} j_4 + j_5 + \frac{3}{2} j_6}
    h^{ \scriptscriptstyle j_2 + j_3 + \frac{j_4}{2} + \frac{j_5}{2} + j_6 }
    \binom{p}{j_1\,\dots\,j_6}
    \norm{\tX_k}_2^{2j_1 + j_4 + j_5} 
    \\
    & \phantom{111} \times
    \norm{\xi_{k+1}}_2^{2j_2 + j_4 + j_6}
    \norm{\eta_{k+1}}_2^{2j_3 + j_5 + j_6} \\
\le&
    \sum_{
        \mathclap{
            \substack{
                \scriptscriptstyle j_1 + \cdots + j_6 = p\\
            }
        }
    }
    2^{3p}    
    \binom{p}{j_1\,\dots\,j_6}
    \norm{\tX_k}_2^{2j_1 + j_4 + j_5} 
    \norm{\xi_{k+1}}_2^{2j_2 + j_4 + j_6}
    \norm{\eta_{k+1}}_2^{2j_3 + j_5 + j_6} \\
\le&
    \sum_{
        \mathclap{
            \substack{
                \scriptscriptstyle j_1 + \cdots + j_6 = p\\
            }
        }
    }
    2^{3p}    
    \binom{p}{j_1\,\dots\,j_6}
    \bracks{ \scriptstyle
        \frac{2j_1 + j_4 + j_5}{2p}  \norm{\tX_k}_2^{2p}  + 
        \frac{2j_2 + j_4 + j_6}{2p} \norm{\xi_{k+1}}_2^{2p} + 
        \frac{2j_3 + j_5 + j_6}{2p}  \norm{\eta_{k+1}}_2^{2p}
    } \\
\le&
    2^{4p} 3^p \bracks{
        \norm{\tX_k}_2^{2p} + 
        \norm{\xi_{k+1}}_2^{2p} + 
        \norm{\eta_{k+1}}_2^{2p}
    }, 
\end{align*}
where the second to last inequality follows from Young's inequality for products with three
variables. 

Therefore,
\begin{align*}
\Exp{
    \norm{\tH_1}_2^{2p}
    | \F_{t_k}
}
\le&
    2^{4p} 3^p \norm{\tX_k}_2^{2p} + 
    2^{4p+1} 3^p \Exp{
        \chi(d)^{2p}
    }. \numberthis \label{eq:th_1^p}
\end{align*}
Similarly, 
\begin{align*}
\norm{\tH_2}_2^{2p}
=&
\norm{\tX_k -\nabla f(\tX_k) h + v_1' + v_2}_2^{2p} \\
\le&
    \Bigl(
        \norm{\tX_k}_2^2 + \norm{\nabla f(\tX_k)}_2^2 h^2 + \norm{v_1' + v_2}_2^2
        \\ & 
        \phantom{11}
        - 2 \abracks{
            \tX_k, \nabla f(\tX_k)
        }h
        + 2 \abracks{
            \tX_k, v_1' + v_2
        }
        - 2 \abracks{
            \nabla f(\tX_k), v_1'+v_2
        }
    \Bigr)^p \\
\le&
    \sum_{
        \mathclap{
            \substack{
                \scriptscriptstyle j_1 + \cdots + j_6 = p\\
            }
        }
    }
    2^{j_4 + j_5 + j_6} 
    \binom{p}{j_1\,\dots\,j_6}
    \norm{\tX_k}_2^{2j_1 + j_4 + j_5}
    \norm{\nabla f(\tX_k)}_2^{2j_2 + j_4 + j_6}
    \norm{v_1' + v_2 }_2^{2j_3 + j_5 + j_6} \\
\le&
    2^{4p} 3^p \bracks{
        \norm{\tX_k}_2^{2p} + 
        \norm{\nabla f(\tX_k)}_2^{2p} + 
        \norm{\xi_{k+1}}_2^{2p} + 
        \norm{\eta_{k+1}}_2^{2p}
    }. 
\end{align*}
Therefore,
\begin{align*}
\Exp{
    \norm{\tH_2}_2^{2p}
    | \F_{t_k}
}
\le&
    2^{4p} 3^p  \bracks{1 + \pi_{2, 2p} (f)} \norm{\tX_k}_2^{2p} + 
    2^{4p + 1} 3^p \bracks{
        \pi_{2, 2p}(f)  + 
        \Exp{ 
            \chi(d)^{2p}
        }
    }. \numberthis \label{eq:th_2^p}
\end{align*}
Thus, combining \eqref{eq:th_1^p} and \eqref{eq:th_2^p},
\begin{align*}
\Exp{
    \norm{
        \tilde{\nabla} f
    }_2^{2p}
    | \F_{t_k}
}
\le&
    \frac{1}{2}
    \Exp{
        \norm{
            \nabla f(\tH_1 )
        }_2^{2p} + 
        \norm{
            \nabla f(\tH_2 )
        }_2^{2p}
        | \F_{t_k}
    } \\
\le&
    \frac{1}{2} \pi_{2, 2p} (f) \Exp{
        2+ 
        \norm{ \tH_1}_2^{2p}  + 
        \norm{ \tH_2}_2^{2p}
        | \F_{t_k}
    } \\
\le&
    N_{8, n}(p)^2 \norm{\tX_k}_2^{2p} + N_{9, n}(p)^2,
\end{align*}
where the $p$-dependent constants are
\eqn{
N_{8, n}(p) &= 2^{2p} 3^{\frac{p}{2}} \bracks{\pi_{2,2p}(f) \bracks{1 + \frac{1}{2} \pi_{2,2p}(f)  } }^{\frac{1}{2}}, \\
N_{9, n}(p) &= \bracks{
\pi_{2, 2p}(f) \bracks{
    2^{4p+1} 3^{p} \Exp{\chi(d)^{2p}} + 2^{4p} 3^p \pi_{2, 2p}(f)  + 1
}}^{\frac{1}{2}} = \mathcal{O}(d^{\frac{p}{2}}).
}
Since $N_{8, n}(p)$ does not depend on the dimension, let
\begin{align*}
N_{8, n} = 
\max 
\{
N_{8, n}({\scriptstyle 2k_2 +k_4 + k_6}):
{
\scriptstyle
k_1, \dots, k_6 \in \mathbb{N}, \;
k_1 + \cdots + k_6 = n, \;
2k_2 + k_3 + k_4 + \frac{k_5}{2} + \frac{3k_6}{2} > 1
}
\}.
\end{align*}
The bound on $B$ reduces to
\begin{align*}
\Exp{B | \F_{t_k}}
\le&
    \sum_{
        \mathclap{
            \substack{
                \scriptscriptstyle k_1 + \cdots + k_6 = n\\
                \scriptscriptstyle 2k_2 + k_3 + k_4 + \frac{k_5}{2} + \frac{3k_6}{2} > 1
            }
        }
    }
    2^{\frac{3n}{2}}
    \binom{n}{k_1\,\dots\,k_6}
    \norm{\tX_k}_2^{\scriptstyle 2k_1 + k_4 + k_5}
    \Exp {
        \chi(d)^{\scriptstyle 4k_3 + 2k_5 + 2k_6}
    }^{1/2}
    \\&
    \times
    \Bigl(
        N_{8, n}
        \norm{\tX_k}_2^{\scriptstyle 2k_2 +k_4 + k_6} + 
        N_{9, n}({ \scriptstyle 2k_2 +k_4 + k_6} )
    \Bigr)\\
\le&
    B_1 + 
    B_2,
\end{align*}
where 
\eqn{
B_1 &=
    \sum_{
        \mathclap{
            \substack{
                \scriptscriptstyle k_1 + \cdots + k_6 = n\\
                \scriptscriptstyle 2k_2 + k_3 + k_4 + \frac{k_5}{2} + \frac{3k_6}{2} > 1
            }
        }
    }
    2^{\frac{3n}{2}}
    \binom{n}{k_1\,\dots\,k_6} 
    \Exp {
        \chi(d)^{\scriptstyle 4k_3 + 2k_5 + 2k_6}
    }^{1/2}
    N_{8, n}
    \norm{\tX_k}_2^{\scriptstyle 2k_1 + 2k_2 + 2k_4 + k_5 + k_6}, \\
B_2 &=     
    \sum_{
        \mathclap{
            \substack{
                \scriptscriptstyle k_1 + \cdots + k_6 = n\\
                \scriptscriptstyle 2k_2 + k_3 + k_4 + \frac{k_5}{2} + \frac{3k_6}{2} > 1
            }
        }
    }
    2^{\frac{3n}{2}}
    \binom{n}{k_1\,\dots\,k_6} 
    \Exp {
        \chi(d)^{\scriptstyle 4k_3 + 2k_5 + 2k_6}
    }^{1/2} 
    N_{9, n}({ \scriptstyle 2k_2 +k_4 + k_6} )
    \norm{\tX_k}_2^{\scriptstyle 2k_1 + k_4 + k_5}.
}
In the following, we bound the expectations of $B_1$ and $B_2$ separately.
By Young's inequality for products and the function $x \mapsto x^{\scriptstyle 1 / (2k_3 + k_5 + k_6)}$ being concave on the positive domain, 
\begin{align*}
\Exp {
    \chi(d)^{\scriptstyle 4k_3 + 2k_5 + 2k_6}
}^{1/2}
&N_{8, n}
\norm{\tX_k}_2^{\scriptstyle 2k_1 + 2k_2 + 2k_4 + k_5 + k_6} \\
\le&
N_{8, n}{
\scriptstyle
 \Bigl(
        \frac{2k_3 + k_5 + k_6 }{2n} \Exp {
            \chi(d)^{\scriptscriptstyle 4k_3 + 2k_5 + 2k_6}
        }^{\frac{2n}{4k_3 + 2k_5 + 2k_6}}  +
        \frac{2k_1 + 2k_2 + 2k_4 + k_5 + k_6}{2n} \norm{\tX_k}_2^{2n}
    \Bigr)
} \\
\le&
    N_{8, n} \bracks{
        \Exp {
            \chi(d)^2
        }^{n} + 
        \norm{\tX_k}_2^{2n}
    }. 
\end{align*}
Hence, 
\begin{align*}
\Exp{
    B_1
    | \F_{t_k}
}
\le&
    \sum_{
        \mathclap{
            \substack{
                \scriptscriptstyle k_1 + \cdots + k_6 = n\\
            }
        }
    }
    2^{\frac{3n}{2}}
    \binom{n}{k_1\,\dots\,k_6} 
    N_{8, n} \bracks{
        \Exp {
            \chi(d)^2
        }^{n} + 
        \norm{\tX_k}_2^{2n}
    } \\
=&
    2^{\frac{3n}{2}}
    6^n 
    N_{8, n} \bracks{
        d^n+ 
        \norm{\tX_k}_2^{2n}
    }. \numberthis \label{eq:b_1}
\end{align*}
Similarly, 
\begin{align*}
\Exp {
        \chi(d)^{\scriptstyle 4k_3 + 2k_5 + 2k_6}
}^{\frac{1}{2}} 
&N_{9, n}({ \scriptstyle 2k_2 +k_4 + k_6} )
    \norm{\tX_k}_2^{\scriptstyle 2k_1 + k_4 + k_5}\\
\le&
{
\scriptstyle
    \bracks{
        \Exp {
                \chi(d)^{\scriptscriptstyle 4k_3 + 2k_5 + 2k_6}
        }^{\frac{1}{2}} 
        N_{9, n}({ \scriptscriptstyle 2k_2 +k_4 + k_6} )
    }^{\frac{2n}{2k_2 + 2k_3 + k_4 + k_5 + 2k_6}}
    +
    \norm{\tX_k}_2^{2n}
} \\
\le&
    N_{10, n} + \norm{\tX_k}_2^{2n}, 
\end{align*}
where
\begin{align*}
N_{10, n} = \max 
\bigl\{
&
    \bigl(
        \E [ \chi(d)^{\scriptstyle 4k_3 + 2k_5 + 2k_6} ]^{\frac{1}{2}} 
        N_{9, n}({ \scriptstyle 2k_2 +k_4 + k_6} )
    \bigr)^{\frac{2n}{2k_2 + 2k_3 + k_4 + k_5 + 2k_6}
}
:
\\ &  \phantom{1111}
{
\scriptstyle
k_1, \dots, k_6 \in \mathbb{N},\;
k_1 + \cdots + k_6 = n, \;
2k_2 + k_3 + k_4 + \frac{k_5}{2} + \frac{3k_6}{2} > 1
}
\bigr\}
=
\mathcal{O}(d^n).
\end{align*}
Hence, 
\begin{align*}
\Exp {
    B_2
    | \F_{t_k}
}
\le&
    \sum_{
        \mathclap{
            \substack{
                \scriptscriptstyle k_1 + \cdots + k_6 = n\\
            }
        }
    }
    2^{\frac{3n}{2}}
    \binom{n}{k_1\,\dots\,k_6} 
    \bracks{
        N_{10, n} + \norm{\tX_k}_2^{2n}
    } \\
\le&
    2^{\frac{3n}{2}}
    6^n 
    \bracks{
        N_{10, n} + \norm{\tX_k}_2^{2n}
    }. \numberthis \label{eq:b_2}
\end{align*}
Therefore, combining \eqref{eq:b_1} and \eqref{eq:b_2},
\begin{align*}
\Exp{
    B
}
=&
\Exp{
    \Exp{
        B_1 + B_2
        | \F_{t_k}
    }
}
\le
    N_{11,n} \Exp{
        \norm{\tX_k}_2^{2n}
    } + 
    N_{12,n}, \numberthis \label{eq:b}
\end{align*}
where
\eqn{
N_{11,n}
=& 
2^{\frac{3n}{2}} 6^n 
\bracks{1 + N_{8, n}},  \\
N_{12,n}
=&
2^{\frac{3n}{2}} 6^n \bracks{
    N_{8, n} d^n + N_{10, n}
} = \mathcal{O}(d^n).
}
Thus, when $h < (3n \alpha / 8N_{11,n})^2$, by \eqref{eq:a} and \eqref{eq:b},
\begin{align*}
\Exp{
    \norm{\tX_{k+1}}^{2n}
}
\le&
    \bracks{
        1 - \frac{3}{4} \alpha n h + N_{11,n} h^{3/2}
    }
    \Exp{
        \norm{\tX_k}_2^{2n}
    }
    +
    N_{7,n} h + N_{12,n} h^{3/2} \\
\le&
    \bracks{
        1 - \frac{3}{8} \alpha n h
    }
    \Exp{
        \norm{\tX_k}_2^{2n}
    }
    +
    N_{7,n} h + N_{12,n} h^{3/2},
\end{align*}
Hence,
\begin{align*}
\Exp{
    \norm{\tX_k}_2^{2n}
}
\le&
\Exp{
    \norm{\tX_0}_2^{2n}
} + 
\frac{8}{3\alpha n} \bracks{N_{7,n} + N_{12,n} }.
\end{align*}
\end{proof}

\subsection{Local Deviation Orders} \label{app:local_order_proof}
We first provide two lemmas on bounding the second and fourth moments of the change in the continuous-time process. 
These will be used later when we verify the order conditions.
\begin{lemm} \label{lemm:second_moment_cont} 
Suppose $X_t$ is the continuous-time process defined by~\eqref{eq:continuous} 
initiated from some iterate of the Markov chain $X_0$ defined by~\eqref{eq:srk}, 
then the second moment of $X_t$ is uniformly bounded by a constant of 
order $\mathcal{O}(d)$, i.e.
\eqn{
    \Exp{ \normtwo{X_t}^2 } \le \lyap_2', \quad \text{for all} \; t \ge 0,
}
where $\lyap_2' = \lyap_2 + 2(\beta + d) / \alpha$.
\end{lemm}
\begin{proof}
By It\^o's lemma and dissipativity,
\eqn{
    \frac{\dee }{\dt} \Exp{ \normtwo{X_t}^2 }
    =&
        -2 \Exp{ \abracks{\nabla f(X_t), X_t } } + 
        2 d 
    \le
        -\alpha \Exp{ \normtwo{X_t}^2 } + 2(\beta + d).
}
Moreover, by Gr\"onwall's inequality,
\eqn{
    \Exp{ \normtwo{X_t}^2 } 
    \le&
         e^{-\alpha t} \Exp{ \normtwo{X_0}^2 } + 2(\beta + d) / \alpha
    \le
        \lyap_2 + 2(\beta + d) / \alpha
    =
        \lyap_2'.
}
\end{proof}

\begin{lemm}[Second Moment of Change] \label{lemm:continuous_2mom_change}
Suppose $X_t$ is the continuous-time process defined by~\eqref{eq:continuous} initiated from some iterate of the Markov chain $X_0$ defined by~\eqref{eq:srk}, then 
\eqn{
    \Exp{\normtwo{X_t - X_0}^2 } \le C_0 t = \mathcal{O}(dt),
    \quad \text{for all} \; 0 \le t \le 1,
}
where $C_0 = 2 \pi_{2,2}(f) \bracks{1 + \lyap_2'} +  4 d$.
\end{lemm}
\begin{proof} \label{proof:continuous_2mom_change}
By Young's inequality,
\eqn{
    \Exp{\normtwo{X_t - X_0}^2}
    &= 
        \Exp{
            \normtwo{ -\int_0^t \nabla f(X_s)\ds + \sqrt{2} B_t }^2
        } \\
    &\le 
        2\Exp{
            \normtwo{\int_0^t \nabla f(X_s) \ds }^2 +
            2 \normtwo{B_t}^2
        } \\
    &\le 
        2 t \int_0^t \Exp{ 
            \normtwo{ \nabla f(X_s) }^2 
        } \ds +
        4 \Exp{ \normtwo{B_t}^2 } \\
    &\le 
        2 \pi_{2,2}(f) t \int_0^t \Exp{
            1 + \normtwo{X_s}^2
        } \ds + 
        4 d t \\
    &\le 
        2 \pi_{2,2}(f) \bracks{1 + \lyap_2'} t+ 
        4 d t.
}
\end{proof}

\begin{lemm}
Suppose $X_t$ is the continuous-time process defined by~\eqref{eq:continuous} 
initiated from some iterate of the Markov chain $X_0$ defined by~\eqref{eq:srk}, 
then the fourth moment of $X_t$ is uniformly bounded by a constant of 
order $\mathcal{O}(d^2)$, i.e.
\eqn{
    \Exp{ \normtwo{X_t}^4 } \le \lyap_4', \quad \text{for all} \; t \ge 0,
}
where 
$\lyap_4' = 
    \lyap_4 + 
    (2\beta + 6)\lyap_2' / \alpha$.
\end{lemm}
\begin{proof}
By It\^o's lemma, dissipativity, and Lemma~\ref{lemm:second_moment_cont}, 
\eqn{
    \frac{\dee}{\dt} \Exp{ \normtwo{X_t}^4 }
    =&
        -4 \Exp{ \normtwo{X_t}^2 \abracks{ \nabla f(X_t), X_t }}
        + 12 \Exp{ \normtwo{X_t}^2 } \\
    \le&
        -2 \alpha \Exp{ \normtwo{X_t}^4 } + (4\beta + 12) \Exp{ \normtwo{X_t}^2 } \\
    \le&
        -2 \alpha \Exp{ \normtwo{X_t}^4 } + (4\beta + 12) \lyap_2'.
}
Moreover, by Gr\"onwall's inequality,
\eqn{
    \Exp{ \normtwo{X_t}^4 }
    \le&
        e^{-2\alpha t} \Exp{ \normtwo{X_0}^4 } + (2\beta + 6) \lyap_2' / \alpha \\
    \le&
        \lyap_4 + (2\beta + 6) \lyap_2' / \alpha 
    =
        \lyap_4'.
}
\end{proof}

\begin{lemm}[Fourth Moment of Change] \label{lemm:continuous_4mom_change}
Suppose $X_t$ is the continuous-time process defined by~\eqref{eq:continuous} 
initiated from some iterate of the Markov chain $X_0$ defined by~\eqref{eq:srk}, 
then
\eqn{
    \Exp{\normtwo{X_t - X_0}^4 } \le C_1 t^2 = \mathcal{O}(d^2t^2),
    \quad \text{for all} \; 0 \le t \le 1,
}
where $C_1 = 8 \pi_{2,4}(f) \bracks{1 + \lyap_4' } + 32 d(d + 2)$.
\end{lemm}
\begin{proof}\label{proof:continuous_4mom_change}
By Young's inequality,
\eqn{
    \Exp{\normtwo{X_t - X_0}^4} 
    =&
        \Exp{
            \normtwo{
                -\int_0^t \nabla f(X_s) \ds + \sqrt{2} B_t
            }^4
        } \\
    =&
        \Exp{
            \bracks{
                \normtwo{-\int_0^t \nabla f(X_s) \ds + \sqrt{2} B_t}^2
            }^2
        } \\
    \le&
        \Exp{
            \bracks{
                2 \normtwo{\int_0^t \nabla f(X_s) \ds }^2 + 
                4 \normtwo{B_t}^2
            }^2
        } \\
    \le&
        \Exp{
            \bracks{
                2 t \int_0^t \normtwo{\nabla f(X_s)}^2 \ds + 
                4 \normtwo{B_t}^2
            }^2
        } \\
    \le&
        \Exp{
            8 t^2 \bracks{
                \int_0^t \normtwo{\nabla f(X_s)}^2 \ds
            }^2 + 
            32 \normtwo{B_t}^4
        } \\
    \le&
        8 t^3 \int_0^t \Exp{
            \normtwo{\nabla f(X_s)}^4
        } \ds + 
        32 \Exp{\normtwo{ B_t }^4} \\
    \le &
        8 \pi_{2,4}(f) t^3 \int_0^t \Exp{
            1 + \normtwo{X_s}^4
        } \ds + 
        32 d(d + 2) t^2 \\
    \le &
        8 \pi_{2,4}(f) \bracks{
            1 + \lyap_4'
        } t^2 + 
        32 d(d + 2) t^2.
}
\end{proof}
\subsubsection{Local Mean-Square Deviation}
\begin{lemm} \label{lemm:discretization_meansqr}
Suppose $X_t$ and $\tX_t$ are the continuous-time process defined by~\eqref{eq:continuous} and Markov chain defined by~\eqref{eq:srk} for time $t \ge 0$, respectively. If $X_t$ and $\tX_t$ are initiated from the same iterate of the Markov chain $X_0$ and share the same Brownian motion, then 
\eqn{
    \Exp{ \normtwo{X_t - \tX_t}^2 } \le C_2 t^4 = \mathcal{O}(d^2 t^4), 
    \quad \text{for all}\; 0 \le t \le 1,
}
where 
\begin{align*}
C_2 =& 
8 C_1^{1/2} (1 + \lyap_4')^{1/2} 
\bracks{\mu_2(f)^2 \pi_{3, 4}(f)^{1/2} + \mu_3(f)^2 \pi_{2,4}(f)^{1/2} } \\
&+ 
    \bracks{
        8 \pi_{2,4}(f) \bracks{ 1 + \lyap_4} + 
        116 d^2 + 
        90 d +
        8 C_0
    } \mu_3 (f)^2.
\end{align*}
\end{lemm}
\begin{proof} \label{proof:discretization_meansqr}
Since the two processes share the same Brownian motion,
\eqn{
    X_t - \tX_t = -\int_0^t \nabla f(X_s) \ds + \frac{t}{2} \bracks{
        \nabla f(\tH_1) + \nabla f(\tH_2)
    }. \numberthis \label{eq:discrete_main}
}
By It\^o's lemma,
\eqn{
    \nabla f(X_s)
    =&
        \nabla f(X_0) -
        \int_0^s \bracks{\nabla^2 f(X_u) \nabla f(X_u) - \lap{\nabla f}(X_u)} \du  + 
        \sqrt{2} \int_0^s \nabla^2 f(X_u) \dB_u \\
    =&
        \nabla f(X_0) - 
        \nabla^2 f(X_0) \nabla f(X_0) s + 
        \sqrt{2} \nabla^2 f(X_0) B_s + R(s),
}
where the remainder is
\eqn{
    R(s) =
    &
        \underbrace{
            \int_0^s \bracks{
                - \nabla^2 f(X_u) \nabla f(X_u)
                + \nabla^2 f(X_0) \nabla f(X_0)
            } \du
        }_{R_1(s)} + 
        \underbrace{\int_0^s \lap{\nabla f}(X_u) \du}_{R_2(s)} 
    \\
    &+ 
        \underbrace{
            \sqrt{2} \int_0^s \bracks{\nabla^2 f(X_u) - \nabla^2 f(X_0)}\dB_u
        }_{R_3(s)}.
}
We bound the second moment of $R(s)$ by bounding those of $R_1(s)$, $R_2(s)$, and $R_3(s)$ separately.
For $R_1(s)$, by the Cauchy–Schwarz inequality, 
\eqn{
    \Exp{ \normtwo{R_1(s)}^2 }
    =&
        \Exp{
            \normtwo{
                \int_0^s \bracks{
                    \nabla^2 f(X_u) \nabla f(X_u) -
                    \nabla^2 f(X_0) \nabla f(X_0)
                } \du
            }^2
        } \\
    =&
        2\Exp{
            \normtwo{
                \int_0^s \bracks{
                    \nabla^2 f(X_u) \nabla f(X_u) -
                    \nabla^2 f(X_0) \nabla f(X_u)
                } \du
            }^2
        } \\
        &+ 2\Exp{
            \normtwo{
                \int_0^s 
                    \bracks{
                        \nabla^2 f(X_0) \nabla f(X_u) - 
                        \nabla^2 f(X_0) \nabla f(X_0)
                    } \du
            }^2
        }
        \\
    \le& 
        2 s \int_0^s \Exp{
            \normtwo{
                \nabla^2 f(X_u) \nabla f(X_u) -
                \nabla^2 f(X_0) \nabla f(X_u)
            }^2
        } \du \\ 
        &+ 2s \int_0^s \Exp{
            \normtwo{
                \nabla^2 f(X_0) \nabla f(X_u) - 
                \nabla^2 f(X_0) \nabla f(X_0)
            }^2
        } \du \\
    \le&
        2 s \int_0^s \Exp{
            \normop{\nabla^2 f(X_u) - \nabla^2 f(X_0)}^2
            \normtwo{\nabla f(X_u)}^2 
        } \du \\ 
        &+ 2s \int_0^s \Exp{
            \normop{\nabla^2 f(X_0)}^2
            \normtwo{\nabla f(X_u) - \nabla f(X_0)}^2
        } \du \\
    \le&
        2 \mu_3(f)^2 s \int_0^s \Exp{
            \normtwo{X_u - X_0}^2
            \normtwo{\nabla f(X_u)}^2 
        } \du \\ 
        &+
        2 \mu_2(f)^2 s \int_0^s \Exp{
            \normop{\nabla^2 f(X_0)}^2
            \normtwo{X_u - X_0}^2
        } \du \\
    \le&
        2 \mu_3(f)^2 s \int_0^s \Exp{
            \normtwo{X_u - X_0}^4
        }^{\nicefrac{1}{2}}
        \Exp{
            \normtwo{\nabla f(X_u)}^4 
        }^{\nicefrac{1}{2}} \du \\ 
        &+
        2 \mu_2(f)^2 s \int_0^s \Exp{
            \normop{\nabla^2 f(X_0)}^4
        }^{\nicefrac{1}{2}}
        \Exp{
            \normtwo{X_u - X_0}^4
        }^{\nicefrac{1}{2}} \du \\
    \le&
        2\mu_3(f)^2 \pi_{2,4}(f)^{1/2} C_1^{1/2} \bracks{1 + \lyap_4'}^{1/2} 
            \int_0^s u \du\\
        &+
        2\mu_2(f)^2 \pi_{3, 4}(f)^{1/2} C_1^{1/2} \bracks{1 + \lyap_4'}^{1/2} s 
            \int_0^s u \du \\
    \le&
        C_1^{1/2} \bracks{ 1 + \lyap_4' }^{1/2} 
        \bracks{\mu_2(f)^2 \pi_{3, 4}(f)^{1/2} + \mu_3(f)^2 \pi_{2,4}(f)^{1/2} } s^3.
    \numberthis \label{eq:r_1}
}
For $R_2(s)$, by Lemma~\ref{lemm:laplacian_bounded}, 
\eqn{
    \Exp{ \normtwo{R_2(s)}^2 }
    =&
        \Exp{
            \normtwo{\int_0^s \lap{\nabla f}(X_u) \du}^2
        } \\
    \le&
        s \int_0^s \Exp{ 
            \normtwo{\lap{\nabla f}(X_u)}^2
        } \du  \\
    \le&
         \mu_3 (f)^2 d^2 s^2. \numberthis \label{eq:r_2}
}
For $R_3(s)$, by It\^o isometry,
\eqn{
    \Exp{ \normtwo{R_3(s)}^2 }
    =&
        2 \Exp{
            \normtwo{
                \int_0^s \bracks{
                    \nabla^2 f(X_u) - \nabla^2 f(X_0) 
                } \dB_u
            }^2
        } \\
    =&
        2 \Exp{
            \int_0^s \normtwo{
                \nabla^2 f(X_u) - \nabla^2 f(X_0) 
            }^2 \du
        } \\
    \le&
        2 \mu_3(f)^2 \int_0^s \Exp{ \normtwo{ X_u - X_0 }^2 } \du
        \\
    \le& 
        2 \mu_3(f)^2 C_0 \int_0^s  u \du \\
    \le &
        \mu_3(f)^2 C_0 s^2. \numberthis \label{eq:r_3}
}
Thus, combining \eqref{eq:r_1}, \eqref{eq:r_2}, and \eqref{eq:r_3},
\eqn{
    \Exp{\normtwo{R(s)}^2}
    \le&
        4\Exp{\normtwo{R_1(s)}^2} + 
        4\Exp{\normtwo{R_2(s)}^2} + 
        4\Exp{\normtwo{R_3(s)}^2} \\
    \le&
        4 C_1^{1/2} (1 + \lyap_4')^{1/2} 
        \bracks{\mu_2(f)^2 \pi_{3, 4}(f)^{1/2} + \mu_3(f)^2 \pi_{2,4}(f)^{1/2} } s^2 \\
        &+ 
            4 \mu_3 (f)^2 \bracks{d^2 + C_0} s^2.
}
Next, we characterize the terms in the Markov chain update. By Taylor's theorem,
\eqn{
    \nabla f(\tH_1) &=
        \nabla f(X_0) +
        \nabla^2 f(X_0) \Delta \tH_1 + 
        \rho_1(t), \\
    \nabla f(\tH_2) &=
        \nabla f(X_0) +
        \nabla^2 f(X_0) \Delta \tH_2 + 
        \rho_2(t),
}
where
\eqn{
    \rho_1(t) & = \int_0^1 (1 - \tau)
        \nabla^3 f( X_0 + \tau \Delta \tH_1 )
            [\Delta\tH_1, \,\Delta\tH_1 ]\dtau, \\
    \rho_2(t) &=  \int_0^1 (1 - \tau)
        \nabla^3 f( X_0 + \tau \Delta \tH_2 )
            [\Delta\tH_2, \, \Delta \tH_2 ]\dtau, \\
    \Delta \tH_1 &= \sqrt{2} \bracks{ \frac{1}{t} \Psi(t) + \frac{1}{\sqrt{6}} B_t},\\
    \Delta \tH_2 &= -\nabla f(X_0) t + \sqrt{2} \bracks{ \frac{1}{t} \Psi(t) - \frac{1}{\sqrt{6}} B_t}, \\
    \Psi(t) &= \int_0^t B_s\ds.
}
We bound the fourth moments of $\Delta \tH_1$ and $\Delta \tH_2$,
\eqn{
    \Exp{
        \normtwo{\Delta \tH_1}^4
    }
    =&
         \Exp{
            \normtwo{
                \sqrt{2} \bracks{ 
                    \frac{1}{t} \Psi(t) + \frac{1}{\sqrt{6}} B_t 
                }
            }^4
        } \\
    \le&
        \frac{32}{t^4} \Exp{
            \normtwo{
                \Psi(t)
            }^4
        } + 
        \frac{8}{9} \Exp{
            \normtwo{
                B_t
            }^4 
        } \\
    =&
        \frac{32}{t^4} \sum_{i=1}^d \Exp{
            \Psi_i(t)^4
        } +
        \frac{32}{t^4} \sum_{i,j=1, i \ne j}^d \Exp{
            \Psi_i(t)^2
        }
        \Exp{
            \Psi_j(t)^2
        } + 
        \frac{8}{9} d(d + 2) t^2 \\ 
    \le&
        \frac{32}{t^4}  \frac{d t^6}{3} +
        \frac{32}{t^4} \frac{d (d - 1) t^6}{9} + 
        \frac{8 d(d + 2) t^2}{9} \\ 
    =&
        \bracks{
            \frac{32 d}{3} +
            \frac{32 d (d - 1)}{9} + 
            \frac{8 d(d + 2)}{9}
        } t^2 \\ 
    \le&
        2d (6d + 5) t^2.
}
Similarly,
\eqn{
    \Exp{
        \normtwo{\Delta \tH_2}^4
    }
    =&
        \Exp{
            \normtwo{-\nabla f(X_0) t + \sqrt{2} \bracks{ \frac{1}{t} \Psi(t) - \frac{1}{\sqrt{6}} B_t}}^4
        } \\
    \le&
        8 \Exp{
            \normtwo{\nabla f(X_0)}^4
        } t^4 + 
        8 \Exp{
            \normtwo{\sqrt{2}
                \bracks{
                    \frac{1}{t} \Psi(t) - \frac{1}{\sqrt{6}} B_t
                }
            }^4
        } \\
    \le&
        8 \pi_{2,4}(f) \Exp{
            1 + \normtwo{X_0}^4
        } t^4  + 
        16d (6d + 5) t^2 \\
    \le&
        8 \pi_{2, 4}(f) \bracks{1 + \lyap_4} t^4 + 16d (6d + 5) t^2 \\
    \le&
        8 \bracks{
            \pi_{2,4}(f) \bracks{1 + \lyap_4} + 2 d (6d + 5)
        } t^2.
}
Using the above information, we bound the second moments of $\rho_1(t)$ and $\rho_2(t)$,
\eqn{
    \Exp{\normtwo{\rho_1(t)}^2 }
    =& 
        \Exp{
            \normtwo{
                \int_0^1 (1 - \tau) 
                \nabla^3 f( X_0 + \tau \Delta \tH_1 )
                    [\Delta\tH_1, \,\Delta\tH_1 ]\dtau
            }^2
        } \\
    \le&
        \int_0^1 \Exp{
            \normtwo{
                \nabla^3 f(X_0 + \tau \Delta\tH_1 )
                    [\Delta\tH_1, \, \Delta\tH_1] 
            }^2
        } \dtau \\
    \le&
        \int_0^1 \Exp{
            \normop{\nabla^3 f(X_0 + \tau \Delta\tH_1)}^2
            \normtwo{\Delta\tH_1}^4
        } \dtau \\
    \le& 
        \mu_3(f)^2 \int_0^1 \Exp{
            \normtwo{\Delta\tH_1}^4
        } \dtau \\
    \le&
        2d (6d + 5) \mu_3(f)^2  t^2.
}
Similarly,
\eqn{
    \Exp{\normtwo{\rho_2(t)}^2 }
    \le&
        \mu_3(f)^2 \int_0^1 \Exp{
            \normtwo{\Delta\tH_2}^4
        } \dtau \\
    \le&
        8 \bracks{
            \pi_{2,4}(f) \bracks{1 + \lyap_4} + 2 d (6d + 5)
        }
        \mu_3(f)^2 t^2.
}
Plugging these results into~\eqref{eq:discrete_main},
\eqn{
    X_t - \tX_t = 
        -\int_0^t R(s) \ds
        - \frac{t}{2} \bracks{\rho_1(t) + \rho_2(t)}.
}
Thus, 
\eqn{
    \Exp{ \normtwo{X_t - \tX_t}^2 }
    =&
        \Exp{
            \normtwo{
                -\int_0^t R(s) \ds
                - \frac{t}{2} \bracks{\rho_1(t) + \rho_2(t)} 
            }^2
        }\\ 
    \le&
        4 t \int_0^t \Exp{\normtwo{ R(s) }^2 } \ds + 
        t^2
            \Exp{
                \normtwo{
                    \rho_1(t)
                }^2
            } + 
        t^2
            \Exp{
                \normtwo{\rho_2(t)}^2
            }
        \\
    \le&
        8 C_1^{1/2} (1 + \lyap_4')^{1/2} 
        \bracks{\mu_2(f)^2 \pi_{3, 4}(f)^{1/2} + \mu_3(f)^2 \pi_{2,4}(f)^{1/2} } t^4 \\
        &+ 
            \bracks{
                8 \pi_{2,4}(f) \bracks{ 1 + \lyap_4} + 
                116 d^2 + 
                90 d +
                8 C_0
            } \mu_3 (f)^2 t^4 \\
    \le&
        C_2 t^4.
}
\end{proof}

\subsubsection{Local Mean Deviation}
\begin{lemm}
Suppose $X_t$ and $\tX_t$ are the continuous-time process defined by~\eqref{eq:continuous} and Markov chain defined by~\eqref{eq:srk} for time $t \ge 0$, respectively. If $X_t$ and $\tX_t$ are initiated from the same iterate of the Markov chain $X_0$ and share the same Brownian motion, then 
\eqn{
    \Exp{
        \normtwo{
            \Exp{
                X_t - \tX_t 
                | \F_{0}
            }
        }^2
    } \le C_3 t^5 =  \mathcal{O}(d^3 t^5), \quad \text{for all}\; 0 \le t \le 1,
}
where 
\begin{align*}
C_3=&
4  \bracks{
        C_1^{1/2} \bracks{ 1 + \lyap_4' }^{1/2} 
        \bracks{\mu_2(f)^2 \pi_{3, 4}(f)^{1/2} + \mu_3(f)^2 \pi_{2,4}(f)^{1/2} }
        + C_0 d \mu_4(f)^2 
    }
    \\ & \phantom{111}
    +
        \frac{1}{4} \mu_3(f)^2 \pi_{2,4}(f) \bracks{1 + \lyap_4}
    + 
        8 \mu_4(f)^2 \bracks{ 
            \pi_{2, 6}(f) \bracks{1 + \lyap_6} + 
            73 (d + 4)^3
        }.
\end{align*}
\end{lemm}
\begin{proof}
The proof is similiar to that of Lemma~\ref{lemm:discretization_meansqr} with slight variations on truncating the expansions. Recall since the two processes share the same Brownian motion,
\eqn{
    X_t - \tX_t = -\int_0^t \nabla f(X_s) \ds + \frac{t}{2} \bracks{
        \nabla f(\tH_1) + \nabla f(\tH_2)
    }.
}
By It\^o's lemma,
\eqn{
    \nabla f(X_s)
    =&
        \nabla f(X_0) -
        \int_0^s \bracks{\nabla^2 f(X_u) \nabla f(X_u) - \lap{\nabla f}(X_u)} \du  + 
        \sqrt{2} \int_0^s \nabla^2 f(X_u) \dB_u \\
    =&
        \nabla f(X_0) - 
        \nabla^2 f(X_0) \nabla f(X_0) s + 
        \sqrt{2} \nabla^2 f(X_0) B_s + 
        \Vec{\Delta}(\nabla f)(X_0) s +
        \barR(s),
}
where the remainder is
\eqn{
    \barR(s) =
    &
        \underbrace{
            \int_0^s \bracks{
                - 
                \nabla^2 f(X_u) \nabla f(X_u)
                +
                \nabla^2 f(X_0) \nabla f(X_0)
            } \du
        }_{\barR_1(s)} \\
    &+ 
        \underbrace{
            \int_0^s \bracks{
                \lap{\nabla f}(X_u) -
                \lap{\nabla f}(X_0)
            }\du
        }_{\barR_2(s)} \\
    &+ 
        \underbrace{
            \sqrt{2} \int_0^s \bracks{\nabla^2 f(X_u) - \nabla^2 f(X_0)}\dB_u
        }_{\barR_3(s)}.
}
By Taylor's theorem with the remainder in integral form,
\eqn{
    \nabla f(\tH_1) &=
        \nabla f(X_0) +
        \nabla^2 f (X_0) \Delta \tH_1 + 
        \frac{1}{2} \nabla^3 f(X_0) [\Delta \tH_1,\, \Delta \tH_1] + 
        \bar{\rho}_1(t), \\
    \nabla f(\tH_2) &=
        \nabla f(X_0) +
        \nabla^2 f (X_0) \Delta \tH_2 + 
        \frac{1}{2} \nabla^3 f (X_0) [\Delta \tH_2,\, \Delta \tH_2] + 
        \bar{\rho}_2(t),
}
where
\eqn{
    \bar{\rho}_1(t) &=
        \frac{1}{2} \int_0^1 (1 - \tau)^2 \nabla^4 f(X_0 + \tau \Delta \tH_1)
        [ \Delta \tH_1,\, \Delta \tH_1, \, \Delta \tH_1 ] \dtau, \\
    \bar{\rho}_2(t) &= 
        \frac{1}{2} \int_0^1 (1 - \tau)^2 \nabla^4 f(X_0 + \tau \Delta \tH_2)
        [ \Delta \tH_2,\, \Delta \tH_2, \, \Delta \tH_2 ] \dtau.
}
Now, we show the following equality in a component-wise manner, 
\begin{align*}
    \frac{t^2}{2} \Exp{
        \lap{\nabla f}(X_0) 
    } +& 
    \frac{t^3}{4} \Exp{
        \nabla^3 f(X_0) [\nabla f(X_0),\, \nabla f(X_0)]
    }
    = \\
        &\frac{t}{4} \E \left[
            \nabla^3 f(X_0) [\Delta\tH_1,\, \Delta\tH_1 ]
            \right]
        + 
        \frac{t}{4} \E \left[
            \nabla^3 f(X_0) [\Delta\tH_2,\, \Delta\tH_2 ]
        \right]. \numberthis \label{eq:cancel_laplacian}
\end{align*}
To see this, recall that odd moments of the Brownian motion is zero. So, for each $\partial_i f$,
\begin{align*}
    \Exp{
        \abracks{
            \Delta\tH_1 , 
            \nabla^2 (\partial_i f) (X_0) \Delta\tH_1
        }
    }
    =&
        \Exp{
            \Exp{
                \Tr{
                    (\Delta\tH_1 )^\top \Delta\tH_1 \nabla^2 (\partial_i f) (X_0)
                }
                | \F_0
            }
        } \\
    =&
        \Exp{
            \Tr{
                \Exp{
                    (\Delta\tH_1 )^\top \Delta\tH_1 | \F_0
                } \nabla^2 (\partial_i f) (X_0)
            }
        } \\
    =&
        2t \bracks{
            \frac{1}{2} + \frac{1}{\sqrt{6}}
        } \Exp{
            \Delta (\partial_i f) (X_0)
        }.
\end{align*}
Similarly,
\begin{align*}
    \Exp{
        \abracks{
            \Delta\tH_2 , 
            \nabla^2 (\partial_i f) (X_0) \Delta\tH_2
        }
    }
    =&
        \Exp{
            \Exp{
                \Tr{
                    (\Delta\tH_2 )^\top \Delta \tH_2 \nabla^2 (\partial_i f) (X_0)
                }
                |\F_0
            }
        } \\
    =&
        \Exp{
            \Tr{
                \Exp{
                    (\Delta\tH_2 )^\top \Delta\tH_2 | \F_0
                } \nabla^2 \partial_i f(X_0)
            }
        } \\
    =&
        2t \bracks{
            \frac{1}{2} - \frac{1}{\sqrt{6}}
        } \Exp{
            \Delta (\partial_i f) (X_0)
        } \\ 
        &+ t^2 \Exp{
            \abracks{
                \nabla f(X_0),
                \nabla^2 (\partial_i f) (X_0) \nabla f(X_0)
            }
        }.
\end{align*}
Adding the previous two equations together, we obtain the desired equality~\eqref{eq:cancel_laplacian}.

Next, we bound the second moments of $\barR_1(s)$ and $\barR_2(s)$. For $\barR_1(s)$, recall from the proof of Lemma~\ref{lemm:discretization_meansqr},
\eqn{
    \Exp{\normtwo{\barR_1(s)}^2}
    =
        \Exp{\normtwo{R_1(s)}^2}
    \le
        C_1^{1/2} \bracks{ 1 + \lyap_4' }^{1/2} 
        \bracks{\mu_2(f)^2 \pi_{3, 4}(f)^{1/2} + \mu_3(f)^2 \pi_{2,4}(f)^{1/2} } s^3. 
}
Additionally for $\barR_2(s)$,
\begin{align*}
    \Exp{\normtwo{\barR_2(s)}^2} 
    =& \Exp{
        \normtwo{
            \int_0^s \bracks{
                \lap{\nabla f}(X_u) -
                \lap{\nabla f}(X_0)
            }\du
        }^2
    } \\
    \le&
        s \int_0^s \Exp{
            \normtwo{
                \lap{\nabla f}(X_u) -
                \lap{\nabla f}(X_0)
            }^2
        } \du \\
    \le&
        d^2 \mu_4(f)^2 s \int_0^s \Exp{
                \normtwo{X_u - X_0}^2
            } \du \\
    \le&
        C_0 d^2 \mu_4(f)^2 s  \int_0^s u \du \\
    \le&
        C_0 d^2 \mu_4(f)^2 \frac{s^3}{2}.
\end{align*}
Since $\barR_3(s)$ is a Martingale,
\begin{align*}
    \normtwo{
        \Exp{\int_0^t \barR(s) \ds | \F_0}  
    }^2 
    =& \normtwo{
        \Exp{\int_0^t \barR_1(s) \ds | \F_0} + 
        \Exp{\int_0^t \barR_2(s) \ds | \F_0}  
    }^2 \\
    \le&
        2 \normtwo{
            \Exp{\int_0^t \barR_1(s) \ds | \F_0}
        }^2 + 
        2 \normtwo{
            \Exp{\int_0^t \barR_2(s) \ds | \F_0}
        }^2 \\
    \le&
        2 t \int_0^t \Exp{
            \normtwo{\barR_1(s)}^2 + \normtwo{\barR_2(s)}^2 | \F_0
        } \ds.
\end{align*}
Therefore, 
\begin{align*}
    \Exp{
        \normtwo{
            \Exp{\int_0^t \barR(s) \ds | \F_0}  
        }^2 
    }
    \le&
        2t \int_0^t \Exp{
            \normtwo{\barR_1(s)}^2 + 
            \normtwo{\barR_2(s)}^2
        } \ds \\
    \le&
        C_1^{1/2} \bracks{ 1 + \lyap_4' }^{1/2} 
        \bracks{\mu_2(f)^2 \pi_{3, 4}(f)^{1/2} + \mu_3(f)^2 \pi_{2,4}(f)^{1/2} } t^5 \\
        &+ C_0 d \mu_4(f)^2 t^5.
\end{align*}
Next, we bound the sixth moments of $\Delta \tH_1$ and $\Delta \tH_2$. Note for two random vectors $a$ and $b$, by Young's inequality and Lemma~\ref{lemm:a_plus_b_n}, we have
\begin{align*}
    \Exp{
        \normtwo{ a + b }^6
    }
    &\le
        \Exp{
            \bracks{
                2\normtwo{a}^2 + 2\normtwo{b}^2 
            }^3
        }
    \le
        32 \Exp{
            \normtwo{a}^6 + \normtwo{b}^6
        }.
\end{align*}
To simplify notation, we define
\eqn{
    &v_1 = \sqrt{2} \bracks{ \frac{1}{2} + \frac{1}{\sqrt{6}} } \xi \sqrt{t}, \quad 
    v_1' = \sqrt{2} \bracks{ \frac{1}{2} - \frac{1}{\sqrt{6}} } \xi \sqrt{t}, \\
    &v_2 = \frac{1}{\sqrt{6}} \eta \sqrt{t} \quad \text{where} \quad \xi, \eta \overset{\text{i.i.d.}}{\sim}
    \mathcal{N}(0, I_d),
}
We bound the sixth moments of $v_1$, $v_1'$ and $v_2$ using $\nicefrac{1}{2} + \nicefrac{1}{\sqrt{6}} < 1$, $\nicefrac{1}{2} - \nicefrac{1}{\sqrt{6}} < \nicefrac{1}{2}$ and the closed form moments of a chi-squared random variable with $d$ degrees of freedom $\chi(d)^2$~\cite{simon2007probability},
\eqn{
    \Exp{ \normtwo{v_1}^6 }
    \le&
        8 \Exp{\normtwo{\xi}^6} t^3
    =
        8 \Exp{ \chi(d)^6 } t^3
    = 
        8 d ( d + 2) (d + 4) t^3
    < 
        8 (d + 4)^3 t^3, \\
    \Exp{ \normtwo{v_1'}^6 } 
    \le&
        \Exp{ \normtwo{\xi}^6 } t^3
    = 
        \Exp{ \chi(d)^6 }  t^3
    = 
        d(d+2)(d+4) t^3
    <
        (d + 4)^3 t^3, \\
    \Exp{ \normtwo{v_2}^6 } 
    =&
        \frac{1}{216} \Exp{ \normtwo{\eta}^6 } t^3
    = 
        \frac{1}{216} \Exp{ \chi(d)^6 } t^3
    = 
        \frac{1}{216} d(d + 2) (d + 4) t^3 
    < 
        \frac{1}{216} (d +4)^3 t^3.
}
Then,
\begin{align*}
    \Exp{
        \normtwo{\Delta \tH_1}^6
    }
    =& \Exp{
        \normtwo{v_1 + v_2}^6
    } 
    \le
        32\Exp{
            \normtwo{v_1}^6 + \normtwo{v_2}^6
        }
    \le
        288 (d + 4)^3 t^3, \\
    \Exp{
        \normtwo{\Delta \tH_2}^6
    }
    =&
        \Exp{
            \normtwo{ -\nabla f(X_0) t + v_1' + v_2 }^6
        } \\
    \le&
        32 \Exp{
            \normtwo{\nabla f(X_0) t }^6
        } t^6 + 
        32 \Exp{
            \normtwo{v_1' + v_2}^6
        } \\
    \le&
        32 \pi_{2, 6}(f) \bracks{1 + \Exp{\normtwo{X_0}^6} } t^6 + 
        1024 \Exp{ \normtwo{v_1'}^6  + \normtwo{v_2}^6}\\
    \le&
        32 \pi_{2, 6}(f) \bracks{1 + \lyap_6 } t^3 + 
        2048 (d +4)^3 t^3  \\
    \le&
        32 \bracks{ 
            \pi_{2, 6} (f) \bracks{1 + \lyap_6} + 
            64 (d + 4)^3
        } t^3. 
\end{align*}
Now, we bound the second moments of $\bar{\rho}_1(t)$ and $\bar{\rho}_2(t)$ using the derived sixth-moment bounds,
\begin{align*}
    \Exp{
        \normtwo{\bar{\rho}_1(t)}^2
    }
    =&
        \Exp{
            \normtwo{
                \frac{1}{2} \int_0^1 (1 - \tau)^2 \nabla^4 f(X_0 + \tau \Delta \tH_1)
                [ \Delta \tH_1,\, \Delta \tH_1, \, \Delta \tH_1 ]
            }^2
        }\\
    \le&
        \frac{1}{4} \sup_{z \in \R^d} \normop{ \nabla^4 f(z)}^2 \Exp{
            \normtwo{\Delta \tH_1}^6
        }  \\
    \le&
        72 \mu_4(f)^2 (d +4)^3 t^3.
\end{align*}
Similarly,
\begin{align*}
    \Exp{
        \normtwo{\bar{\rho}_2(t)}^2
    }
    =&
        \Exp{
            \normtwo{
                \frac{1}{2} \int_0^1 (1 - \tau)^2 \nabla^4 f(X_0 + \tau \Delta \tH_2)
                [ \Delta \tH_2,\, \Delta \tH_2, \, \Delta \tH_2 ]
            }^2
        }\\
    \le&
        \frac{1}{4} \sup_{z \in \R^d} \normop{ \nabla^4 f(z)}^2  \Exp{
            \normtwo{\Delta \tH_2}^6
        } \\
    \le&
        8 \mu_4(f)^2 \bracks{ 
            \pi_{2, 6}(f) \bracks{1 + \lyap_6} + 
            64 (d + 4)^3
        } t^3.
\end{align*}
Thus, 
\begin{align*}
    &\Exp{
        \normtwo{
            \Exp{
                X_t - \tX_t | \F_0
            }
        }^2
    }\\
    =&
        \E \Biggl[
            \Bigl\|
                \E \bigl[
                    - \int_0^t \bar{R}(s) \ds + 
                    \frac{t^3}{4} \nabla^3 f(X_0) [\nabla f(X_0), \, \nabla f(X_0)]
                    + \frac{t}{2} \bar{\rho_1}(t) +  \frac{t}{2} \bar{\rho_2}(t)  | \F_0
                \bigr]
            \Bigr\|^2
        \Biggr] \\
    \le& 
        4\Exp{
            \normtwo{
                \Exp{\int_0^t \barR(s) \ds | \F_0}  
            }^2 
        }
        + 
        \frac{t^6}{4} \Exp{
            \normtwo{ \nabla^3 f(X_0) [\nabla f(X_0), \, \nabla f(X_0)] }^2
        }
        \\ &
        + 
        t^2 \Exp{
            \normtwo{
                \bar{\rho}_1(t)
            }^2
            +
            \normtwo{
                \bar{\rho}_2(t)
            }^2
        } 
        \\
    \le&
        4  \bracks{
            C_1^{1/2} \bracks{ 1 + \lyap_4' }^{1/2} 
            \bracks{\mu_2(f)^2 \pi_{3, 4}(f)^{1/2} + \mu_3(f)^2 \pi_{2,4}(f)^{1/2} }
            + C_0 d \mu_4(f)^2 
        }t^5 \\
        &+
            \frac{1}{4} \mu_3(f)^2 \Exp{ \normtwo{\nabla f(X_0) }^4 } t^6 \\
        &+ 
            72 \mu_4(f)^2 (d +4)^3 t^5
        +
            8 \mu_4(f)^2 \bracks{ 
                \pi_{2, 6} \bracks{1 + \lyap_6} + 
                64 (d + 4)^3
            } t^5 \\
    \le&
        4  \bracks{
            C_1^{1/2} \bracks{ 1 + \lyap_4' }^{1/2} 
            \bracks{\mu_2(f)^2 \pi_{3, 4}(f)^{1/2} + \mu_3(f)^2 \pi_{2,4}(f)^{1/2} }
            + C_0 d \mu_4(f)^2 
        } t^5 \\
        &+
            \frac{1}{4} \mu_3(f)^2 \pi_{2,4}(f) \bracks{1 + \lyap_4} t^5 \\
        &+ 
            8 \mu_4(f)^2 \bracks{ 
                \pi_{2, 6}(f) \bracks{1 + \lyap_6} + 
                73 (d + 4)^3
            } t^5 \\
    \le&
        C_3 t^5.
\end{align*}
\end{proof}

\subsection{Invoking Theorem~\ref{theo:master}}
Now, we invoke Theorem~\ref{theo:master} with our derived constants. We obtain that if the constant step size
\eqn{
    h <
    1
    \wedge C_h 
    \wedge \frac{1}{2\alpha}
    \wedge \frac{1}{8\mu_1(b)^2 + 8 \mu_1^{\mathrm{F}}(\sigma)^2}
    ,
}
where
\eqn{
    \text{
    \footnotesize
        $
        C_h =
                \frac{2d}{\pi_{2,2}(f)} 
        \wedge \frac{2\pi_{2,1}(f)}{\pi_{2,2}(f)} 
        \wedge \frac{\alpha}{4\mu_2(f) \pi_{2,2}(f)}
        \wedge \frac{3\alpha}{2N_1 + 2N_2 + 4}
        \wedge \min \Biggl\{ \bracks{\frac{3\alpha l }{8N_{11,l} } }^2 : l=2, 3 \Biggr\},
        $
    }
}
and the smoothness conditions on the strongly convex potential in Theorem~\ref{theo:srk_ld} holds, 
then the uniform local deviation bounds~\eqref{eq:uniform_local_deviation_orders} hold with $\lambda_1 = C_2$ and $\lambda_2 = C_3$, and consequently the bound~\eqref{eq:main_bound} holds. 
This concludes that to converge to a sufficiently small positive tolerance $\epsilon$, $\tilde{\mathcal{O}}(d \epsilon^{-2/3})$ iterations are required, since $C_2$ is of order $\mathcal{O}(d^2)$, and $C_3$ is of order $\mathcal{O}(d^3)$.
\section{Proof of Theorem~\ref{theo:srk_id}} \label{app:srk_id}
\subsection{Moment Bounds}
Verifying the order conditions in Theorem~\ref{theo:master} for SRK-ID 
requires bounding the second and fourth moments of the Markov chain.

The following proofs only assume Lipschitz smoothness of the drift coefficient $b$
and diffusion coefficient $\sigma$ to a certain order and a generalized notion of dissipativity for It\^o diffusions.
\begin{defi}[Dissipativity] \label{defi:dissipativity_general}
For constants 
$\alpha, \beta > 0$,
the diffusion satisfies the following
\eqn{
    -2\abracks{ b(x), x} - \normf{\sigma(x)}^2 \ge \alpha \normtwo{x}^2 - \beta, \quad \text{for all} \; x \in \R^d.
}
\end{defi}
For general It\^o diffusions, dissipativity directly follows from uniform dissipativity, 
where $\beta$ is an appropriate constant of order $\mathcal{O}(d)$.
Additionally, we assume the discretization has a constant step size $h$ and the timestamp
of the $k$th iterate is $t_k$ as per the proof of Theorem~\ref{theo:master}. 
To simplify notation, we rewrite the update as
\eqn{
    \tX_{k+1} = \tX_k + b(\tX_k) h + \sigma(\tX_k) \xi_{k+1} h^{1/2} + \tY_{k+1}, \quad 
    \xi_{k+1} \sim \mathcal{N}(0, I_d),
}
where
\eqn{
    \tY_{k+1}^{(i)} = \bracks{
        \sigma_i(\tH_1^{(i)}) - \sigma_i(\tH_2^{(i)})
    } h^{1/2}, \quad
    \tY_{k+1} = \frac{1}{2} \sum_{i=1}^m \tY_{k+1}^{(i)}.
}
Note that $\xi_{k+1}$ and $\tY_{k+1}$ are not independent, 
since we model $I_{(\cdot)} = (I_{(1)}, \dots, I_{(m)})^\top$ as $\xi_{k+1} h^{1/2}$. 
Moreover, we define the following notation
\eqn{
    I_{(\cdot, i)} &= (I_{(1,i)}, \dots, I_{(m,i)})^\top, \quad
    \Delta \tH^{(i)} = \sigma(\tX_k) I_{(\cdot,i)} h^{-1/2}, \quad \; i=1, \dots, m.
}
Hence, the variables 
$\tH_1^{(i)}$ and $\tH_2^{(i)}$ can be written as 
\eqn{
    \tH_1^{(i)} = \tX_k + \Delta \tH^{(i)}, \quad 
    \tH_2^{(i)} = \tX_k - \Delta \tH^{(i)}.
}
We first bound the second moments of $\tY_{k}$, using the following moment inequality.
\begin{theo}[{\cite[Sec. 1.7, Thm. 7.1]{mao2007stochastic}}]\label{theo:moment_ineq}
Let $p \ge 2$. If $\{G_s\}_{s\ge0}$ is a $d \times m$ matrix-valued process, 
and $\{B_t\}_{t\ge0}$ is a $d$-dimensional Brownian motion, both of which are
adapted to the filtration $\{\F_s\}_{s\ge0}$ such that for some fixed $t > 0$, 
the following relation holds
\eqn{
    \Exp{ \int_0^t \normf{G_s}^p \ds } < \infty.
}
Then,
\eqn{
    \Exp{
        \normtwo{ \int_0^t G_s \dB_s }^p
    }
    \le 
    \bracks{\frac{p (p-1)}{2}}^{p/2} t^{(p-2)/ 2}
    \Exp{
        \int_0^t \normf{ G_s }^p \ds
    }.
}
In particular, equality holds when $p=2$.
\end{theo}
The above theorem can be proved directly using It\^o's lemma and It\^o isometry, 
with the help of H\"older's inequality. The theorem can also be seen as a natural consequence 
of the Burkholder-Davis-Gundy Inequality~\cite{mao2007stochastic}.

\begin{coro} \label{coro:delta_H_p}
Let even integer $p \ge 2$. Then, the following relation holds
\eqn{
    \Exp{
        \normtwo{ \Delta \tH^{(i)}  }^p | \F_{t_k}
    }
    \le
        \bracks{
            \frac{p(p-1)}{2}
        }^p \pi_{1, p}^{\mathrm{F}} (\sigma) \bracks{ 
            1 + \normtwo{\tX_k}^{p/2}
        } h^{p/2}.
}
\end{coro}
\begin{proof}
It is clear that the integrability condition in Theorem~\ref{theo:moment_ineq} holds for 
the inner and outer integrals of $\Delta \tH^{(i)}$. Hence, by repeatedly applying the theorem,
\eqn{
    \Exp{
        \normtwo{ \Delta \tH^{(i)}  }^p
        | \F_{t_k}
    }
    =&
        \Exp{
            \normtwo{ \sigma(\tX_k) I_{(\cdot, i)} }^p
            | \F_{t_k}
        }
        h^{-p/2} \\
    =&
        \Exp{
            \normtwo{ \int_{t_k}^{t_{k+1}} 
            \int_{t_k}^s \sigma(\tX_k) \dB_u \dB_s^{(i)}  }^p
            | \F_{t_k}
        }
        h^{-p/2} \\
    \le&
        \bracks{
            \frac{p(p-1)}{2}
        }^{p/2} h^{-1}
        \int_{t_k}^{t_{k+1}} \Exp{
            \normtwo{
                \int_{t_k}^s \sigma(\tX_k) \dB_u
            }^p
            | \F_{t_k}
        } \ds \\
    \le&
        \bracks{
            \frac{p(p-1)}{2}
        }^{p} h^{-1}
        \int_{t_k}^{t_{k+1}}
        s^{(p-2)/2}
        \int_{t_k}^s
        \Exp{
            \normf{\sigma(\tX_k)}^p
            | \F_{t_k}
        }
        \du
        \ds \\
    \le&
        \bracks{
            \frac{p(p-1)}{2}
        }^{p}
        \pi_{1, p}^{\mathrm{F}} (\sigma)
        \bracks{
            1 + \normtwo{\tX_k}^{p/2}
        }
        h^{p/2}.
}
\end{proof}

\begin{lemm}[Second Moment Bounds for $\tY_k$] \label{lemm:yk_bound}
The following relation holds
\eqn{
    \Exp{
        \normtwo{ \tY_{k+1} }^2
        | \F_{t_k}
    }
    \le&
    2^2 3^4 m^2 \mu_2(\sigma)^2 \pi_{1,4}^{\mathrm{F}} (\sigma) \bracks{
        1 + \normtwo{\tX_k}^2
    } h^3. 
}
\begin{proof}
By Taylor's Theorem with the remainder in integral form,
\eqn{
    \normtwo{ \tY_{k+1}^{(i)} }
    =&
        \normtwo{ 
            \sigma_i(\tX_k + \Delta \tH^{(i)}) - 
            \sigma_i(\tX_k - \Delta \tH^{(i)}) 
        } h^{1/2} \\
    =&
        \normtwo{
            \int_0^1 \bracks{
                \nabla \sigma_i(\tX_k + \tau \Delta \tH^{(i)}) - 
                \nabla \sigma_i(\tX_k - \tau \Delta \tH^{(i)}) 
            } \Delta \tH^{(i)} \dtau
        } h^{1/2}\\
    \le&
        h^{1/2} \int_0^1 
            \normop{
                \nabla \sigma_i(\tX_k + \tau \Delta \tH^{(i)}) - 
                \nabla \sigma_i(\tX_k - \tau \Delta \tH^{(i)})
            }
            \normtwo{ \Delta \tH^{(i)} }
            \dtau \\
    \le&
        \mu_2(\sigma) h^{1/2} \normtwo{ \Delta \tH^{(i)} }^2
        \int_0^1 
            2 \tau
            \dtau \\
    \le&
        \mu_2(\sigma) h^{1/2} \normtwo{ \Delta \tH^{(i)} }^2
    . \numberthis \label{eq:yk_second}
}
By \eqref{eq:yk_second} and Corollary~\ref{coro:delta_H_p},
\eqn{
    \Exp{
        \normtwo{ \tY_{k+1}^{(i)} }^2
        | \F_{t_k}
    }
    \le&
        \mu_2(\sigma)^2 \Exp{
            \normtwo{\Delta \tH^{(i)}}^4
            | \F_{t_k}
        } h
    \le
        6^4 \mu_2(\sigma)^2 \pi_{1,4}^{\mathrm{F}} (\sigma) \bracks{
            1 + \normtwo{\tX_k}^2
        } h^3
}
Therefore,
\eqn{
    \Exp{
        \normtwo{ \tY_{k+1} }^2
        | \F_{t_k}
    }
    \le&
        \frac{m}{4} \sum_{i=1}^m \Exp{
            \normtwo{ \tY_{k+1}^{(i)} }^2
            | \F_{t_k}
        }
    \le
        2^2 3^4 m^2 \mu_2(\sigma)^2 \pi_{1,4}^{\mathrm{F}} (\sigma) \bracks{
            1 + \normtwo{\tX_k}^2
        } h^3.
}
\end{proof}
\end{lemm}
To prove the following moment bound lemmas for SRK-ID, we recall a standard quadratic moment bound result whose proof we omit and provide a reference of.
\begin{lemm}[{\cite[Lemma F.1]{erdogdu2018global}}] \label{lemm:quadratic}
Let even integer $p \ge 2$ and $f:\R^d \to \R^{d\times m}$ be Lipschitz. 
For $\xi \sim \mathcal{N}(0, I_m)$ independent from the $d$-dimensional 
random vector $X$, the following relation holds
\eqn{
    \Exp{ \normtwo{ f(X) \xi }^p } \le (p - 1)!! \Exp{ \normf{ f(X) }^p}.
}
\end{lemm}

\subsubsection{Second Moment Bound}
\begin{lemm}\label{lemm:second_moment_id}
If the second moment of the initial iterate is finite, then 
the second moments of Markov chain iterates defined in \eqref{eq:srk_id} are uniformly bounded, i.e.
\eqn{
    \Exp{\normtwo{\tX_k}^{2}} \le \uyap_{2}, \quad \text{for all} \; {k \in \mathbb{N}}
}
where 
\eqn{
    \uyap_2 = \Exp{
        \normtwo{\tX_0}^2
    } + M_2, 
}
and constants $M_1$ and $M_2$ are given in the proof, if the constant step size
\eqn{
h < 1
\wedge \frac{1}{m^2} 
\wedge \frac{\alpha^2}{4M_1^2}.
}
\end{lemm}
\begin{proof}
By direct computation, 
\eqn{
    \normtwo{ \tX_{k+1} }^2
    =&
        \normtwo{\tX_k}^2 + 
        \normtwo{b(\tX_k)}^2 h^2 + 
        \normtwo{\sigma(\tX_k) \xi_{k+1} }^2 h + 
        \normtwo{\tY_{k+1}}^2
        \\&
        +
        2\abracks{ \tX_k, b(\tX_k) } h + 
        2\abracks{ \tX_k, \sigma(\tX_k) \xi_{k+1} } h^{1/2} + 
        2\abracks{ \tX_k, \tY_{k+1} }
        \\&
        +
        2\abracks{ b(\tX_k), \sigma(\tX_k) \xi_{k+1} } h^{3/2} + 
        2\abracks{ b(\tX_k), \tY_{k+1}} h 
        \\&
        + 
        2\abracks{ \sigma(\tX_k) \xi_{k+1}, \tY_{k+1} } h^{1/2}. 
}
By Lemma~\ref{lemm:quadratic} and dissipativity,
\eqn{
    \Exp{
        2\abracks{ \tX_k, b(\tX_k) } h + \normtwo{\sigma(\tX_k) \xi_{k+1} }^2 h
        | \F_{t_k}
    }
    =&
        2\abracks{ \tX_k, b(\tX_k) } h + \normf{\sigma(\tX_k)}^2 h \\
    \le&
        -\alpha \normtwo{\tX_k}^2 h + \beta h. 
}
We bound the remaining terms by direct computation. By linear growth, 
\eqn{
    \normtwo{ b(\tX_k) }^2 h^2
    \le
    \pi_{1,2}(b) \bracks{1 + \normtwo{\tX_k}^2 } h^2.
}
By Lemma~\ref{lemm:yk_bound}, for $h < 1 \wedge 1 / m^2$, 
\eqn{
    \Exp{
        \normtwo{\tY_{k+1}}^2 | \F_{t_k}
    }
    \le&
        2^2 3^4 m^2 \mu_2(\sigma)^2 \pi_{1,4}^{\mathrm{F}} (\sigma)
            \bracks{1 + \normtwo{\tX_k}^2 } 
            h^3 \\
    \le&
        2^2 3^4 m \mu_2(\sigma)^2 \pi_{1,4}^{\mathrm{F}} (\sigma)
        \bracks{1 + \normtwo{\tX_k}^2 } 
        h^{3/2}
    . 
}
By Lemma~\ref{lemm:yk_bound},
\eqn{
    \Exp{
        \abracks{ \tX_k, \tY_{k+1}}
        | \F_{t_k}
    }
    \le&
        \normtwo{\tX_k} \Exp{
            \normtwo{ \tY_{k+1} }
            | \F_{t_k}
        } \\
    \le&
        \normtwo{\tX_k} \Exp{
            \normtwo{ \tY_{k+1} }^2
            | \F_{t_k}
        }^{1/2} \\
    \le&
        2^2 3^2 m \mu_2(\sigma) \pi_{1,4}^{\mathrm{F}} (\sigma)^{1/2}
        \bracks{1 + \normtwo{\tX_k}^2 }
        h^{3/2}
    . 
}
Similarly, by Lemma~\ref{lemm:yk_bound},
\eqn{
    \Exp{
        \abracks{b(\tX_k), \tY_{k+1}}
        | \F_{t_k}
    }
    \le&
        \normtwo{b(\tX_k)} \Exp{
            \normtwo{\tY_{k+1}}
            | \F_{t_k}
        } \\
    \le&
        \normtwo{b(\tX_k)} \Exp{
            \normtwo{\tY_{k+1}}^2
            | \F_{t_k}
        }^{1/2} \\
    \le&
        2^2 3^2 m \mu_2(\sigma) \pi_{1,4}^{\mathrm{F}} (\sigma)^{1/2} \pi_{1,1}(b)  \bracks{
            1 + \normtwo{\tX_k}^2
        } h^{3/2}
    . 
}
By Lemma~\ref{lemm:yk_bound} and Lemma~\ref{lemm:quadratic}, 
\eqn{
    \Exp{
        \abracks{
            \sigma(\tX_k) \xi_{k+1}, \tY_{k+1}
        }
        | \F_{t_k}
    }
    \le&
        \Exp{
            \normtwo{ \sigma(\tX_k) \xi_{k+1} }
            \normtwo{ \tY_{k+1} }
        | \F_{t_k}
        } \\
    \le&
        \Exp{
            \normtwo{ \sigma(\tX_k) \xi_{k+1} }^2
            | \F_{t_k}
        }^{1/2}
        \Exp{
            \normtwo{ \tY_{k+1} }^2
            | \F_{t_k}
        }^{1/2} \\
    \le&
        \Exp{
            \normf{ \sigma(\tX_k) }^2
            | \F_{t_k}
        }^{1/2}
        \Exp{
            \normtwo{ \tY_{k+1} }^2
            | \F_{t_k}
        }^{1/2} \\
    \le&
        2^2 3^2 m
            \mu_2(\sigma)
            \pi_{1,4}^{\mathrm{F}} (\sigma)^{1/2}
            \pi_{1,2}^{\mathrm{F}} (\sigma)^{1/2}
        \bracks{
            1 + \normtwo{ \tX_k }^2
        }
        h^{3/2}
    . 
}
Putting things together, for $h < 1 \wedge {\alpha^2}/(4 M_1^2)$, 
\eqn{
    \Exp{
        \normtwo{\tX_{k+1}}^2
        | \F_{t_k}
    }
    \le&
        \bracks{1 - \alpha h + M_1 h^{3/2}} \normtwo{\tX_k}^2
        +
        \beta h + M_1 h^{3/2} \\
    \le&
        \bracks{1 - \alpha h / 2} \normtwo{\tX_k}^2 + 
        \beta h + M_1 h^{3/2}, 
}
where
\eqn{
    M_1 =& 
    \pi_{1,2}(b)
    + 2^3 3^2 m \mu_2(\sigma) \pi_{1,4}^{\mathrm{F}} (\sigma)^{1/2} \bracks{
        1 + 
        \mu_2(\sigma) \pi_{1,4}^{\mathrm{F}}(\sigma)^{1/2} +
        \pi_{1,1}(b) +
        \pi_{1,2}^{\mathrm{F}} (\sigma)^{1/2}
    }
    .
}
Unrolling the recursion gives the following for $h < 1\wedge 1 / m^2$
\eqn{
    \Exp{
        \normtwo{\tX_k}^2
    }
    \le&
        \Exp{
            \normtwo{\tX_0}^2
        }
        +
        2 \bracks{\beta + M_1 h^{1/2}} / \alpha \\
    \le&
        \Exp{
            \normtwo{\tX_0}^2
        }
        +
        M_2,
    \quad \text{for all} \; k \in \mathbb{N}, 
}
where
\begin{footnotesize}
\eqn{
    M_2 =
    2 \bracks{
        \beta + 
        \pi_{1,2}(b)
        \pi_{1,2}(b)
        + 2^3 3^2 \mu_2(\sigma) \pi_{1,4}^{\mathrm{F}} (\sigma)^{1/2} \bracks{
            1 + 
            \mu_2(\sigma) \pi_{1,4}^{\mathrm{F}}(\sigma)^{1/2} +
            \pi_{1,1}(b) +
            \pi_{1,2}^{\mathrm{F}} (\sigma)^{1/2}
        }
    } / \alpha
    .
}
\end{footnotesize}
\end{proof}

\subsubsection{\texorpdfstring{$2n$}{}th Moment Bound}
Before bounding the $2n$th moments, we first generalize Lemma~\ref{lemm:yk_bound} 
to arbitrary even moments.
\begin{lemm} \label{lemm:zk_even_moment}
Let even integer $p \ge 2$ and $\tZ_{k+1} = \tY_{k+1}h^{-3/2}$. 
Then, the following relation holds
\eqn{
    \Exp{
        \normtwo{
            \tZ_{k+1}
        }^{p}
        | \F_{t_k}
    }
    \le&
        m^p \mu_2(\sigma)^p
        \bracks{
            \frac{2p(2p - 1)}{2}
        }^{2p}
        \pi_{1,2p}^{\mathrm{F}}(\sigma) 
        \bracks{
            1 + \normtwo{\tX_k}^{p}
        }.
}
\end{lemm}
\begin{proof}
For $i \in \{1, 2, \dots, m\}$, by \eqref{eq:yk_second},
\eqn{
    \normtwo{\tZ_{k+1}^{(i)} }
    =
        \tY_{k+1}^{(i)} h^{-3/2}
    \le
        \mu_2(\sigma) h^{-1} \normtwo{\Delta \tH^{(i)}}^2. 
}
Hence, by Corollary~\ref{coro:delta_H_p},
\eqn{
    \Exp{
        \normtwo{
            \tZ_{k+1}^{(i)}
        }^{p}
        | \F_{t_k}
    }
    \le&
        \mu_2(\sigma)^p h^{-p} \Exp{
            \normtwo{\Delta \tH^{(i)}}^{2p}
            | \F_{t_k}
        } \\
    \le&
        \mu_2(\sigma)^p
        \bracks{
            \frac{2p(2p - 1)}{2}
        }^{2p}
        \pi_{1,2p}^{\mathrm{F}}(\sigma) 
        \bracks{
            1 + \normtwo{\tX_k}^{p}
        }.
}
The remaining follows easily from Lemma~\ref{lemm:a_plus_b_n}. 
\end{proof}

\begin{lemm} \label{lemm:fourth_moment_id}
For $n \in \mathbb{N}_{+}$, 
if the $2n$th moment of the initial iterate is finite, then
the $2n$th moments of Markov chain iterates defined in \eqref{eq:srk_id} are uniformly bounded, i.e.
\eqn{
    \Exp{\normtwo{\tX_k}^{2n}} \le \uyap_{2n}, \quad \text{for all} \; {k \in \mathbb{N}}
}
where 
\eqn{
    \uyap_{2n} ={
        \Exp{ \normtwo{ \tX_0 }^{2n} }
        +
        \frac{2}{n\alpha}
        \bracks{
            \beta \uyap_{2(n-1)} +
            2^{23n - 1}
            10^n
            n^{8n}
            \pi_{1,2n}(b)
            \pi_{1, 8n}^{\mathrm{F}} (\sigma)^{1/2}
            \mu_2(\sigma)^{2n}
        }
    },
}
if the step size 
\eqn{
h < 1
\wedge \frac{1}{m^2}
\wedge \frac{\alpha^2}{4M_1^2}
\wedge \min \left\{  \bracks{ \frac{\alpha l}{2M_{3, l}} }^2: l=2, \dots, n \right\}.
}
\end{lemm}
\begin{proof}
Our proof is by induction. The base case is given in Lemma~\ref{lemm:second_moment_id}. For the inductive
case, we prove that the $2n$th moment is uniformly bounded by a constant, assuming the $2(n\text{-}1)$th moment is uniformly bounded by a constant.

By the multinomial theorem, 
\eqn{
    \Exp{
        \normtwo{\tX_{k+1}}^2
        | \F_{t_k}
    }
    =&
        \E \Biggl[
            \biggl(
                \normtwo{\tX_k}^2 + 
                \normtwo{b(\tX_k)}^2 h^2 + 
                \normtwo{\sigma(\tX_k) \xi_{k+1} }^2 h + 
                \normtwo{\tY_{k+1}}^2
                \\&
                \phantom{1111}
                +
                2\abracks{ \tX_k, b(\tX_k) } h + 
                2\abracks{ \tX_k, \sigma(\tX_k) \xi_{k+1} } h^{1/2} + 
                2\abracks{ \tX_k, \tY_{k+1} }
                \\&
                \phantom{1111}
                +
                2\abracks{ b(\tX_k), \sigma(\tX_k) \xi_{k+1} } h^{3/2} + 
                2\abracks{ b(\tX_k), \tY_{k+1}} h 
                \\&
                \phantom{1111}
                + 
                2\abracks{ \sigma(\tX_k) \xi_{k+1}, \tY_{k+1} } h^{1/2}
            \biggl)^n
            | \F_{t_k}
        \Biggr] \\
    =&
        \E \Biggl[
            \biggl(
                \normtwo{\tX_k}^2 + 
                \normtwo{b(\tX_k)}^2 h^2 + 
                \normtwo{\sigma(\tX_k) \xi_{k+1} }^2 h + 
                \normtwo{\tZ_{k+1}}^2 h^3
                \\&
                \phantom{1111}
                +
                2\abracks{ \tX_k, b(\tX_k) } h + 
                2\abracks{ \tX_k, \sigma(\tX_k) \xi_{k+1} } h^{1/2} + 
                2\abracks{ \tX_k, \tZ_{k+1} } h^{3/2}
                \\&
                \phantom{1111}
                +
                2\abracks{ b(\tX_k), \sigma(\tX_k) \xi_{k+1} } h^{3/2} + 
                2\abracks{ b(\tX_k), \tZ_{k+1}} h^{5/2}
                \\&
                \phantom{1111}
                + 
                2\abracks{ \sigma(\tX_k) \xi_{k+1}, \tZ_{k+1} } h^2
            \biggl)^n
            | \F_{t_k}
        \Biggr] \\
    =&
        \normtwo{\tX_k}^{2n} + \Exp{A | \F_{t_k}} h + \Exp{B | \F_{t_k}} h^{3/2},
}
where by the Cauchy–Schwarz inequality,
    \eqn{
        A =& 
            \scriptstyle
            n \normtwo{\tX_k}^{2(n-1)} \bracks{
                2 \abracks{ \tX_k, b(\tX_k) } + 
                \normtwo{ \sigma(\tX_k) \xi_{k+1} }^2
            }
            + 2n(n-1) \normtwo{\tX_k}^{2(n-2)} \abracks{ \tX_k, \sigma(\tX_k) \xi_{k+1} }, \\
        B \le&
            \sum_{
                \mathclap{
                    \substack{
                        \scriptscriptstyle (k_1, \dots, k_{10}) \in J
                    }
                }
            }
            \;
            2^{n}
            \binom{n}{k_1 \, \dots \, k_{10}}
            \normtwo{\tX_k}^{p_1}
            \normtwo{b(\tX_k)}^{p_2}
            \normtwo{ \sigma(\tX_k) \xi_{k+1} }^{p_3}
            \normtwo{\tZ_{k+1}}^{p_4},
    }
the indicator set 
\eqn{
    J =
        & \Bigl\{
            (k_1, \dots, k_{10}) \in \mathbb{N}^{10}: 
            k_1 + \dots + k_{10} = n, \;
            \\& \phantom{1111111111111111111}
            2k_2 + k_3 + 3 k_4 + k_5 + \frac{k_6}{2} + \frac{3 k_7}{2}
                 + \frac{3 k_8}{2} + \frac{5k_9}{2} + 2k_{10} > 1
        \Bigr\}, 
}
and with slight abuse of notation, we hide the explicit dependence on $k_1, \dots, k_{10}$ for the exponents
\eqn{
    p_1 =& 2k_1 + k_5 + k_6 + k_7, \\
    p_2 =& 2k_2 + k_5 + k_8 + k_9, \\
    p_3 =& 2k_3 + k_6 + k_8 + k_{10}, \\
    p_4 =& 2k_4 + k_7 + k_9 + k_{10}.
}
By dissipativity, 
\eqn{
    \Exp{
        A
        | \F_{t_k}
    }
    \le&
        -n\alpha \normtwo{\tX_k}^{2n} + n\beta \normtwo{\tX_k}^{2(n-1)}.
    \numberthis
    \label{eq:id_a_bound}
}
Note that $p_1 + p_2 + p_3 + p_4 = 2n$. Since $h < 1 \wedge 1 / m^2$, we may cancel out the $m$ factor in some of the terms. 
One can verify that the only remaining term that is $m$-dependent is 
\eqn{
    \abracks{ \tX_k, \tZ_{k+1} } = \mathcal{O}(m h^{3/2}).
}
Using this information, Lemma~\ref{lemm:zk_even_moment}, Lemma~\ref{lemm:quadratic}, the Cauchy–Schwarz inequality,
and $p_3 + p_4 \le 2n$, 
{
\footnotesize
\eqn{
    &\Exp{
        B
        | \F_{t_k}
    } \\
    \le&
        \sum_{
            \mathclap{
                \substack{
                    \scriptscriptstyle (k_1, \dots, k_{10}) \in J
                }
            }
        }
        \;
        2^n
        \binom{n}{k_1 \, \dots \, k_{10}}
        \normtwo{\tX_k}^{p_1}
        \normtwo{b(\tX_k)}^{p_2}
        \Exp{
            \normtwo{\sigma(\tX_k) \xi_{k+1}}^{p_3}
            \normtwo{\tZ_{k+1} m^{-1}}^{p_4}
            | \F_{t_k}
        } 
        m
        \\
    \le&
        \sum_{
            \mathclap{
                \substack{
                    \scriptscriptstyle (k_1, \dots, k_{10}) \in J
                }
            }
        }
        \;
        2^n
        \binom{n}{k_1 \, \dots \, k_{10}}
        \normtwo{\tX_k}^{p_1}
        \normtwo{b(\tX_k)}^{p_2}
        \Exp{
            \normtwo{\sigma(\tX_k) \xi_{k+1} }^{2 p_3}
            | \F_{t_k}
        }^{1/2}
        \Exp{
            \normtwo{\tZ_{k+1} m^{-1} }^{2 p_4}
            | \F_{t_k}
        }^{1/2}
        m
        \\
    \le&
        \sum_{
            \mathclap{
                \substack{
                    \scriptscriptstyle (k_1, \dots, k_{10}) \in J
                }
            }
        }
        \;
        2^n
        \binom{n}{k_1 \, \dots \, k_{10}}
        \normtwo{\tX_k}^{p_1}
        \pi_{1, p_2}(b)
        \bracks{
            1 + \normtwo{\tX_k}^{p_2}
        }
        \bracks{(2p_3 - 1)!!}^{1/2}
        \pi_{1, p_3}^{\mathrm{F}} (\sigma)
        \bracks{
            1 + \normtwo{\tX_k}^{p_3}
        }
        \\& \phantom{11111}
        \times
        \mu_2(\sigma)^{p_4}
        \bracks{ 8p_4^2 }^{2 p_4}
        \pi_{1, 4p_4}^{\mathrm{F}} (\sigma)^{1/2}
        \bracks{
            1 + \normtwo{\tX_k}^{p_4}
        }
        m\\
    \le&
        2^n 
        \bracks{1 + \normtwo{\tX_k}}^{2n}
        \sum_{
            \mathclap{
                \substack{
                    \scriptscriptstyle (k_1, \dots, k_{10}) \in J
                }
            }
        }
        \pi_{1, p_2}(b)
        \bracks{(2p_3 - 1)!!}^{1/2} 
        \pi_{1, p_3}^{\mathrm{F}} (\sigma)
        \mu_2(\sigma)^{p_4}
        \bracks{ 8 p_4^2 }^{2p_4}
        \pi_{1, 4p_4}^{\mathrm{F}} (\sigma)^{1/2}
        \binom{n}{k_1 \, \dots \, k_{10}}
        m\\
    \le&
        2^n 
        \bracks{1 + \normtwo{\tX_k}}^{2n}
        \pi_{1, 2n}(b)
        \pi_{1, 2n}^{\mathrm{F}} (\sigma)
        \mu_2(\sigma)^{2n}
        \pi_{1, 4n}^{\mathrm{F}} (\sigma)^{1/2}
        2^{20n} n^{8n}
        \sum_{
            \mathclap{
                \substack{
                    \scriptscriptstyle 
                    k_1, \dots, k_{10} \in \mathbb{N} \\
                    k_1 + \cdots + k_{10} = n
                }
            }
        }
        \;\;\;\;
        \binom{n}{k_1 \, \dots \, k_{10}}
        m\\
    \le&
        2^{23n - 1}
        10^n
        n^{8n}
        \pi_{1,2n}(b)
        \pi_{1, 8n}^{\mathrm{F}} (\sigma)^{1/2}
        \mu_2(\sigma)^{2n}
        m
        \bracks{
            1 + \normtwo{\tX_k}^{2n}
        }.\numberthis \label{eq:id_b_bound}
}
}

By the inductive hypothesis, \eqref{eq:id_a_bound} and \eqref{eq:id_b_bound}, and 
$h < 1 \wedge n^2 \alpha^2 / (4 M_{3,n}^2)$, 
we obtain the recursion 
\eqn{
    \Exp{
        \Exp{
            \normtwo{ \tX_{k+1} }^{2n}
            | \F_{t_k}
        }
    }
    \le&
        \text{
            \footnotesize
            $   
                \bracks{
                    1 - n \alpha h + M_{3,n} h^{3/2}
                }
                \Exp{
                    \normtwo{ \tX_k }^{2n}
                }
                + 
                n \beta h \Exp{
                    \normtwo{\tX_k}^{2(n-1)}
                }
                + M_3 h^{3/2}
            $
        }\\
    \le&
        \text{
            \footnotesize
            $
                \bracks{
                    1 - n \alpha h + M_{3,n} h^{3/2}
                }
                \Exp{
                    \normtwo{ \tX_k }^{2n}
                }
                + 
                n \beta \uyap_{2(n-1)} h + M_{3,n} h^{3/2}
            $
        }\\
    \le&
        \bracks{ 1 - n \alpha h / 2  }
        \Exp{
            \normtwo{ \tX_k }^{2n}
        }
        +
        n \beta \uyap_{2(n-1)} h + M_{3,n} h^{3/2},
}
where the constant 
$
M_{3,n} = 
    2^{23n - 1}
    10^n
    n^{8n}
    \pi_{1,2n}(b)
    \pi_{1, 8n}^{\mathrm{F}} (\sigma)^{1/2}
    \mu_2(\sigma)^{2n}
    m.
$

For $h < 1 \wedge 1 / m^2$, by unrolling the recursion, we obtain
\eqn{
    \Exp{
        \normtwo{ \tX_k }^{2n}
    }
    \le
        \Exp{
            \normtwo{ \tX_0 }^{2n}
        }
        +
        \frac{2}{n \alpha} 
            \bracks{
                n \beta \uyap_{2(n-1)} + M_{3, n} h^{1/2}
            }
    \le
        \uyap_{2n},
    \quad \text{for all} \; k \in \mathbb{N},
}
where 
\eqn{
    \uyap_{2n} ={
        \Exp{
            \normtwo{ \tX_0 }^{2n}
        }
        +
        \frac{2}{n\alpha}
        \bracks{
            \beta \uyap_{2(n-1)} +
            2^{23n - 1}
            10^n
            n^{8n}
            \pi_{1,2n}(b)
            \pi_{1, 8n}^{\mathrm{F}} (\sigma)^{1/2}
            \mu_2(\sigma)^{2n}
        }
    }
    .
}
\end{proof}

\subsection{Local Deviation Orders} \label{app:local_order_proof_id}
In this section, we verify the local deviation orders for SRK-ID.
The proofs are again by matching up terms in the It\^o-Taylor expansion of the 
continuous-time process to terms in the Taylor expansion of the numerical 
integration scheme. Extra care needs to be taken for a tight dimension dependence.

\begin{lemm} \label{lemm:second_moment_cont_id} 
Suppose $X_t$ is the continuous-time process defined by~(\ref{eq:continuous_general}) 
initiated from some iterate of the Markov chain $X_0$ defined by~(\ref{eq:srk_id}), 
then the second moment of $X_t$ is uniformly bounded, i.e.
\eq{
    \Exp{ \normtwo{X_t}^2 } \le \uyap_2', \quad \text{for all} \; t \ge 0.
}
where $\uyap_2' = \uyap_2 + \beta / \alpha$.
\end{lemm}
\begin{proof}
By It\^o's lemma and dissipativity,
\eqn{
    \frac{\dee}{\dt} \Exp{ \normtwo{X_t}^2 }
    =&
        \Exp{ 2\abracks{ X_t, b(X_t) } + \normf{\sigma(X_t)}^2}
    \le
        -\alpha \Exp{ \normtwo{X_t}^2 } + \beta. 
}
Moreover, by Gr\"onwall's inequality,
\eqn{
    \Exp{ \normtwo{X_t}^2 } 
    \le&
         e^{-\alpha t} \Exp{ \normtwo{X_0}^2 } + \beta / \alpha
    \le
        \uyap_2 + \beta / \alpha
    =
        \uyap_2'. 
}
\end{proof}

\begin{lemm}[Second Moment of Change] \label{lemm:continuous_2mom_change_id}
Suppose $X_t$ is the continuous-time process defined by~(\ref{eq:continuous_general}) initiated from some iterate of the Markov chain $X_0$ defined by~(\ref{eq:srk_id}), then
\eqn{
    \Exp{\normtwo{X_t - X_0}^2 } \le D_0 t, \quad \text{for all}\; 0 \le t \le 1,
}
where $D_0 = 2 \bracks{ \pi_{1,2}(b) + \pi_{1,2}^{\mathrm{F}} (\sigma)} \bracks{1 + \uyap_2'}$.
\end{lemm}
\begin{proof}
By It\^o isometry,
\eqn{
    \Exp{\normtwo{X_t - X_0}^2}
    &= 
        \Exp{
            \normtwo{\int_0^t b (X_s)\ds + \int_0^t \sigma(X_s) \dB_s }^2
        } \notag \\
    &\le 
        2\Exp{
            \normtwo{\int_0^t b (X_s) \ds }^2 +
            \normtwo{\int_0^t \sigma (X_s) \dB_s}^2
        } \notag \\
    &\le 
        2 t \int_0^t \Exp{ \normtwo{ b (X_s) }^2  } \ds +
        2 \int_0^t \Exp{ \normf{\sigma(X_s)}^2 } \ds \notag \\
    &\le 
        2 \pi_{1,2}(b) t \int_0^t \Exp{1 + \normtwo{ X_s }^2 } \ds + 
        2 \pi_{1,2}^{\mathrm{F}} (\sigma) \int_0^t \Exp{1 + \normtwo{ X_s }^2 } \ds \notag \\
    &\le
        2 \bracks{ \pi_{1,2}(b) + \pi_{1,2}^{\mathrm{F}} (\sigma) } \bracks{1 + \uyap_2'} t.
}
\end{proof}
To bound the fourth moment of change in continuous-time, we use the following lemma. 
\begin{lemm}[{\cite[adapted from Lemma A.1]{erdogdu2018global}}] \label{lemm:murat_bound}
Assuming $\{X_t\}_{t\ge0}$ is the solution to the SDE~(\ref{eq:continuous_general}), 
under the condition that the drift coefficient $b$ and diffusion coefficient $\sigma$
are Lipschitz. If $\sigma$ satisfies the following sublinear growth condition 
\eq{
    \normf{\sigma(x)}^l \le \pi_{1, l}^{\mathrm{F}}(\sigma) \bracks{1 + \norm{x}^{l/2}},
    \quad \text{for all} \; x \in \R^d, l=1,2, \dots,
}
and the diffusion is dissipative, then for $n \ge 2$, 
we have the following relation
\eq{
    \A \normtwo{x}^n \le - \frac{\alpha n}{4} \normtwo{x}^n + \beta_n,
}
where the (infinitesimal) generator $\mathcal{A}$ is 
defined as 
\eq{
    \A f(x) = \lim_{t \downarrow 0} \frac{ \Exp{ f(X_t) | X_0 = x } - f(x)}{t},
}
and the constant
$\beta_n = \mathcal{O} (d^{ \frac{n}{2} })$.
\end{lemm}
\begin{proof}
By definition of the generator and dissipativity,
\eqn{
    \A \normtwo{x}^n =& 
        n \normtwo{x}^{n-2} \abracks{x, b(x)} + 
        \frac{n}{2} \normtwo{x}^{n-2} \normf{\sigma(x)}^2 +
        \frac{n (n-2)}{2} \normtwo{x}^{n-4} \abracks{
            \mathrm{vec}(xx^\top),
            \mathrm{vec}(\sigma \sigma^\top(x))
        }\\
    \le&
        -\frac{\alpha n}{2} \normtwo{x}^n + 
        \frac{\beta n}{2} \normtwo{x}^{n-2} + 
        \frac{n(n-2)}{2} \normtwo{x}^{n-2} \pi_{1,2}^{\mathrm{F}}(\sigma)
            \bracks{
                1 + \normtwo{x}
            } \\
    =&
        -\frac{\alpha n}{2} \normtwo{x}^n + 
        \frac{n (n-2)}{2} \pi_{1,2}^{\mathrm{F}}(\sigma) \normtwo{x}^{n-1} + 
        \bracks{
            \frac{\beta n}{2} + \frac{n (n-2)}{2} \pi_{1,2}^{\mathrm{F}}(\sigma)
        } 
        \normtwo{x}^{n-2}
    .
}
By Young's inequality,
\eqn{
    \frac{n (n-2)}{2} \pi_{1,2}^{\mathrm{F}}(\sigma) \normtwo{x}^{n-1}
    =&
        \frac{n (n-2)}{2} \pi_{1,2}^{\mathrm{F}}(\sigma) \bracks{
            \frac{8}{\alpha n}
        }^{\frac{n-1}{n}}
        \cdot
        \normtwo{x}^{n-1} \bracks{
            \frac{\alpha n}{8}
        }^{\frac{n-1}{n}} \\
    \le&
        \frac{1}{n} \bracks{
            \frac{n (n-2)}{2} 
        }^n
        \pi_{1,2}^{\mathrm{F}}(\sigma)^n
        \bracks{
            \frac{8}{\alpha n}
        }^{n-1}
        +
        \frac{n-1}{n} \frac{\alpha n}{8} \normtwo{x}^n \\
    =&
        \frac{ (n-2)^n }{ 2^{2n - 3} \alpha^{n-1} } \pi_{1,2}^{\mathrm{F}}(\sigma)^n
        + \frac{\alpha (n-1)}{8} \normtwo{x}^n
    .
}
Similarly,
\eqn{
    \bracks{
        \frac{\beta n}{2} + \frac{n (n-2)}{2} \pi_{1,2}^{\mathrm{F}}(\sigma)
    } 
    \normtwo{x}^{n-2}
    =&
        \bracks{
            \frac{\beta n}{2} + \frac{n (n-2)}{2} \pi_{1,2}^{\mathrm{F}}(\sigma)
        }
        \bracks{
            \frac{8}{\alpha n}
        }^{ \frac{n-2}{n} }
        \cdot
        \normtwo{x}^{n-2}
        \bracks{
            \frac{\alpha n}{8}
        }^{ \frac{n-2}{n} }\\
    \le&
        \frac{2}{n} 
        \bracks{
            \frac{\beta n}{2} + \frac{n (n-2)}{2} \pi_{1,2}^{\mathrm{F}}(\sigma)
        }^{\frac{n}{2}}
        \bracks{
            \frac{\alpha n}{8}
        }^{\frac{n-2}{2}}
        + 
        \frac{\alpha (n-2)}{8}
        \normtwo{x}^n
    .
}
We define the following shorthand notation
\eq{
    \beta_n^{(1)} =& \frac{ (n-2)^n }{ 2^{2n - 3} \alpha^{n-1} } \pi_{1,2}^{\mathrm{F}}(\sigma)^n = 
        \mathcal{O}( d^{\frac{n}{2}} ), \\
    \beta_n^{(2)} =& 
        \frac{2}{n} 
        \bracks{
            \frac{\beta n}{2} + \frac{n (n-2)}{2} \pi_{1,2}^{\mathrm{F}}(\sigma)
        }^{\frac{n}{2}}
        \bracks{
            \frac{\alpha n}{8}
        }^{\frac{n-2}{2}} = 
        \mathcal{O}( d^{\frac{n}{2}} ).
}
Putting things together, we obtain the following bound
\eqn{
    \A \normtwo{x}^n
    \le&
        - \frac{\alpha n}{2} \normtwo{x}^n + 
        \frac{\alpha (n-1)}{8} \normtwo{x}^n +
        \frac{\alpha (n-2)}{8} \normtwo{x}^n + 
        \beta_n^{(1)} + \beta_n^{(2)} \\
    \le&
        - \frac{\alpha n}{4} \normtwo{x}^n + 
        \beta_n,
}
where $\beta_n = \beta_n^{(1)} + \beta_n^{(2)} = \mathcal{O}(d^{\frac{n}{2}})$.
\end{proof}
\begin{lemm}
Suppose $X_t$ is the continuous-time process defined by~(\ref{eq:continuous_general}) 
initiated from some iterate of the Markov chain $X_0$ defined by~(\ref{eq:srk_id}), 
then the fourth moment of $X_t$ is uniformly bounded, i.e.
\eq{
    \Exp{ \normtwo{X_t}^4 } \le \uyap_4', \quad \text{for all} \; t \ge 0,
}
where 
$
\uyap_4' = 
    \uyap_4 + 
    \beta_4 / \alpha
$.
\end{lemm}
\begin{proof}
By Dynkin’s formula~\cite{oksendal2003stochastic} applied to the function $(t, x) \mapsto e^{\alpha t} \normtwo{x}^4$ and Lemma~\ref{lemm:murat_bound},
\eqn{
    e^{\alpha t} \Exp{
        \normtwo{ X_t }^4
        | \F_0
    }
    =&
        \normtwo{X_0}^4 + 
        \int_0^t \Exp{
            \alpha e^{\alpha s} \normtwo{X_s}^4 + 
            e^{\alpha s} \A \normtwo{X_s}^4
            | \F_0
        } \ds \\
    \le&
        \normtwo{X_0}^4 + 
        \int_0^t \Exp{
            \alpha e^{\alpha s} \normtwo{X_s}^4 
            - \alpha e^{\alpha s} \normtwo{X_s}^4 
            + e^{\alpha s} \beta_4 
            | \F_0
        } \ds \\
    =&
        \normtwo{X_0}^4 + 
        \frac{e^{\alpha t} - 1}{\alpha} \beta_4
    .
}
Hence,
\eqn{
    \Exp{ \normtwo{X_t}^4 } =
    \Exp{
        \Exp{
            \normtwo{X_t}^4
            | \F_0
        }
    }
    \le&
         e^{-\alpha t} \Exp{ \normtwo{X_0}^4 } + \beta_4 / \alpha
    \le
        \uyap_4 + \beta_4 / \alpha
    =
        \uyap_4'. 
}
\end{proof}

\begin{lemm}[Fourth Moment of Change] \label{lemm:continuous_4mom_change_id}
Suppose $X_t$ is the continuous-time process defined by~(\ref{eq:continuous_general}) from some 
iterate of the Markov chain $X_0$ defined by~(\ref{eq:srk_id}), then 
\eqn{
    \Exp{\normtwo{X_t - X_0}^4 } \le D_1 t^2, \quad \text{for all}\; 0 \le t \le 1,
}
where $D_1 =
8 \bracks{
    \pi_{1,4}(b) + 
    36 \pi_{1,4}^{\mathrm{F}} (\sigma)
} 
\bracks{1 + \uyap_4'}
$.
\end{lemm}
\begin{proof}
By Theorem~\ref{theo:moment_ineq}, 
\eqn{
    \Exp{\normtwo{X_t - X_0}^4} 
    =&
        \Exp{ \normtwo{ \int_0^t b (X_s) \ds + \int_0^t \sigma(X_s) \dB_s }^4 } \notag \\
    \le&
        8 \Exp{
            \normtwo{ \int_0^t b(X_s) \ds }^4 + 
            \normtwo{ \int_0^t \sigma(X_s) \dB_s}^4 
        } \\
    \le&
        8 t^3 \int_0^t \Exp{ \normtwo{ b(X_s) }^4 } \ds + 
        288 t \Exp{
            \int_0^t \normf{ \sigma(X_s)}^4 \ds
        } \\
    \le&
        8 \bracks{
            \pi_{1,4}(b) + 
            36 \pi_{1,4}^{\mathrm{F}} (\sigma)
        } \bracks{1 + \uyap_4'} t^2. 
}
\end{proof}
\subsubsection{Local Mean-Square Deviation}
\begin{lemm} \label{lemm:id_discretization_meansqr}
Suppose $X_t$ and $\tX_t$ are the continuous-time process defined by~(\ref{eq:continuous_general}) and Markov chain defined by~(\ref{eq:srk_id}) for time $t \ge 0$, respectively. If $X_t$ and $\tX_t$ are initiated from the same iterate of the Markov chain $X_0$, and they share the same Brownian motion, then 
\eqn{
    \Exp{ \normtwo{X_t - \tX_t}^2 } \le D_3 t^3,
    \quad \text{for all} \; 0 \le t \le 1,
}
where 
\eqn{
    D_3 =& \Bigl(
                16 D_0 \mu_1(b)^2 + 
                \frac{16}{3} \mu_2(\sigma)^2 \pi_{1,4}^{1/2} D_1^{1/2} (1 + \uyap_2'^{1/2} )m^2 +
                \frac{16}{3} \mu_1(\sigma)^4 m^2 D_0
            \\&+
                16 \mu_1(\sigma)^2 \pi_{1,2}(b)^2 (1 + \uyap_2') m + 
                4m^3 \mu_2(\sigma)^2 \pi_{1,4}^{\mathrm{F}} (\sigma) (1 + \uyap_2')
            \\&+
                2^7 3^4 m^2 \mu_2(\sigma)^2 \pi_{1,4}^{\mathrm{F}} (\sigma) \bracks{
                    1 + \uyap_2
                }
        \Bigr). 
}
\end{lemm}

\begin{proof}
Recall the operators $L$ and $\Lambda_i$ ($i=1, \dots, m$) defined in~\eqref{eq:ito_taylor_ops}.
By It\^o's lemma, 
\eqn{
    X_t - X_0
    =&  \int_0^t b(X_s) \ds + 
        \sigma(X_0) B_t   \\
        &+\sum_{i=1}^m \sum_{l=1}^m \int_0^t \int_0^s \Lambda_l (\sigma_i) (X_u) \dB_u^{(l)} \dB_s^{(i)}
        +\sum_{i=1}^m \int_0^t \int_0^s L(\sigma_i)(X_u) \du \dB_s^{(i)}\\
    =&
        \int_0^t b(X_s) \ds + 
        \sigma(X_0) B_t +
        \sum_{i=1}^m \sum_{l=1}^m \int_0^t \int_0^s \nabla \sigma_i (X_u) \sigma_l(X_u) \dB_u^{(l)} \dB_s^{(i)}+ S(t),
}
where
{
\footnotesize
    \eqn{
        S(t)
        =&
            \underbrace{
                \sum_{i=1}^m \int_0^t \int_0^s \nabla \sigma_i(X_u) b(X_u) \du \dB_s^{(i)}
            }_{S_1(t)} +
            \underbrace{
                \frac{1}{2} \sum_{i=1}^m \sum_{l=1}^m \int_0^t 
                    \int_0^s \nabla^2 \sigma_i(X_u) [\sigma_l(X_u),\, \sigma_l(X_u)]\du \dB_s^{(i)}
            }_{S_2(t)}. 
    }
}
By Taylor's theorem with the remainder in integral form,
\eq{
    \sigma_i (\tH_1^{(i)}) &= 
        \sigma_i(X_0) +
        \nabla \sigma_i(X_0) \Delta \tH^{(i)} + 
        \phi_1^{(i)}(t), \\
    \sigma_i (\tH_2^{(i)}) &= 
        \sigma_i(X_0) 
        -\nabla \sigma_i(X_0) \Delta \tH^{(i)} + 
        \phi_2^{(i)}(t),
}
where
\eq{
    \phi_1^{(i)}(t) &= \int_0^1 (1 - \tau) \nabla^2 \sigma_i (X_0 + \tau \Delta \tH^{(i)} ) 
        [\Delta \tH^{(i)}, \, \Delta \tH^{(i)}] \dtau, \\
    \phi_2^{(i)}(t) &= \int_0^1 (1 - \tau) \nabla^2 \sigma_i (X_0 - \tau \Delta \tH^{(i)} ) 
        [\Delta \tH^{(i)}, \, \Delta \tH^{(i)}] \dtau, \\
    \Delta \tH^{(i)} &= \sum_{l=1}^m \sigma_l(X_0) \frac{I_{(l, i)}}{\sqrt{t}}.
}
Hence, 
\eqn{
    X_t - \tX_t =& 
        \int_0^t \bracks{
            b(X_s) - b(X_0)
        } \ds \\
        &+ 
        \sum_{i=1}^m \sum_{l=1}^m
        \int_0^t \int_0^s \bracks{
            \nabla \sigma_i(X_u) \sigma_l (X_u)  - 
            \nabla \sigma_i(X_0) \sigma_l (X_0)
        } \dB_u^{(l)} \dB_s^{(i)} \\
        &+
        S(t) - \frac{1}{2} \sum_{i=1}^m \bracks{\phi_1^{(i)}(t) - \phi_2^{(i)}(t)} \sqrt{t}. 
}
Since $b$ is $\mu_1(b)$-Lipschitz,
\eqn{
    \Exp{
        \normtwo{
            \int_0^t \bracks{ b(X_s) - b(X_0) } \ds
        }^2
    }
    \le&
        \mu_{1} (b)^2 t \int_0^t \Exp{ \normtwo{X_s - X_0}^2  }\ds \\
    \le&
        \mu_{1} (b)^2 t \int_0^t D_0 s \ds \\
    \le&
        \frac{1}{2} D_0 \mu_{1} (b)^2 t^3. 
}
We define the following,
\eqn{
    A(t)
    =&
        A_1 (t) + A_2 (t)
    = \sum_{i=1}^m \sum_{l=1}^m
        \int_0^t \int_0^s \bracks{
            \nabla \sigma_i(X_u) \sigma_l (X_u)  - 
            \nabla \sigma_i(X_0) \sigma_l (X_0)
        } \dB_u^{(l)} \dB_s^{(i)}, 
}
where 
\eq{
    A_1 (t) =&
        \sum_{i=1}^m \sum_{l=1}^m
            \int_0^t \int_0^s \bracks{
                \nabla \sigma_i(X_u) \sigma_l (X_u)  - 
                \nabla \sigma_i(X_0) \sigma_l (X_u)
            } \dB_u^{(l)} \dB_s^{(i) }, \\
    A_2 (t) =&
        \sum_{i=1}^m \sum_{l=1}^m
        \int_0^t \int_0^s \bracks{
            \nabla \sigma_i(X_0) \sigma_l (X_u)  - 
            \nabla \sigma_i(X_0) \sigma_l (X_0)
        } \dB_u^{(l)} \dB_s^{(i)}.
}
By It\^o isometry and the Cauchy-Schwarz inequality,
\eqn{
    \Exp{ \normtwo{A_1 (t)}^2 } =& 
        \sum_{i=1}^m \sum_{l=1}^m \int_0^t \int_0^s \Exp{ 
            \normtwo{
                \nabla \sigma_i(X_u) \sigma_l (X_u) - \nabla \sigma_i(X_0) \sigma_l (X_u)
            }^2
        }\du \ds \\
    \le&
        \sum_{i=1}^m \sum_{l=1}^m \int_0^t \int_0^s \Exp{ 
            \normop{ \nabla \sigma_i(X_u) - \nabla \sigma_i(X_0) }^2
            \normtwo{ \sigma_l (X_u) }^2
        }\du \ds \\
    \le&
        \sum_{i=1}^m \sum_{l=1}^m \int_0^t \int_0^s 
        \Exp{ \normop{ \nabla \sigma_i(X_u) - \nabla \sigma_i(X_0) }^4 }^{1/2}
        \Exp{ \normtwo{ \sigma_l (X_u) }^4 }^{1/2}
        \du \ds \\
    \le& 
        \mu_2 (\sigma)^2 \pi_{1, 4} (\sigma)^{1/2}
        m^2 \int_0^t \int_0^s 
        \Exp{ \normtwo{ X_u - X_0 }^4 }^{1/2}
        \Exp{
            1 + \normtwo{ X_u }^2
        }^{1/2}
        \du \ds \\
    \le&
        \mu_2 (\sigma)^2 \pi_{1, 4} (\sigma)^{1/2}
        D_1^{1/2} \bracks{ 1 + \uyap_2'^{1/2} } m^2 \frac{t^3}{6} 
    . \numberthis \label{eq:A_1}
}
Similarly,
\eqn{
    \Exp{ \normtwo{A_2 (t)}^2 } =& 
        \sum_{i=1}^m \sum_{l=1}^m \int_0^t \int_0^s \Exp{ 
            \normop{
                \nabla \sigma_i(X_0) \sigma_l (X_u) - \nabla \sigma_i(X_0) \sigma_l (X_0)
            }^2
        }\du \ds \\
    \le&
        \sum_{i=1}^m \sum_{l=1}^m \int_0^t \int_0^s \Exp{ 
            \normop{ \nabla \sigma_i(X_0) }^2
            \normtwo{ \sigma_l (X_u) - \sigma_l (X_0) }^2
        }\du \ds \\
    \le&
        \mu_1(\sigma)^2 \sum_{i=1}^m \sum_{l=1}^m \int_0^t \int_0^s 
        \Exp{ \normtwo{ \sigma_l (X_u) - \sigma_l (X_0) }^2 }
        \du \ds \\
    \le&
        \mu_1(\sigma)^4 m^2 \int_0^t \int_0^s 
        \Exp{ \normtwo{ X_u - X_0 }^2 }
        \du \ds \\
    \le&
        \frac{1}{6} \mu_1(\sigma)^4 m^2 D_0 t^3
    . \numberthis \label{eq:A_2}
}
By It\^o isometry, 
\eqn{
    \Exp{
        \normtwo{S_1(t)}^2
    }
    =&
        \sum_{i=1}^m \int_0^t s \int_0^s \Exp{
            \normtwo{ \nabla \sigma_i(X_u) b(X_u) }^2
        } \du \ds \\
    \le&
        \sum_{i=1}^m \int_0^t s \int_0^s \Exp{
            \normop{ \nabla \sigma_i(X_u) }^2 \normtwo{b(X_u)}^2 
        } \du \ds \\
    \le&
        \mu_1(\sigma)^2  \pi_{1,2}(b)^2
        \sum_{i=1}^m \int_0^t s \int_0^s 
        \Exp{1 + \normtwo{X_u}^2 }
        \du \ds \\
    =&
        \frac{1}{2} \mu_1(\sigma)^2  \pi_{1,2}(b)^2 \bracks{1 + \uyap_2' } m t^3
    . \numberthis \label{eq:S_1}
}
Similarly,
\eqn{
    \Exp{ \normtwo{S_2(t)}^2 }
    =&
        \frac{1}{4} \sum_{i=1}^m \int_0^t \Exp{
            \normtwo{
                \int_0^s \sum_{l=1}^m \nabla^2 \sigma_i(X_u) [\sigma_l(X_u), \, \sigma_l(X_u)] \du
            }^2 \ds
        } \\
    \le&
        \frac{1}{4} m \sum_{i=1}^m \sum_{l=1}^m \int_0^t s \int_0^s \Exp{
            \normtwo{
                \nabla^2 \sigma_i(X_u) [\sigma_l(X_u), \, \sigma_l(X_u)]
            }^2
        } \du \ds \\
    \le&
        \frac{1}{4} m \sum_{i=1}^m \sum_{l=1}^m \int_0^t s \int_0^s \Exp{
            \normop{ \nabla^2 \sigma_i(X_u) }^2
            \normtwo{\sigma_l(X_u)}^4
        } \du \ds \\
    \le&
        \frac{1}{4} \sigma_2(\sigma)^2 \pi_{1, 4}(\sigma) m^3
        \int_0^t s \int_0^s \Exp{1 + \normtwo{X_u}^2 }
        \du \ds \\
    \le&
        \frac{1}{8} \sigma_2(\sigma)^2 \pi_{1, 4}(\sigma) m^3
        \bracks{1 + \uyap_2'} t^3. \numberthis \label{eq:S_2}
}
By Corollary~\ref{coro:delta_H_p},
\eq{
    \Exp{
        \normtwo{ \Delta \tH^{(i)}}^4
    }
    =&
        \Exp{
            \Exp{
                \normtwo{ \Delta \tH^{(i)}}^4 
                | \F_{t_k}
            }
        } \\
    \le&
        6^4 \pi_{1,4}^{\mathrm{F}} (\sigma) \Exp{
            1 + \normtwo{\tX_k}^2
        } t^2 \\
    \le&
        6^4 \pi_{1,4}^{\mathrm{F}} (\sigma) \bracks{
            1 + \uyap_2
        } t^2.
}
Now, we bound the second moments of $\phi_1^{(i)}(t)$ and $\phi_2^{(i)}(t)$,
\eqn{
    \Exp{ \normtwo{ \phi_1^{(i)}(t) }^2 } 
    =&
        \Exp{
            \normtwo{
                \int_0^1 (1 - \tau) 
                \nabla^2 \sigma_i (X_0 + \tau \Delta \tH^{(i)} ) 
                [\Delta \tH^{(i)}, \, \Delta \tH^{(i)}] \dtau
            }^2
        } \\
    \le&
        \Exp{
            \normop{ \nabla^2 \sigma_i (X_0 + \tau \Delta \tH^{(i)} ) }^2
            \normtwo{ \Delta \tH^{(i)} }^4
        } \\
    \le&
        6^4 \mu_{2} (\sigma)^2  \pi_{1,4}^{\mathrm{F}} (\sigma) (1 + \uyap_2) t^2
        . \numberthis \label{eq:phi_1}
}
Similarly,
\eqn{
    \Exp{ \normtwo{ \phi_2^{(i)}(t) }^2 } 
    \le&
        6^4 \mu_{2} (\sigma)^2  \pi_{1,4}^{\mathrm{F}} (\sigma) (1 + \uyap_2) t^2
    . \numberthis \label{eq:phi_2}
}
Hence, by \eqref{eq:phi_1} and \eqref{eq:phi_2},
\eqn{
    \Exp{ \normtwo{ \frac{1}{2} \sum_{i=1}^m \bracks{
        \phi_1^{(i)}(t) - \phi_2^{(i)}(t) } \sqrt{t} }^2
    }
    \le&
        \frac{m}{4} t \sum_{i=1}^m \Exp{
            \normtwo{ \phi_1^{(i)}(t) - \phi_2^{(i)}(t) }^2
        } \\
    \le&
        2^2 3^4 m^2 \mu_{2} (\sigma)^2  \pi_{1,4}^{\mathrm{F}} (\sigma) (1 + \uyap_2) t^3. 
    \numberthis \label{eq:phi_diff}
}
Combining \eqref{eq:A_1}, \eqref{eq:A_2}, \eqref{eq:S_1}, \eqref{eq:S_2}, and \eqref{eq:phi_diff},
\eqn{
    \Exp{ \normtwo{X_t - \tX_t}^2 }
    \le&
        32 \Exp{ \normtwo{ \int_0^t \bracks{ b(X_s) - b(X_0) } \ds }^2 }
        \\&+
        32 \Exp{ \normtwo{A_1(t)}^2 + \normtwo{A_2(t)}^2 }
        \\&+
        32 \Exp{ \normtwo{S_1(t)}^2 + \normtwo{S_2(t)}^2 }
        \\&+
        32 \Exp{ \normtwo{ \frac{1}{2} \sum_{i=1}^m \bracks{
            \phi_1^{(i)}(t) - \phi_2^{(i)}(t) } \sqrt{t} }^2
        } \\
    \le&
        \Bigl(
                16 D_0 \mu_1(b)^2 + 
                \frac{16}{3} \mu_2(\sigma)^2 \pi_{1,4}^{1/2} D_1^{1/2} (1 + \uyap_2'^{1/2} )m^2 +
                \frac{16}{3} \mu_1(\sigma)^4 m^2 D_0
            \\&+
                16 \mu_1(\sigma)^2 \pi_{1,2}(b)^2 (1 + \uyap_2') m + 
                4m^3 \mu_2(\sigma)^2 \pi_{1,4}^{\mathrm{F}} (\sigma) (1 + \uyap_2')
            \\&+
                2^7 3^4 m^2 \mu_2(\sigma)^2 \pi_{1,4}^{\mathrm{F}} (\sigma) \bracks{
                    1 + \uyap_2
                }
        \Bigr) t^3. 
}
\end{proof}

\subsubsection{Local Mean Deviation}
\begin{lemm} \label{lemm:id_discretization_mean}
Suppose $X_t$ and $\tX_t$ are the continuous-time process defined by~(\ref{eq:continuous_general}) and Markov chain defined by~(\ref{eq:srk_id}) for time $t \ge 0$, respectively. If $X_t$ and $\tX_t$ are initiated from the same iterate of the Markov chain $X_0$, and they share the same Brownian motion, then 
\eqn{
    \Exp{
        \normtwo{ \Exp{ X_t - \tX_t  | \F_{0} } }^2
    } \le D_4 t^4,
    \quad \text{for all} \; 0 \le t \le 1,
}
where 
\eqn{
    D_4 =& 
        \Bigl(
            \frac{4}{3} \mu_1(b)^2 \pi_{1, 2} (b) \bracks{1 + \uyap_2'}  + 
            \frac{1}{3} \mu_2(b)^2 \pi_{1, 4} (\sigma) \bracks{1 + \uyap_2'} m^2
            \\& \phantom{11}
            + 
                2^4 3^5 5^6 \mu_3(\sigma)^2 \pi_{1,6}^{\mathrm{F}}(\sigma) \bracks{
                    1 + \uyap_4^{3/4}
                }
        \Bigr). 
}
\end{lemm}

\begin{proof}
Recall the operators $L$ and $\Lambda_i$ ($i=1, \dots, m$) defined in~\eqref{eq:ito_taylor_ops}.
By It\^o's lemma, 
\eqn{
    X_t - X_0 
    =& 
        \,b(X_0) t
        + \sum_{i=1}^m \int_0^t \sigma_i(X_s) \dB_s^{(i)} \\
        &+ \sum_{i=1}^m \int_0^t \int_0^s \Lambda_i (b) (X_u) \dB_u^{(i)} \ds
        + \int_0^t \int_0^s L(b) (X_u) \du \ds\\
    =& 
        \,b(X_0) t
        + \sum_{i=1}^m \int_0^t \sigma_i(X_s) \dB_s^{(i)}
        + \sum_{i=1}^m \int_0^t \int_0^s  \nabla b (X_u) \sigma_i(X_u) \dB_u^{(i)} \ds
        + \bar{S}(t),
}
where 
\eq{
    \bar{S}(t) =& 
        \underbrace{
            \int_0^t \int_0^s \nabla b (X_u) b(X_u) \du \ds
        }_{\bar{S}_1(t)} + 
        \underbrace{
            \frac{1}{2} \sum_{i=1}^m \int_0^t \int_0^s \nabla^2 b(X_u) 
            [\sigma_i(X_u),\, \sigma_i(X_u)] \du \ds
        }_{\bar{S}_2(t)}
    . 
}
Now, we bound the second moments of $\bar{S}_1(t)$ and $\bar{S}_2(t)$,
\eqn{
    \Exp{ \normtwo{ \bar{S}_1(t) }^2 }
    =&
        \Exp{
            \normtwo{
                \int_0^t \int_0^s \nabla b (X_u) b(X_u) \du \ds
            }^2
        } \\
    \le&
        t \int_0^t s \int_0^s \Exp{
            \normtwo{ \nabla b(X_u) b(X_u) }^2
        } \du \ds \\
    \le&
        t \int_0^t s \int_0^s \Exp{
            \normop{ \nabla b(X_u)}^2 \normtwo{b(X_u) }^2
        } \du \ds \\
    \le&
        \mu_1(b)^2 \pi_{1, 2} (b) t \int_0^t s \int_0^s \Exp{
            1 + \normtwo{X_u}^2
        } \du \ds \\
    \le&
        \frac{1}{3} \mu_1(b)^2 \pi_{1, 2} (b) \bracks{1 + \uyap_2'} t^4. 
    \numberthis \label{eq:bars_1}
}
Similarly,
\eqn{
    \Exp{ \normtwo{ \bar{S}_2(t) }^2 }
    =&
        \Exp{
            \normtwo{
                \frac{1}{2} \sum_{i=1}^m \int_0^t \int_0^s \nabla^2 b(X_u) 
                [\sigma_i(X_u),\, \sigma_i(X_u)] \du \ds 
            }^2
        } \\
    \le&
        \frac{m}{4} \sum_{i=1}^m 
            t \int_0^t s \int_0^s \Exp{
                \normtwo{ \nabla^2 b(X_u) [\sigma_i(X_u), \, \sigma_i(X_u)] }^2
            } \du \ds \\
    \le&
        \frac{m}{4} \sum_{i=1}^m 
            t \int_0^t s \int_0^s \Exp{
                \normop{ \nabla^2 b(X_u)}^2 \normtwo{ \sigma_i(X_u) }^4 
            } \du \ds \\
    \le&
        \frac{m}{4} \mu_2(b)^2 \sum_{i=1}^m 
            t \int_0^t s \int_0^s \Exp{
                \normtwo{ \sigma_i(X_u) }^4 
            } \du \ds \\
    \le&
        \frac{m^2}{4} \mu_2(b)^2 \pi_{1, 4} (\sigma)
            t \int_0^t s \int_0^s 
            \Exp{1 + \normtwo{X_u}^2 } \du \ds \\
    \le&
        \frac{1}{12} \mu_2(b)^2 \pi_{1, 4} (\sigma) \bracks{1 + \uyap_2'} m^2 t^4. 
    \numberthis \label{eq:bars_2}
}
By Corollary~\ref{coro:delta_H_p},
\eqn{
    \Exp{
        \normtwo{ \Delta \tH^{(i)}}^6
    }
    =&
        \Exp{
            \Exp{
                \normtwo{ \Delta \tH^{(i)}}^6
                | \F_{t_k}
            }
        } \\
    \le&
        3^6 5^6 \pi_{1,6}^{\mathrm{F}} (\sigma) \Exp{
            1 + \normtwo{\tX_k}^3
        } t^3 \\
    \le&
        3^6 5^6 \pi_{1,6}^{\mathrm{F}} (\sigma) \bracks{
            1 + \uyap_4^{3/4}
        } t^3. 
}
Now, we bound the second moment of the difference between $\phi_1^{(i)}(t)$ and $\phi_2^{(i)}(t)$,
\eqn{
    \Exp{ \normtwo{ \phi_1^{(i)}(t) - \phi_2^{(i)}(t) }^2 }
    \le&
        \text{
        \footnotesize
            $\Exp{
                \int_0^1 
                \normop{ 
                    \nabla^2 \sigma_i(X_0 + \tau \Delta \tH^{(i)}) - 
                    \nabla^2 \sigma_i(X_0 - \tau \Delta \tH^{(i)})
                }^2
                \normtwo{ \Delta \tH^{(i)} }^4
                \dtau
            }$
        } \\
    \le&
        \mu_3(\sigma)^2 \int_0^1 \Exp{
            \normtwo{
                2 \tau \Delta \tH^{(i)}
            }^2
            \normtwo{ \Delta \tH^{(i)} }^4
        } \dtau  \\
    \le&
        \frac{4}{3} \mu_3(\sigma)^2 \Exp{ \normtwo{ \Delta \tH^{(i)} }^6 } \\
    \le&
        2^2 3^5 5^6 \mu_3(\sigma)^2 \pi_{1,6}^{\mathrm{F}}(\sigma) \bracks{
            1 + \uyap_4^{3/4}
        } t^3
    . 
    \numberthis \label{eq:phi_diff'}
}
Hence, combining \eqref{eq:bars_1}, \eqref{eq:bars_2}, and \eqref{eq:phi_diff'},
\eqn{
    \Exp{
        \normtwo{
            \Exp{ X_t - \tX_t | \F_0}
        }^2
    }
    =&
        \Exp{
            \normtwo{
                \Exp{
                    \bar{S}(t)
                    | \F_0
                }
                - 
                \Exp{
                    \frac{1}{2} \sum_{i=1}^m \bigl( 
                        \phi_1^{(i)}(t) - \phi_2^{(i)}(t)
                    \bigr)
                    \sqrt{t} 
                    | \F_0
                }
            }^2 
        }\\
    \le&
        4 \Exp{ \normtwo{ \bar{S}_1(t) }^2 + \normtwo{ \bar{S}_2(t) }^2 } +  
        4 \Exp{
            \Bigl\|
                \frac{1}{2} \sum_{i=1}^m \bigl( 
                    \phi_1^{(i)}(t) - \phi_2^{(i)}(t)
                \bigr)
                \sqrt{t} 
            \Bigr\|_2^2
        } \\
    \le&
        \Bigl(
            \frac{4}{3} \mu_1(b)^2 \pi_{1, 2} (b) \bracks{1 + \uyap_2'}  + 
            \frac{1}{3} \mu_2(b)^2 \pi_{1, 4} (\sigma) \bracks{1 + \uyap_2'} m^2
            \\&+ 
                2^4 3^5 5^6 \mu_3(\sigma)^2 \pi_{1,6}^{\mathrm{F}}(\sigma) \bracks{
                    1 + \uyap_4^{3/4}
                }
        \Bigr) t^4.  
}
\end{proof}

\subsection{Invoking Theorem~\ref{theo:master}}
Now, we invoke Theorem~\ref{theo:master} with our derived constants. We obtain that if the constant step size
\eq{
    h <
    1
    \wedge C_h 
    \wedge \frac{1}{2\alpha}
    \wedge \frac{1}{8\mu_1(b)^2 + 8 \mu_1^{\mathrm{F}}(\sigma)^2},
}
where
\eq{
        C_h =
        \frac{1}{m^2}
        \wedge \frac{\alpha^2}{4M_1^2}
        \wedge \frac{\alpha^2 }{M_{3, 2}^2 },
}
and the smoothness conditions in Theorem~\ref{theo:srk_id} of the drift and diffusion coefficients are satisfied for 
a uniformly dissipative diffusion,
then the uniform local deviation bounds~\eqref{eq:uniform_local_deviation_orders} hold with $\lambda_1 = D_3$ and $\lambda_2 = D_4$, and consequently the bound~\eqref{eq:main_bound} holds. 
This concludes that to converge to a sufficiently small positive tolerance $\epsilon$, $\tilde{\mathcal{O}} (d^{3/4}m^2 \epsilon^{-1} )$ iterations are required, since $D_3$ is of order $\mathcal{O}(d^{3/2} m^3)$, and $D_4$ is of order $\mathcal{O}(d^{3/2} m^2)$. 
Note that the dimension dependence worsens if one were to further convert the Frobenius norm dependent constants to be based on the operator norm.

\section{Convergence Rate for Example~\ref{exam:em_unif_diss}} \label{app:em_unif_diss}
\subsection{Moment Bound}
Verifying the order conditions in Theorem~\ref{theo:master} of the EM scheme for uniformly dissipative diffusions 
requires bounding the second moments of the Markov chain. Recall, dissipativity (Definition~\ref{defi:dissipativity_general}) follows from uniform dissipativity of the It\^o diffusion.
\begin{lemm} \label{lemm:second_moment_em}
If the second moment of the initial iterate is finite, then 
the second moments of Markov chain iterates defined in \eqref{eq:em_general} are uniformly bounded, i.e.
\eqn{
    \Exp{\normtwo{\tX_k}^{2}} \le \wyap_{2}, 
    \quad \text{for all} \; {k \in \mathbb{N}},
}
where $\wyap_2 = \Exp{ \bigl\| \tX_0 \bigr\|_2^2 } + 2(\pi_{1,2}(b) + \beta) / \alpha$, 
if the constant step size $h < 1 \wedge \alpha / (2 \pi_{1,2}(b))$.
\end{lemm}
\begin{proof}
By direct computation,
\eqn{
    \normtwo{ \tX_{k+1} }^2
    =&
        \normtwo{\tX_k}^2 + \normtwo{b(\tX_k)}^2 h^2 + \normtwo{\sigma(\tX_k) \xi_{k+1} }^2 h
        \\&
        + 2 \abracks{ \tX_k, b(\tX_k) } h
        + 2 \abracks{ \tX_k, \sigma(\tX_k) \xi_{k+1} } h^{1/2}
        \\&
        + 2 \abracks{ b(\tX_k), \sigma(\tX_k) \xi_{k+1}} h^{3/2}.
}
Recall by Lemma~\ref{lemm:quadratic} and dissipativity,
\eqn{
    \Exp{
        2\abracks{ \tX_k, b(\tX_k) } h + \normtwo{\sigma(\tX_k) \xi_{k+1} }^2 h
        | \F_{t_k}
    }
    \le&
        -\alpha \normtwo{\tX_k}^2 h + \beta h.
}
By odd moments of Gaussian variables being zero and the step size condition,
\eqn{
    \Exp{
        \normtwo{ \tX_{k+1} }^2
        | \F_{t_k}
    }
    \le&
        (1 - \alpha h) \normtwo{\tX_k}^2 + \normtwo{b(\tX_k)}^2 h^2 + \beta h \\
    \le&
        (1 - \alpha h + \pi_{1,2} (b) h^2 ) \normtwo{\tX_k}^2 + \pi_{1,2}(b) h^2 + \beta h \\
    \le&
        \bracks{ 1 - \alpha h / 2 } \normtwo{\tX_k}^2 + \pi_{1,2}(b) h^2 + \beta h.
}
By unrolling the recursion,
\eq{
    \Exp{ \normtwo{\tX_{k}}^2 }
    \le&
        \Exp{ \normtwo{\tX_0}^2 } + 2(\pi_{1,2}(b) + \beta) / \alpha,
    \quad \text{for all} \; k \in \mathbb{N}
    .
}
\end{proof}

\subsection{Local Deviation Orders}
Before verifying the local deviation orders, we first state two 
auxiliary lemmas. We omit the proofs, since they are almost 
identical to that of Lemma~\ref{lemm:second_moment_cont} and Lemma~\ref{lemm:continuous_2mom_change}, 
respectively.
\begin{lemm}
Suppose $X_t$ is the continuous-time process defined by \eqref{eq:continuous_general} 
initiated from some iterate of the Markov chain $X_0$ defined by \eqref{eq:em_general}, 
then the second moment of $X_t$ is uniformly bounded, i.e.
\eq{
    \Exp{ \normtwo{X_t}^2 } \le \wyap_2 + \beta / \alpha = \wyap_2', \quad \text{for all} \; t \ge 0.
}
\end{lemm}
\begin{lemm}
Suppose $X_t$ is the continuous-time process defined by \eqref{eq:continuous_general}
initiated from some iterate of the Markov chain $X_0$ defined by \eqref{eq:em_general}, then
\eq{
    \Exp{\normtwo{X_t - X_0}^2 } \le E_0 t
    , \quad \text{for all} \; t \ge 0,
}
where $E_0 = 2 \bracks{ \pi_{1,2}(b) + \pi_{1,2}^{\mathrm{F}} (\sigma)} \bracks{1 + \wyap_2'}$.
\end{lemm}

\subsubsection{Local Mean-Square Deviation}
\begin{lemm}
Suppose $X_t$ and $\tX_t$ are the continuous-time process defined 
by \eqref{eq:continuous_general} and Markov chain defined by \eqref{eq:em_general}
for time $t \ge 0$, respectively. If $X_t$ and $\tX_t$ are initiated from the same iterate 
of the Markov chain $X_0$ and share the same Brownian motion, then
\eq{
    \Exp{ \normtwo{X_t - \tX_t}^2 } \le E_1 t^2
    , \quad \text{for all} \; 0 \le t \le 1,
}
where 
$E_1 = \bracks{ 
    \mu_1(b)^2 + \mu_1^{\mathrm{F}} (\sigma)^2
} E_0$.
\end{lemm}
\begin{proof}
By It\^o isometry and Lipschitz of the drift and diffusion coefficients,
\eqn{
    \Exp{ \normtwo{X_t - \tX_t}^2 } 
    \le&
        2 \Exp{  \normtwo{ \int_0^t \bracks{ b(X_s) - b(X_0) } \ds }^2 }
        + 
        2 \Exp{ \normtwo{ \int_0^t \bracks{ \sigma(X_s) - \sigma(X_0) } \dB_s }^2 } \\
    \le&
        2t \Exp{ \int_0^t \normtwo{ b(X_s) - b(X_0) }^2 \ds } + 
        2 \Exp{ \int_0^t \normf{\sigma(X_s) - \sigma(X_0) }^2 \ds } \\
    \le&
        2 \bracks{ 
            \mu_1(b)^2 t + \mu_1^{\mathrm{F}} (\sigma)^2
        } \int_0^t \Exp{ \normtwo{ X_s - X_0 }^2 } \ds \\
    \le&
        \bracks{ 
            \mu_1(b)^2 + \mu_1^{\mathrm{F}} (\sigma)^2
        } E_0 t^2.
}
\end{proof}

\subsubsection{Local Mean Deviation} 
\begin{lemm}
Suppose $X_t$ and $\tX_t$ are the continuous-time process defined by \eqref{eq:continuous_general} and Markov chain defined by\eqref{eq:em_general} for time $t \ge 0$, respectively. If $X_t$ and $\tX_t$ are initiated from the same iterate of the Markov chain $X_0$ and share the same Brownian motion, then
\eqn{
    \Exp{
        \normtwo{ \Exp{ X_t - \tX_t  | \F_{0} } }^2
    } \le E_2 t^3
    , \quad \text{for all} \; 0 \le t \le 1,
}
where 
$E_2 = \mu_1(b) E_0 / 2$.
\end{lemm}
\begin{proof}
By It\^o's lemma,
\eqn{
    X_t - X_0 =& 
    \int_0^t b(X_s) ds + \sigma(X_0) B_t
    \\&
    + \sum_{i=1}^m \sum_{l=1}^m \int_0^t \int_0^s \Lambda_l (\sigma_i)(X_u) \dB_u^{(l)} \dB_s^{(i)}
    + \sum_{i=1}^m \int_0^t \int_0^s L(\sigma_i) (X_u) \du \dB_s^{(i)}.
}
Since the last two terms in the above inequality are Martingales, 
\eq{
    \Exp{ X_t - X_0 | \F_0 }
    = \Exp{
        \int_0^t \bracks{ b(X_s) - b(X_0) } \ds
        | \F_0
    }.
}
Hence, by Jensen's inequality,
\eqn{
    \Exp{
        \normtwo{
            \Exp{ X_t - \tX_t  | \F_{0} }
        }^2
    }
    =&
        \Exp{
            \normtwo{ 
                \Exp{
                    \int_0^t \bracks{ b(X_s) - b(X_0) } \ds
                | \F_0
                }
            }^2
        } \\
    \le&
        \Exp{
            \normtwo{
                \int_0^t \bracks{ b(X_s) - b(X_0) } \ds
            }^2
        } \\
    \le&
        \mu_1(b) t \int_0^t \Exp{
            \normtwo{ X_s - X_0 }^2
        } \ds \\
    \le&
        \mu_1(b) E_0 t^3 / 2.
}
\end{proof}

\subsection{Invoking Theorem~\ref{theo:master}}
Now, we invoke Theorem~\ref{theo:master} with our derived constants. We obtain that if the constant
step size
\eq{
    h <
    1
    \wedge \frac{\alpha}{2 \pi_{1,2}(b)}
    \wedge \frac{1}{2\alpha}
    \wedge \frac{1}{8\mu_1(b)^2 + 8 \mu_1^{\mathrm{F}}(\sigma)^2},
}
and the smoothness conditions of the drift and diffusion coefficients are satisfied for a 
uniformly dissipative diffusion,
then the uniform local deviation bounds \eqref{eq:uniform_local_deviation_orders}
hold with $\lambda_1 = E_1$ and $\lambda_2 = E_2$, and consequently the bound \eqref{eq:main_bound} holds.
This concludes that for a sufficiently small positive tolerance $\epsilon$, $\tilde{\mathcal{O}}(d \epsilon^{-2})$ iterations are required, since both $E_1$ and $E_2$ are of order $\mathcal{O}(d)$. 
If one were to convert the Frobenius norm dependent constants to be based on the operator norm, then $E_1$ is of order $\mathcal{O}( d (d+m)^2 )$, and $E_2$ is of order $\mathcal{O}( d (d + m) )$. This yields the convergence rate of $\tilde{\mathcal{O}}(d (d+m)^2 \epsilon^{-2})$.

\section{Convergence of SRK-LD Under an Unbiased Stochastic Oracle} \label{app:srk-ld-stochasticgrad}

We provide an informal analysis on the scenario where the oracle is stochastic. 
We denote the new interpolated values under the stochastic oracle as $\hat{H}_1$ and $\hat{H}_2$, and the new iterate value as $\hat{X}_k$.
We assume (i) the stochastic oracle is unbiased, i.e. $\E[\hat{\nabla} f(x)] = f(x)$ for all $x\in\R^d$,
(ii) the stochastic oracle has finite variance at the Markov chain iterates and ``interpolated'' values, i.e. 
$\E[ \| \hat{\nabla} f(Y) - \nabla f(Y) \|_2^2 ]\le \sigma^2 d$, for some finite $\sigma$, where $Y$ may be $\hat{X}_k$, $\hat{H}_1$, or $\hat{H}_2$\footnote{There is slight ambiguity in terms of which iteration's interpolated values should $\tH_1$ and $\tH_2$ correspond to. 
For notational simplicity, we have avoided using a subscript or superscript for the iteration index $k$, and almost always make $\tH_1$ and $\tH_2$ appear along with the original iterate $\tX_k$. 
}, and (iii) the randomness in the stochastic oracle is independent of that of the Brownian motion.

Fix iteration index $k \in \mathbb{N}$, let $\tilde{D}_h^{(k)}$ and $\hat{D}_h^{(k)}$ denote the local deviations under the exact and stochastic oracles, respectively. 
Then, assuming the step size is chosen sufficiently small such that the Markov chain moments are bounded, 
\eqn{
	\Exp{
		\normtwo{
			\hat{D}_h^{(k)}
		}^2
	}
	\le&
	2\Exp{
		\normtwo{
			\tilde{D}_h^{(k)}
		}^2
	} + 
	2\Exp{
		\normtwo{
			\tilde{D}_h^{(k)} - \hat{D}_h^{(k)}
		}^2
	}
	\\
	\le&
	2\Exp{
		\normtwo{
			\tilde{D}_h^{(k)}
		}^2
	}
	+
	4 \Exp{
		\normtwo{
			\hat{\nabla} f(\hat{H}_1) - \nabla f(\tH_1)
		}^2
	}
	+
	4 \Exp{
		\normtwo{
			\hat{\nabla} f(\hat{H}_2) - \nabla f(\tH_2)
		}^2
	}
	\\
	\le&
	2\Exp{
		\normtwo{
			\tilde{D}_h^{(k)}
		}^2
	}
	+
	4\sigma^2 d
	+
	4 \Exp{
		\normtwo{
			\hat{\nabla} f(\hat{H}_2) - \nabla f(\hat{H}_2) + \nabla f(\hat{H}_2) - \nabla f(\tH_2)
		}^2
	}
	\\
	\le&
	\mathcal{O}(h^4 + \sigma^2).
}
Similarly, one can derive the new local mean deviation,
\eq{
	\Exp{
		\normtwo{
			\Exp{
				\hat{D}_h^{(k)}
				| 
				\F_{t_{k-1}}
			}
		}^2
	}
	\le&
	\Exp{
		\normtwo{
			\Exp{
				\tilde{D}_h^{(k)}
				| 
				\F_{t_{k-1}}
			}
			+
			\Exp{
				\hat{D}_h^{(k)} - \tilde{D}_h^{(k)}
				| 
				\F_{t_{k-1}}
			}
		}^2
	}
	\\
	\le&
	\Exp{
		\normtwo{
			\Exp{
				\tilde{D}_h^{(k)}
				| 
				\F_{t_{k-1}}
			}
		}^2
	} + 
	\Exp{
		\normtwo{
			\hat{D}_h^{(k)} - \tilde{D}_h^{(k)}
		}^2
	}
	\\
	=&
	\mathcal{O}(h^5 + \sigma^2).
}
One can replace the corresponding terms in \eqref{eq:master_recursion} and obtain a recursion. 
Note however, to ensure unrolling the recursion gives a convergence bound, one would need that $\sigma^2 < \mathcal{O}(\alpha h)$. 

\section{Auxiliary Lemmas}
We list standard results used to develop our theorems and include their proofs for completeness.
\begin{lemm}\label{lemm:a_plus_b_n}
For $x_1, \dots, x_m \in \R$ and $m, n \in \mathbb{N}_{+}$, we have
\eqn{
    \bracks{\sum_{i=1}^m x_i}^n \le m^{n-1} \sum_{i=1}^m x_i^n.
}
\end{lemm}
\begin{proof}
Recall the function $f(x) = x^n$ is convex for $n \in \mathbb{N}_{+}$. Hence,
\begin{align*}
    \bracks{
        \frac{\sum_{i=1}^m x_i}{m}
    }^n \le \frac{\sum_{i=1}^m x_i^n}{m}.
\end{align*}
Multiplying both sides of the inequality by $m^n$ completes the proof.
\end{proof}
\begin{lemm} \label{lemm:integrate_bs_ds}
For the $d$-dimensional Brownian motion $\{B_t\}_{t \ge 0}$,
\eqn{
Z_t = \int_0^t \int_0^s \dB_u \ds \sim \mathcal{N}\bracks{0, t^3 I_d / 3}.
}
\end{lemm}
\begin{proof}
 We consider the case where $d=1$. The multi-dimensional case follows naturally, since we assume different dimensions of the Brownian motion vector are independent. Let $t_k = \delta k$, we define
\eqn{
    S_m 
    &= \sum_{k=0}^{m - 1} B_{t_k} (t_{k+1} - t_k) =
    \sum_{k=1}^{m - 1} \bracks{B_{t_{k+1}} - B_{t_k}} \bracks{t_k - t}.
}
Since $S_m$ is a sum of Gaussian random variables, it is also Gaussian. By linearity of expectation and independence of Brownian motion increments,
\eqn{
    \Exp{S_m} &= 0, \quad \\
    \Exp{S_m^2} &= \sum_{k=1}^{m-1} \bracks{t_k - t}^2 \Exp{
        \bracks{B_{t_{k+1}} - B_{t_k}}^2
    } \to \int_0^t \bracks{s - t}^2 \ds = {t^3}/{3} \quad \text{as} \quad m\to \infty.
}
Since $S_m \overset{\text{a.s.}}{\to} Z_t$ as $m \to \infty$ by the strong law of large numbers, we conclude that $Z_t \sim \mathcal{N}\bracks{0, t^3/3}$.
\end{proof}

\begin{lemm}
For $n \in \mathbb{N}$ and the $d$-dimensional Brownian motion $\{B_t\}_{t \ge 0}$,
\eqn{
    \Exp{
        \norm{B_t}_2^{2n}
    } = t^n d (d + 2) \cdots (d + 2n - 2).
}
\end{lemm}
\begin{proof}
Note $\norm{B_t}_2^2$ may be expressed as the sum of squared Gaussian random variables, i.e. 
\eqn{
    \norm{B_t}_2^2 = t \sum_{i=1}^d \xi_i^2, \quad \text{where} \quad \xi_i \overset{\text{i.i.d.}}{\sim} \N(0, 1).
}
Observe that this is also a multiple of the chi-squared random variable with $d$ degrees of freedom $\chi(d)^2$. Its $n$th moment has the following closed form~\cite{simon2007probability},
\eqn{
    \Exp{ \chi(d)^{2n} } = 2^n \frac{\Gamma \bracks{n + \frac{d}{2}}}{ \Gamma\bracks{\frac{d}{2}} } = d(d + 2) \cdots (d + 2n - 2).
}
Thus,
\eqn{
    \Exp{\norm{B_t}_2^{2n}} = t^n \Exp{\chi(d)^{2n}} = t^n d (d + 2) \cdots (d + 2n - 2) .
}
\end{proof}

\begin{lemm} \label{lemm:laplacian_bounded}
For $f: \R^d \to \R$ which is $C^3$, suppose its Hessian is $\mu_3$-Lipschitz under the operator norm and Euclidean norm, i.e. 
\eqn{
    \norm{ \nabla^2 f(x) - \nabla^2 f(y)}_{\mathrm{op}} \le \mu_3 \norm{x - y}_2, \quad \text{for all}\; x, y\in\R^d.
}
Then, the vector Laplacian of its gradient is bounded, i.e.
\eqn{
    \norm{\Vec{\Delta} (\nabla f) (x)}_2 \le d \mu_3, \quad \text{for all}\; x\in\R^d.
}
\end{lemm}
\begin{proof}
See proof of Lemma 6 in~\cite{dalalyan2019user}. 
\end{proof}

\begin{lemm} \label{lemm:laplacian_lipschitz}
For $f: \R^d \to \R$ which is $C^4$, suppose its third derivative is $\mu_4$-Lipschitz under the operator norm and Euclidean norm, i.e. 
\eqn{
        \norm{ \nabla^3 f(x) - \nabla^3 f(y)}_{\mathrm{op}} \le \mu_4 \norm{x - y}_2, \quad \text{for all}\; x, y\in\R^d.
}
Then, the vector Laplacian of its gradient is $d \mu_4$-Lipschitz, i.e. 
\eqn{
    \norm{ \Vec{\Delta} (\nabla f)(x) - \Vec{\Delta} (\nabla f)(y) }_2
    \le 
    d \mu_4 \norm{x - y}_2.
}
\end{lemm}
\begin{proof}
Let $g(x) = \Delta (f)(x)$. Since $f \in C^4$, we may switch the order of partial derivatives,
\eqn{
    \norm{\Vec{\Delta} (\nabla f)(x) - \Vec{\Delta} (\nabla f)(y)}_2
    =
    \norm{
        \nabla g(x) - \nabla g(y)
    }_2.
}
By Taylor's theorem with the remainder in integral form,
\begin{align*}
    \norm{
        \nabla g(x) - \nabla g(y)
    }_2
    =& 
    \norm{
            \int_0^1 \nabla^2 g\bracks{
                y + \tau (x - y)
            } (x - y) \dtau 
    }_2 \\
    \le&
        \int_0^1 \normop{
        \nabla^2 g\bracks{
                y + \tau (x - y)
            }
        } \norm{x - y}_2 \dtau \\
    \le&
        \sup_{z \in \R^d} \normop{ \nabla^2 g(z) } \norm{x - y}_2.
\end{align*}
Note that $\nabla^2 g(x)$ can be written as a sum of $d$ matrices, each being a sub-tensor of $\nabla^4 f(x)$, due to the the trace operator, i.e. 
\eqn{
    \nabla^2 g(x) = \sum_{i=1}^d G_i(x), \quad \text{where} \quad G_i(x)_{jk} = \partial_{iijk} f(x).
}
Since the operator norm of $\nabla^4 f(x)$ upper bounds the operator norm of each of its sub-tensor, 
\eqn{
    \normop{\nabla^2 g(x)} \le \sum_{i=1}^d  \normop{G_i(x)} \le d \normop{\nabla^4 f(x)}
}
Recall the third derivative is $\mu_4$-Lipschitz, we obtain
\eqn{
    \norm{
        \nabla g(x) - \nabla g(y)
    }_2
    \le&
        d \mu_3 \norm{x - y}_2.
}
\end{proof}

\section{Estimating the Wasserstein Distance} \label{app:bias_correction}\tabularnewline
For a Borel measure $\mu$ defined on a compact and separable topological space $\mathcal{X}$, a sample-based
empirical measure $\mu_n$ may asymptotically serve as a proxy to $\mu$ in the $W_p$ sense for $p \in [1, \infty)$, i.e.
\eq{
    W_p(\mu, \hat{\mu}_n) \xrightarrow[]{\text{$\mu$-a.s.}} 0.
}
This is a consequence of the Wasserstein distance metrizing weak convergence~\cite{villani2008optimal}
and that the empirical measure converges weakly to $\mu$ almost surely~\cite{varadarajan1958convergence}.

However, in the finite-sample setting, this distance is typically non-negligible and worsens
as the dimensionality increases.
Specifically, generalizing previous results based on the $1$-Wasserstein distance~\cite{dudley1969speed,dobric1995asymptotics},
\citet{weed2017sharp} showed that for $p \in [1, \infty)$,
\eq{
    W_p(\mu, \hat{\mu}_n) \gtrsim n^{-1/t},
}
where $t$ is less than the lower Wasserstein dimension $d_*(\mu)$.
This presents a severe challenge in estimating the $2$-Wasserstein distance between probability measures using samples.

To better detect convergence,
we zero center a simple sample-based estimator by subtracting
the null responses and obtain the following new estimator:
\begin{align*}
\tilde{W}_2^2(\mu, \nu) =&
\frac{1}{2} \bracks{
    W_2^2(\hat{\mu}_n, \hat{\nu}_n) +
    W_2^2(\hat{\mu}_n', \hat{\nu}_n') -
    W_2^2(\hat{\mu}_n, \hat{\mu}_n') -
    W_2^2(\hat{\nu}_n, \hat{\nu}_n')
},
\end{align*}
where $\hat{\nu}_n$ and $\hat{\nu}_n'$ are based on two independent samples of size $n$ from $\mu$,
and similarly for $\hat{\nu}_n$ and $\hat{\nu}_n'$ from $\nu$.
This estimator is inspired by the contruction of distances in the maximum mean discrepancy family~\cite{gretton2012kernel} and the Sinkhorn divergence~\cite{peyre2019computational}.
Note that the $2$-Wasserstein distance between finite samples can be computed conveniently with existing packages~\cite{flamary2017pot} that solves a linear program.
Although the new estimator is not guaranteed to be unbiased across all settings, it is unbiased when the two distributions are the same.

Since our correction is based on a heuristic, the new estimator is still biased.
To empirically characterize the effectiveness of the correction, we compute the discrepancy between the squared $2$-Wasserstein distance for two continuous densities and the finite-sample estimate obtained from $\mathrm{i.i.d.}$ samples. When $\mu$ and $\nu$ are Gaussians with means $m_1, m_2\in\R^d$ and covariance matrices $\Sigma_1,\Sigma_2\in\R^{d\times d}$, we have the following convenient closed-form
\begin{align}
    W_2^2(\mu, \nu) =
        \normtwo{m_1 - m_2}^2 +
        \Tr{\Sigma_1+\Sigma_2 - 2(\Sigma_1^{1/2}\Sigma_2\Sigma_1^{1/2})^{1/2} }.
\end{align}

\begin{figure}[ht]
\centering
\begin{minipage}[t]{0.45\linewidth}
\centering
{\includegraphics[width=0.98\textwidth]{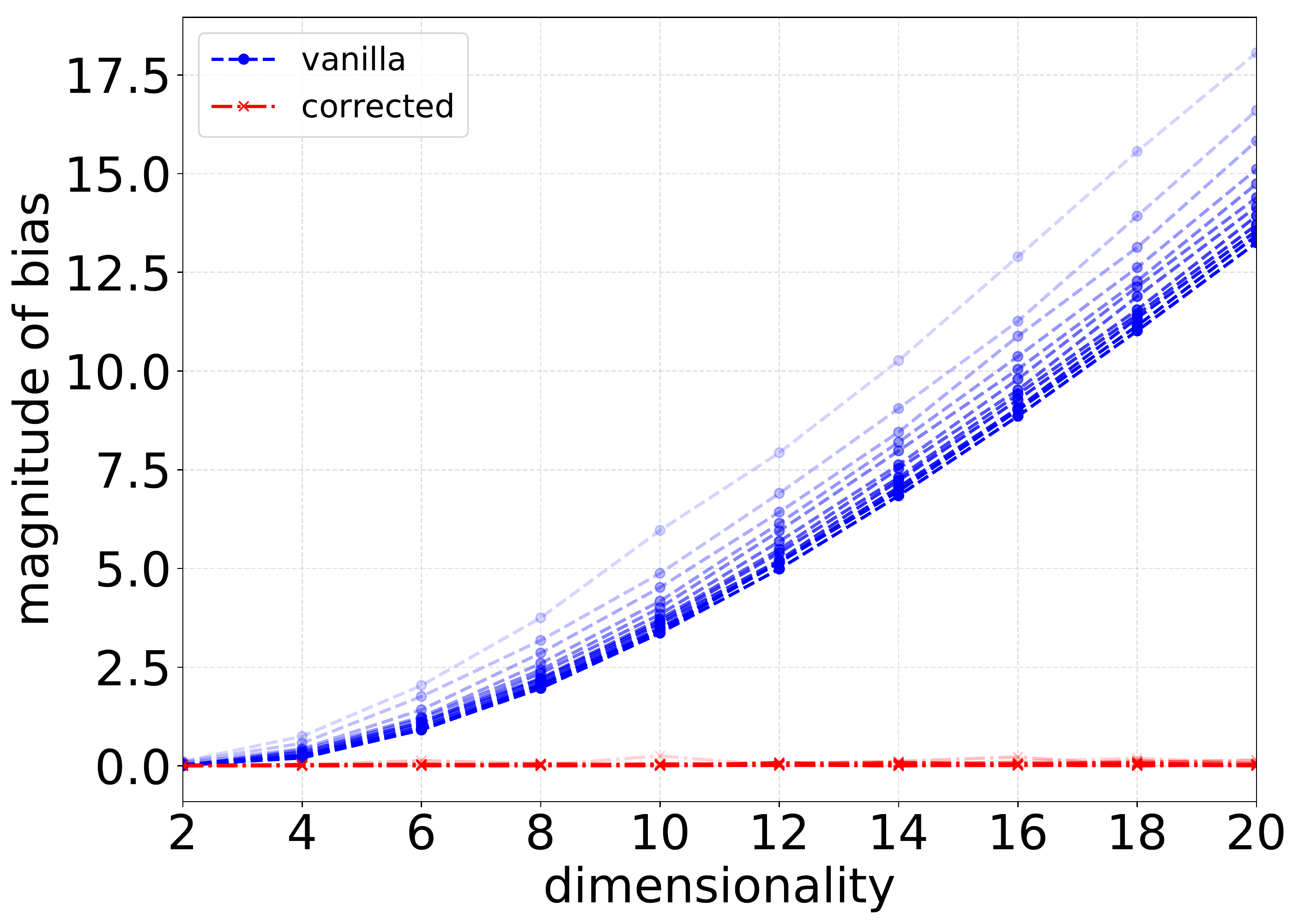}}
(a) different in mean
\end{minipage}
\begin{minipage}[t]{0.45\linewidth}
\centering
{\includegraphics[width=0.98\textwidth]{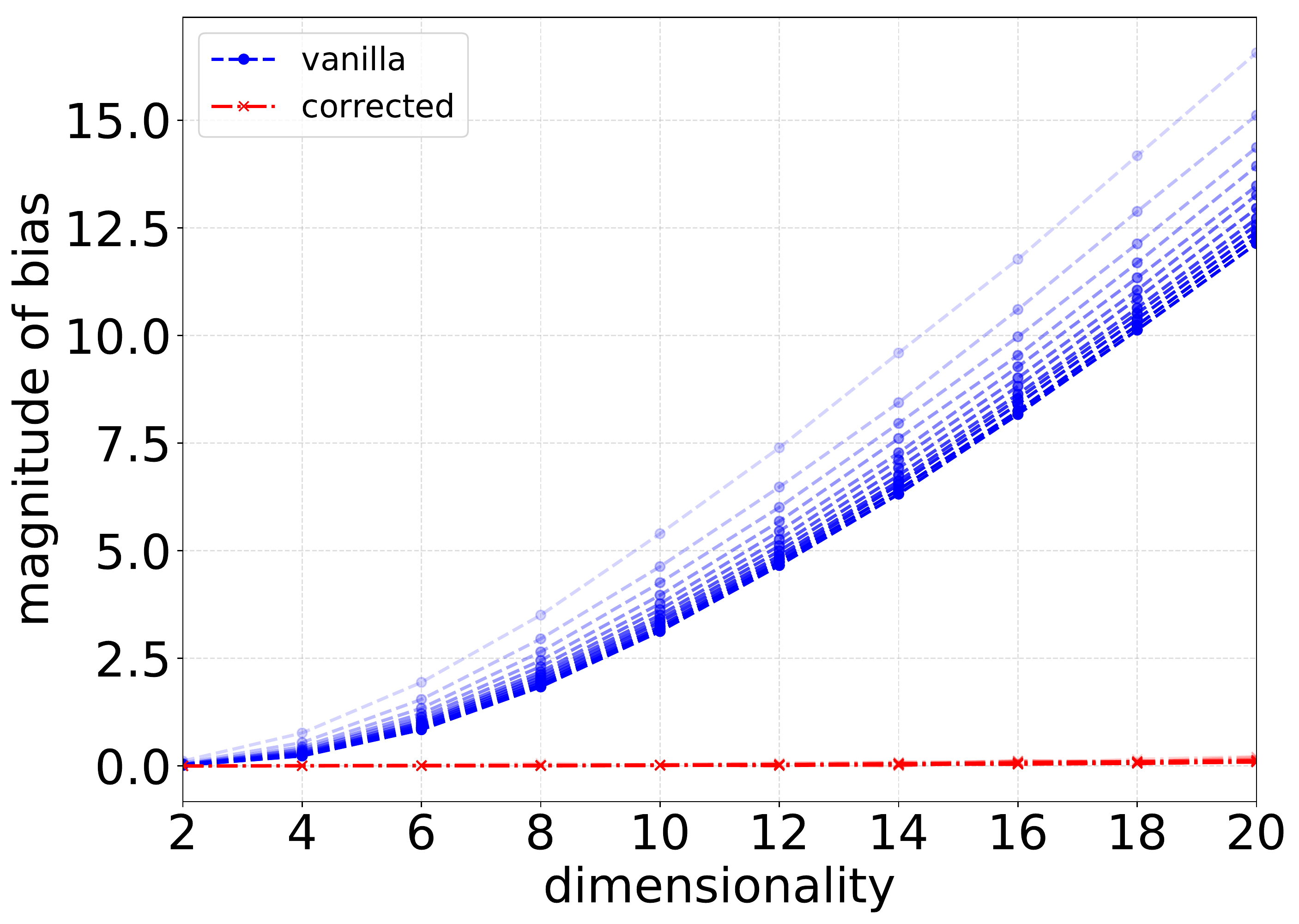}}
(b) different in covariance
\end{minipage}
 \caption{Absolute value between $W_2^2(\mu,\nu)$ and the sample averages of estimators $\hat{W}_2^2$ (vanilla) and $\tilde{W}_2^2$ (corrected) for Gaussian $\mu$ and $\nu$. Darker curves correspond to larger number of samples used to compute the empirical estimate (ranging from 100 to 1000).
 (a) $m_1=0, m_2=\mathbf{1}_d, \Sigma_1=\Sigma_2 = I_d$.
 (b) $m_1=m_2=0, \Sigma_1=I_d, \Sigma_2=I_d/2 + \mathbf{1}_d \mathbf{1}_d^\top/5$.}
\label{fig:W2}
\end{figure}

We compare the vanilla estimate $\hat{W}_2^2(\mu,\nu, n)$ and the corrected estimate $\tilde{W}_2^2(\mu,\nu, n)$ by their magnitude of deviation from the true value $W_2^2(\mu, \nu)$:
\begin{align*}
\LL|{W_2^2(\mu,\nu) - \E[\hat{W}_2^2(\mu, \nu, n)]}\RR|, \quad
\LL|W_2^2(\mu, \nu) - \E[\tilde{W}_2^2(\mu, \nu, n)]\RR|,
\end{align*}
where the expectations are approximated via averaging 100 independent draws.
Figure~\ref{fig:W2} reports the deviation across different sample sizes and dimensionalities, where $\mu$ and $\nu$ differ only in either mean or covariance. While the corrected estimator is not unbiased, it is relatively more accurate. 

\begin{figure}[ht]
\centering
\begin{minipage}[t]{0.45\linewidth}
\centering
{\includegraphics[width=0.98\textwidth]{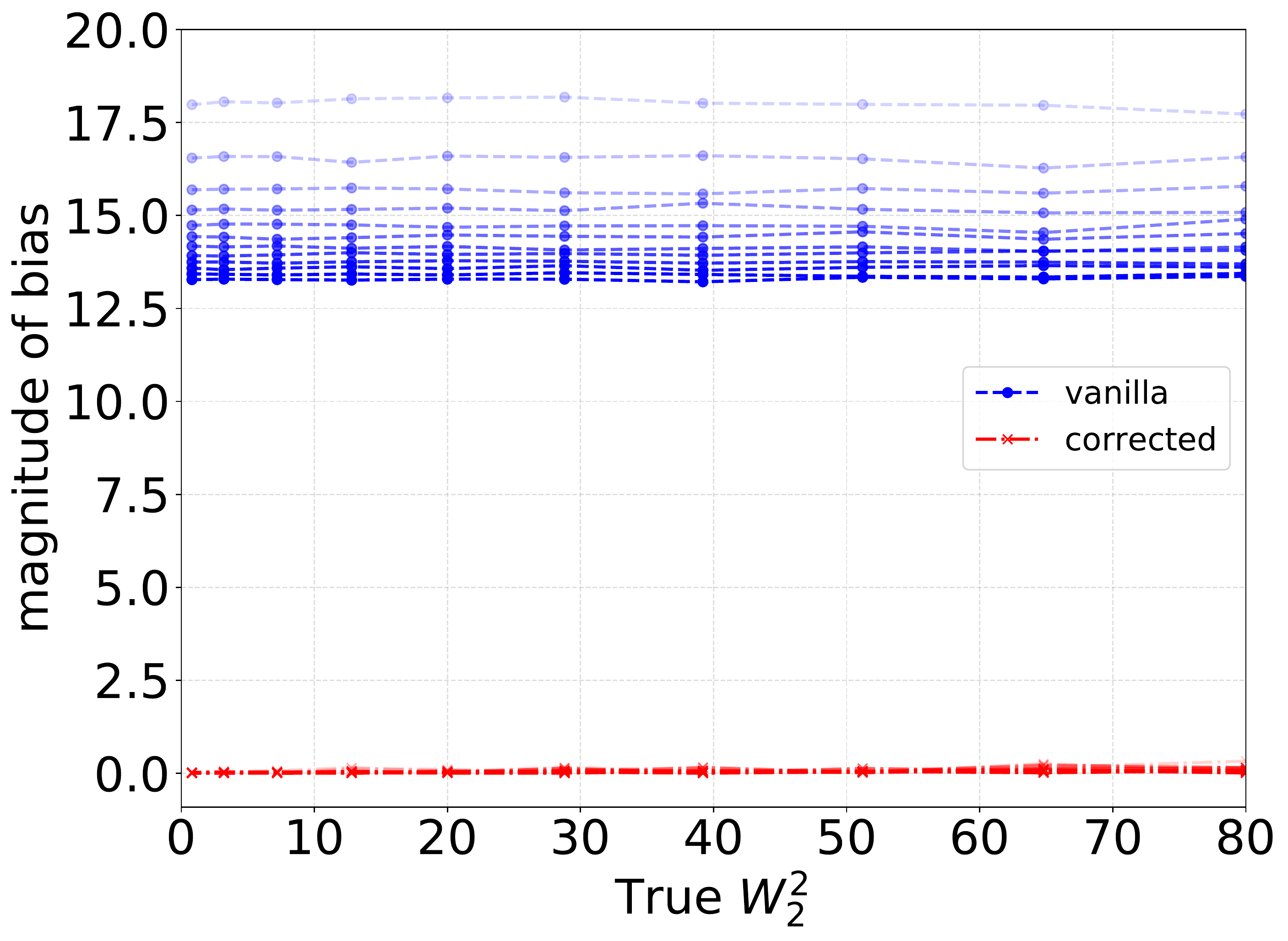}}
(a) different in mean
\end{minipage}
\begin{minipage}[t]{0.45\linewidth}
\centering
{\includegraphics[width=0.98\textwidth]{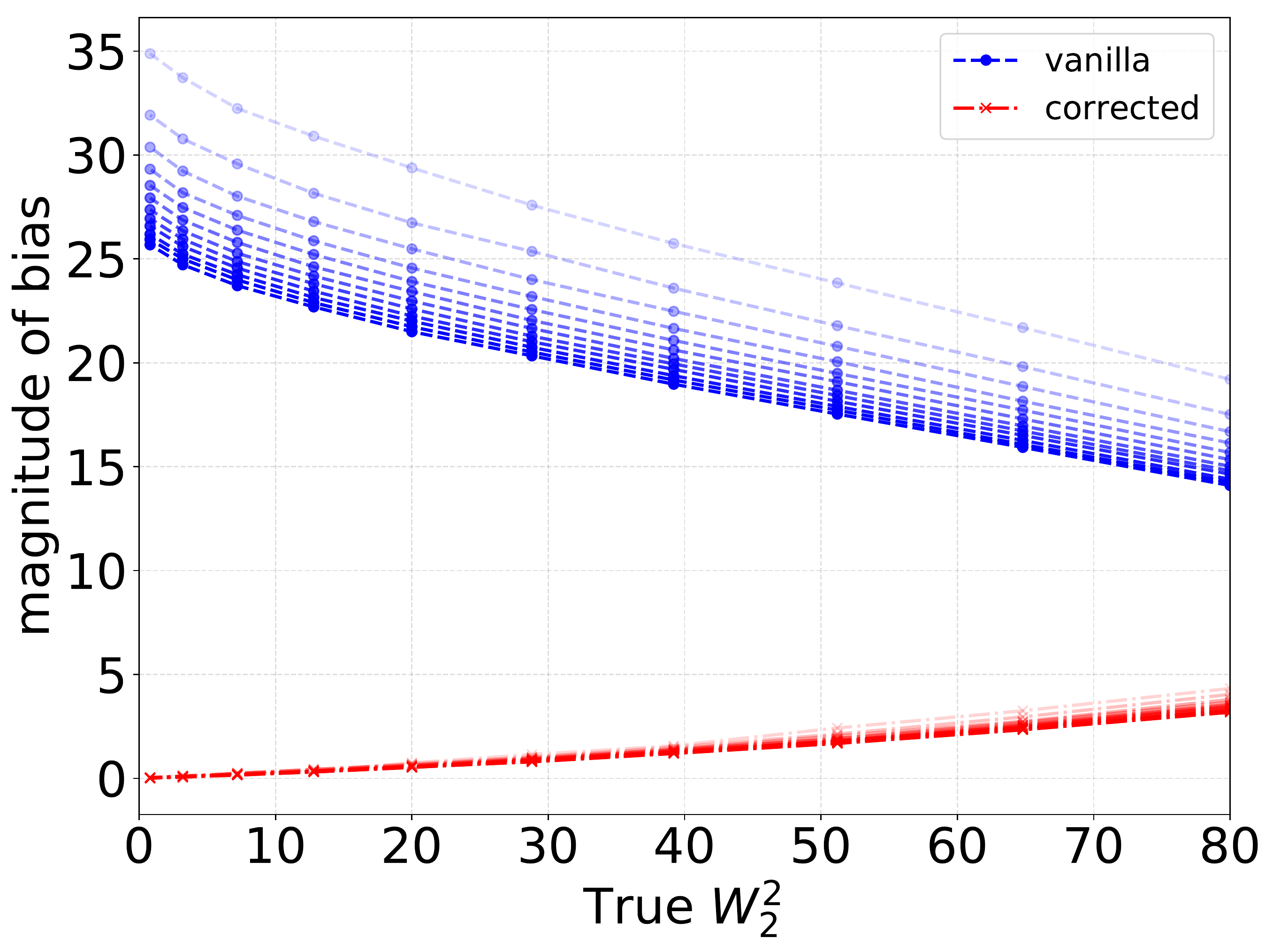}}
(b) different in covariance
\end{minipage}
 \caption{Absolute value between $W_2^2(\mu,\nu)$ and the sample averages of estimators $\hat{W}_2^2$ (vanilla) and $\tilde{W}_2^2$ (corrected) for Gaussian $\mu$ and $\nu$. Darker curves correspond to larger number of samples used to compute the empirical estimate (ranging from 100 to 1000). We fix $d=20$ and interpolate the mean and the covariance matrix, i.e. $m = \alpha m_1 + (1-\alpha) m_2, \Sigma = \alpha\Sigma_1 + (1-\alpha) \Sigma_2, \alpha\in[0,1]$.
 (a) $m_1=0, m_2=2\mathbf{1}_d, \Sigma_1=\Sigma_2 = I_d$.
 (b) $m_1=m_2=0, \Sigma_1=2I_d, \Sigma_2=I_d/2 + \mathbf{1}_d \mathbf{1}_d^\top/5$.}
\label{fig:W2-interp}
\end{figure}

In addition, Figure~\ref{fig:W2-interp} demonstrates that our bias-corrected estimator becomes more accurate as the two distributions are closer. This indicates that our proposed estimator may provide a more reliable estimate of the 2-Wasserstein distance when the sampling algorithm is close to convergence.

\section{Additional Numerical Studies} \label{app:additional_exp}
In this section, we include additional numerical studies complementing Section~\ref{sec:exp}.

\subsection{Strongly Convex Potentials}
We first include additional plots of error estimates in $W_2$ and the energy distance for sampling from a Gaussian mixture and the posterior of BLR. 
The results indicate that the reduction in asymptotic error is consistent across problems with varying dimensionalities that we consider. 
In the end, we conduct a wall time analysis and show that SRK-LD is competitive in practice.

\subsubsection{Additional Results}
Figure~\ref{fig:strongly_convex_additional_W2} shows the estimated $W_2$ error as the number of iterations increase for the 2D and 20D Gaussian mixture and BLR problems with the parameter settings described in Section~\ref{sec:exp}.
We observe consistent improvement in the asymptotic error across different settings in which we experimented.
\begin{figure}[ht]
\centering
\begin{minipage}[t]{0.325\linewidth}
\centering
{\includegraphics[width=0.98\textwidth]{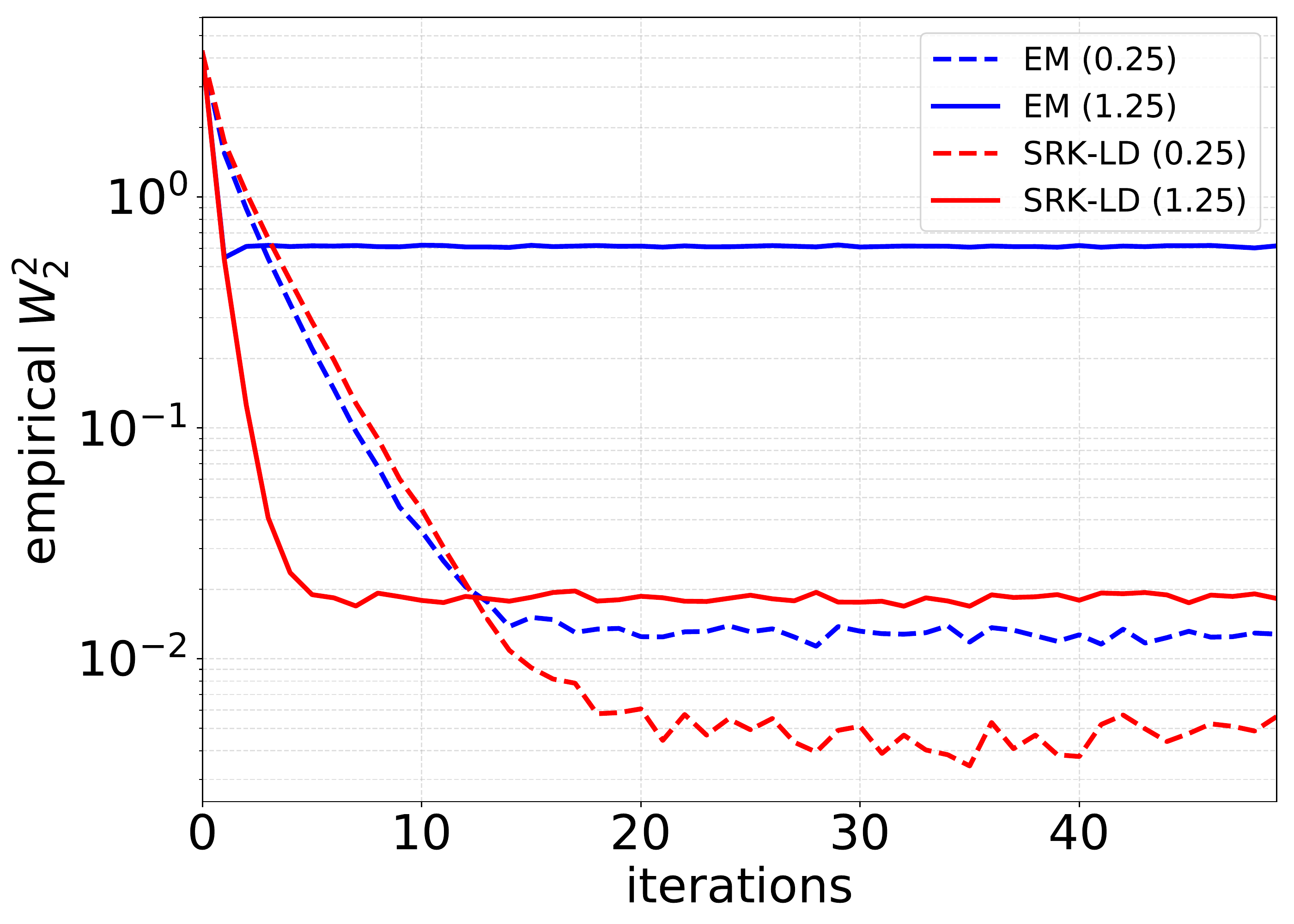}}
(a) Gaussian Mixture (2D)
\end{minipage}
\begin{minipage}[t]{0.325\linewidth}
\centering
{\includegraphics[width=0.98\textwidth]{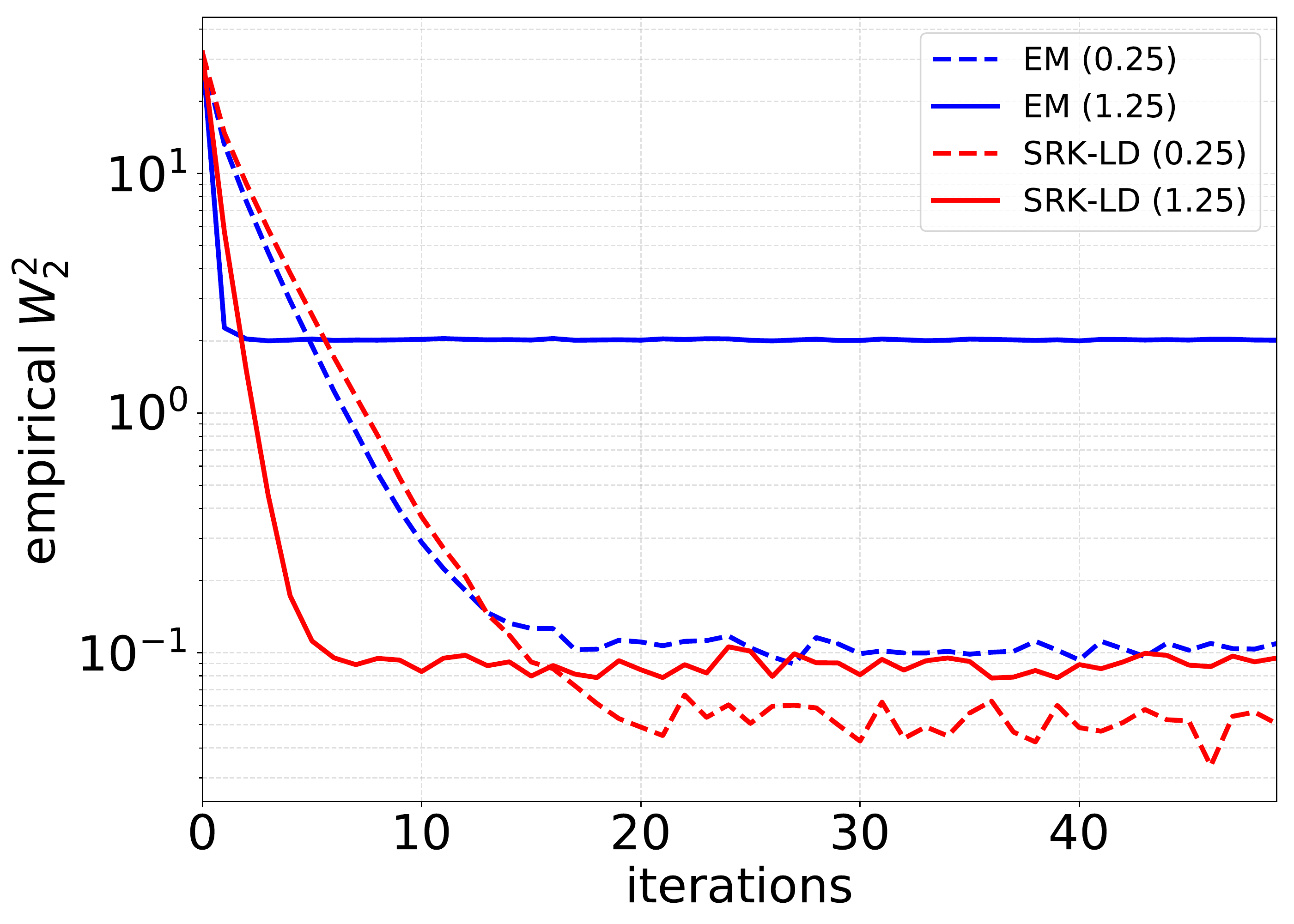}}
(b) Gaussian Mixture (20D)
\end{minipage}
\begin{minipage}[t]{0.325\linewidth}
\centering
{\includegraphics[width=0.98\textwidth]{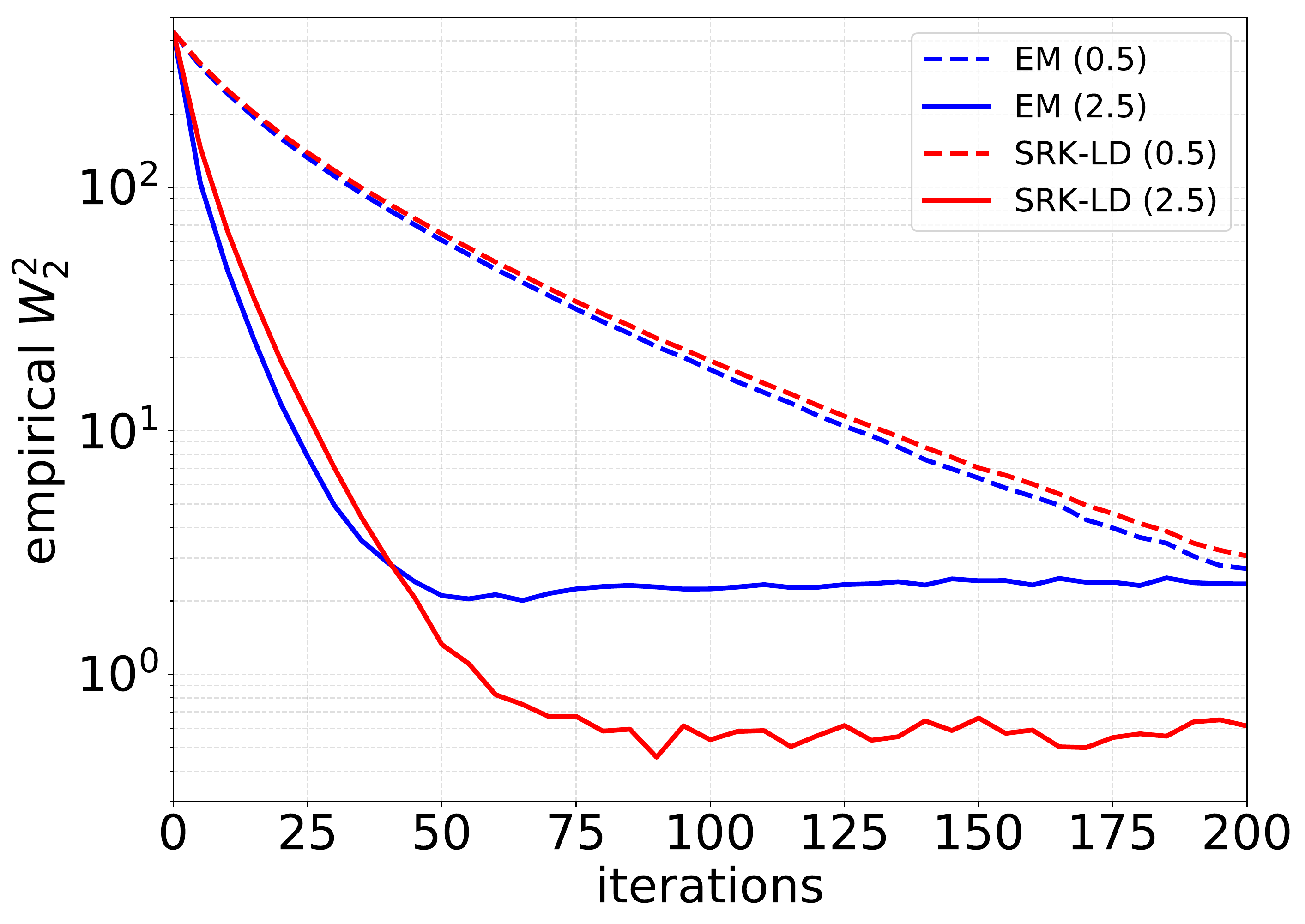}}
(c) BLR (20D)
\end{minipage}
\caption{Error in $W_2^2$ for strongly log-concave sampling. Legend denotes``scheme (step size)''. }
\label{fig:strongly_convex_additional_W2}
\end{figure}

In addition to reporting the estimated squared $W_2$ values, we also evaluate the two schemes by estimating the energy distance~\cite{szekely2003statistics,szekely2013energy} under the Euclidean norm. For probability measures $\mu$ and $\nu$ on $\R^d$ with finite first moments, this distance is defined to be the square root of  
\eqn{
    D_E(\mu, \nu)^2 =
        2 \Exp{ \normtwo{Y - Z} } - 
        \Exp{ \normtwo{Y - Y'} } -
        \Exp{ \normtwo{Z - Z'} }, \numberthis \label{eq:energy_distance}
}
where $Y, Y' \overset{\mathrm{i.i.d.}}{\sim} \mu$ and $Z, Z' \overset{\mathrm{i.i.d.}}{\sim} \nu$.
The moment condition is required to ensure that the expectations in~\eqref{eq:energy_distance} is finite. 
This holds in our settings due to derived moment bounds. 
Since exactly computing the energy distance is intractable, we estimate the quantity using the following (biased) V-statistic~\cite{sejdinovic2013equivalence}
\eq{
    \hat{D}_E(\mu, \nu)^2 = 
        \frac{2}{mn} \sum_{i=1}^m \sum_{j=1}^n \normtwo{Y_i - Z_j}
        - \frac{1}{m^2} \sum_{i=1}^m \sum_{j=1}^m \normtwo{Y_i - Y_j}
        - \frac{1}{n^2} \sum_{i=1}^n \sum_{j=1}^n \normtwo{Z_i - Z_j}
    ,
}
where $Y_i \overset{\mathrm{i.i.d.}}{\sim} \mu $ for $i=1, \dots, m$ and $Z_j \overset{\mathrm{i.i.d.}}{\sim} \nu $ for $j=1,\dots, n$.
Figure~\ref{fig:strongly_convex_additional_ed} shows the estimated energy distance as the number of iteration increases on a semi-log scale. 
We use 5k samples each for the Markov chain and the target distribution to compute the V-statistic, where the target distribution is approximated following the same procedure as described in Section~\ref{subsec:strongly_convex_exp}. 
These plots show that SRK-LD achieves lower asymptotic errors compared to the EM scheme, where the error is measured in the energy distance. 
This is consistent with the case where the error is estimated in $W_2^2$.

\begin{figure}[ht]
\centering
\begin{minipage}[t]{0.45\linewidth}
\centering
{\includegraphics[width=0.98\textwidth]{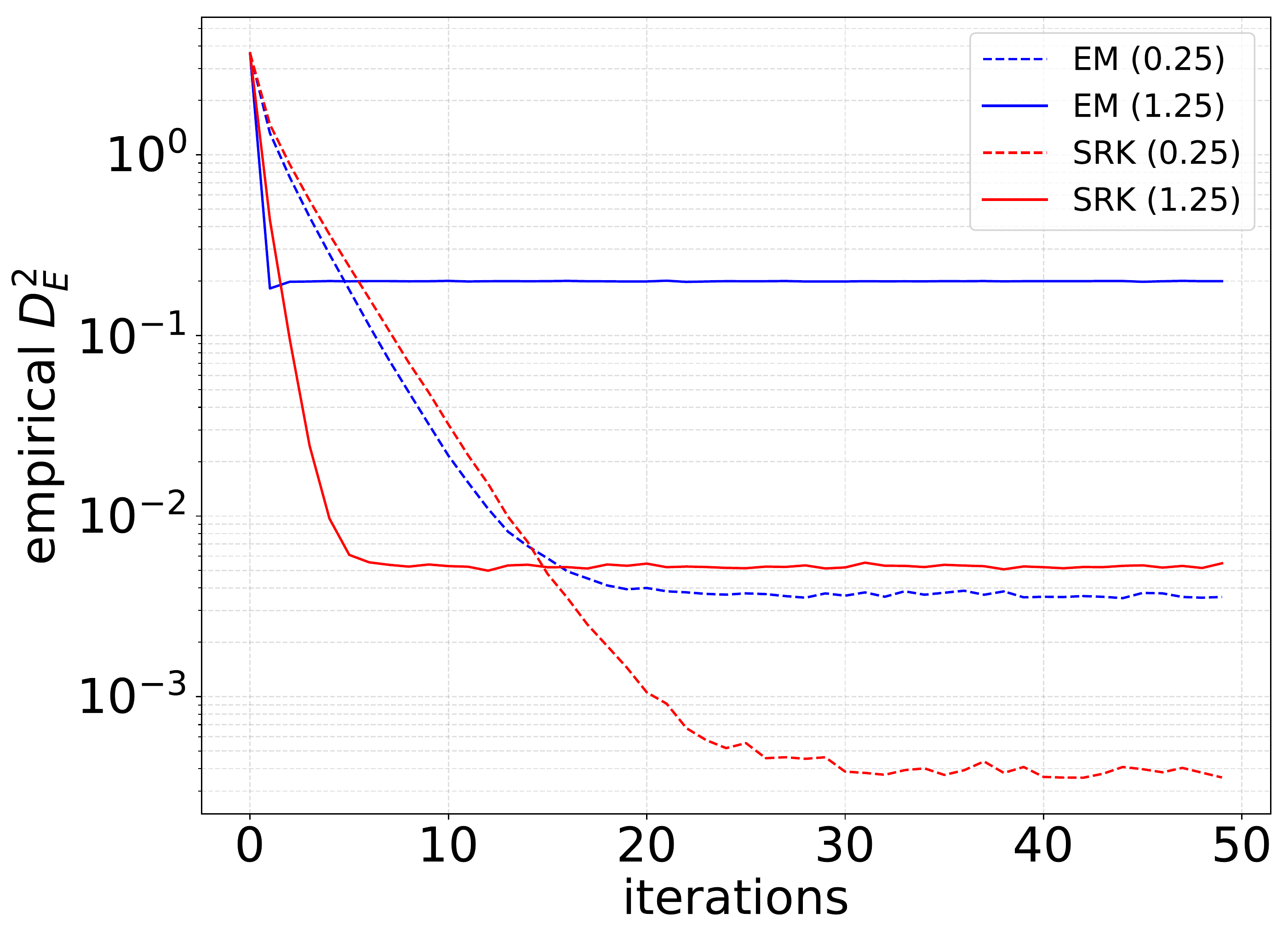}}
(a) Gaussian Mixture (2D)
\end{minipage}
\begin{minipage}[t]{0.45\linewidth}
\centering
{\includegraphics[width=0.98\textwidth]{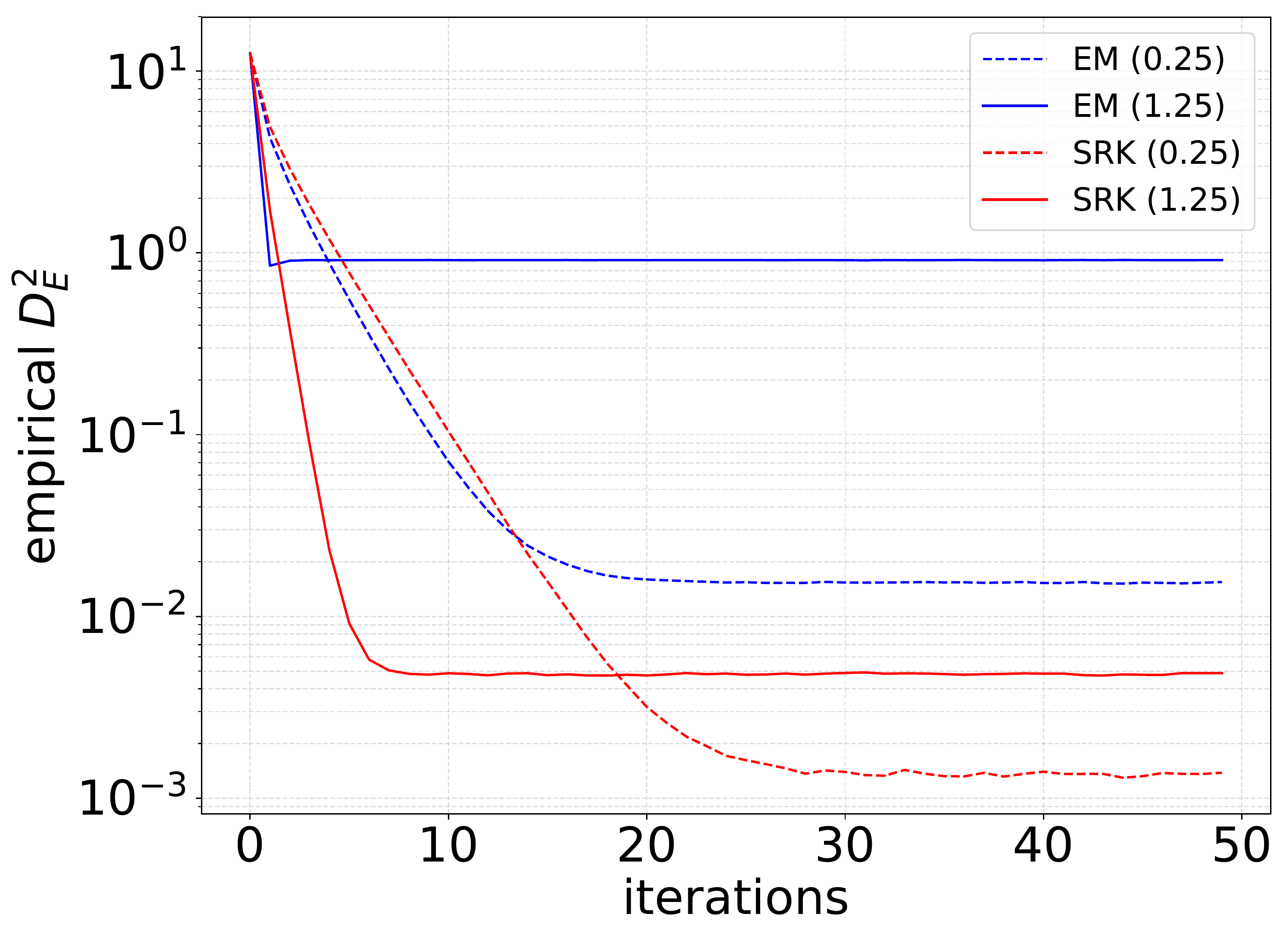}}
(b) Gaussian Mixture (20D)
\end{minipage}
\begin{minipage}[t]{0.45\linewidth}
\centering
{\includegraphics[width=0.98\textwidth]{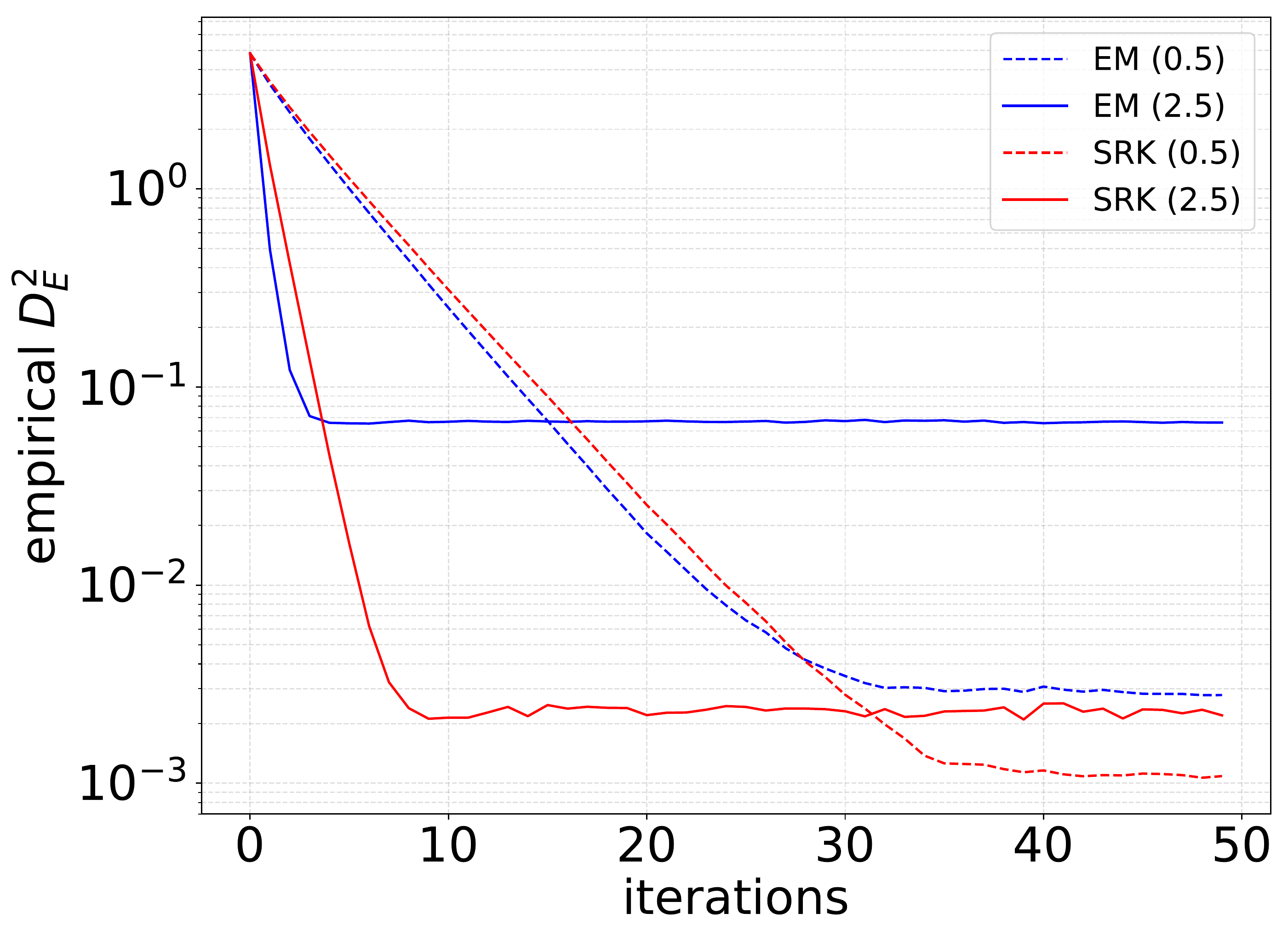}}
(c) BLR (2D)
\end{minipage}
\begin{minipage}[t]{0.45\linewidth}
\centering
{\includegraphics[width=0.98\textwidth]{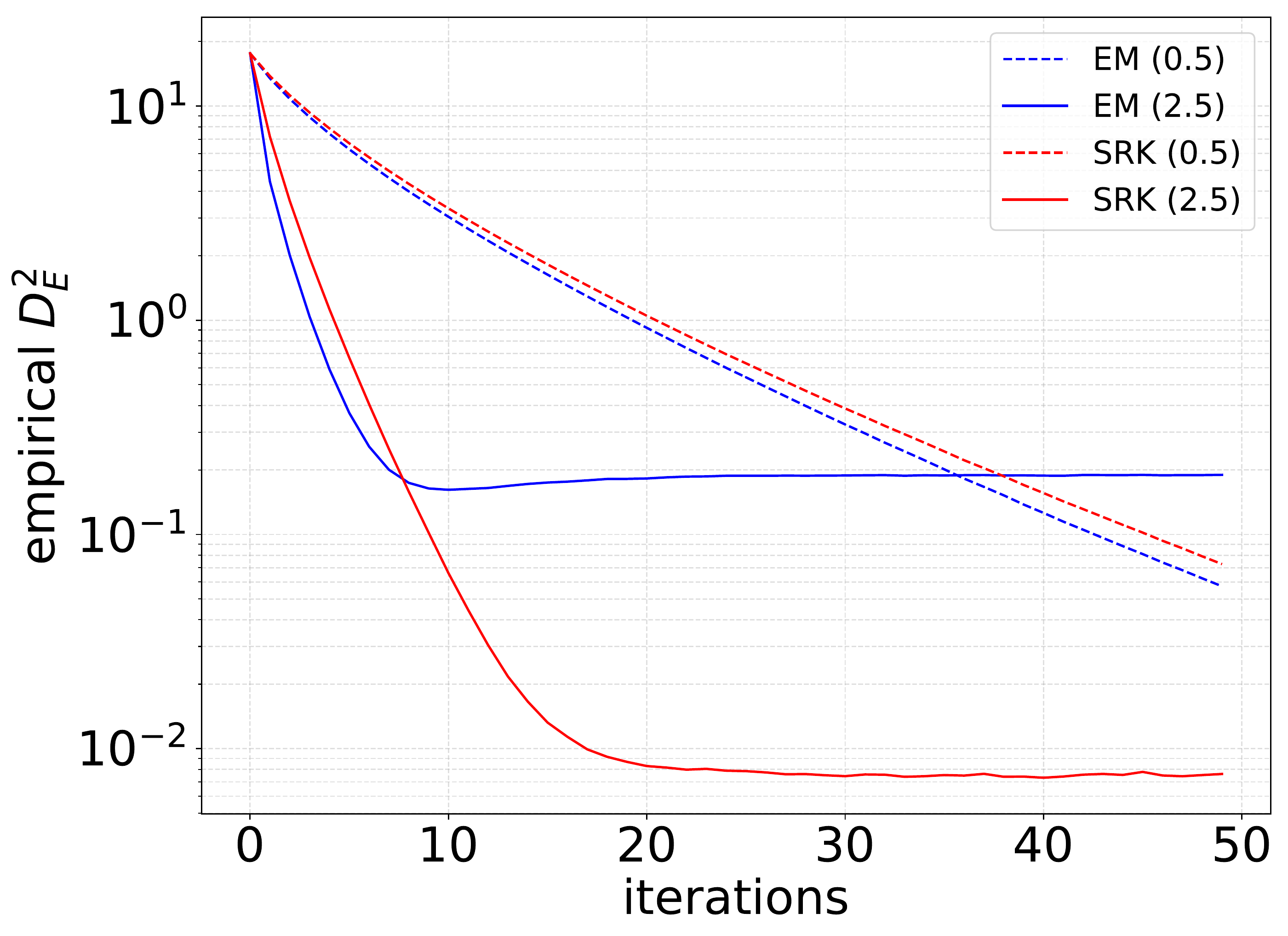}}
(d) BLR (20D)
\end{minipage}
\caption{Error in $D_E^2$ for strongly log-concave sampling. Legend denotes``scheme (step size)''.}
\label{fig:strongly_convex_additional_ed}
\end{figure}

\subsubsection{Asymptotic Error vs Dimensionality and Step Size}\begin{figure}[ht]
\centering
\begin{minipage}[t]{0.45\linewidth}
\centering
{\includegraphics[width=0.98\textwidth]{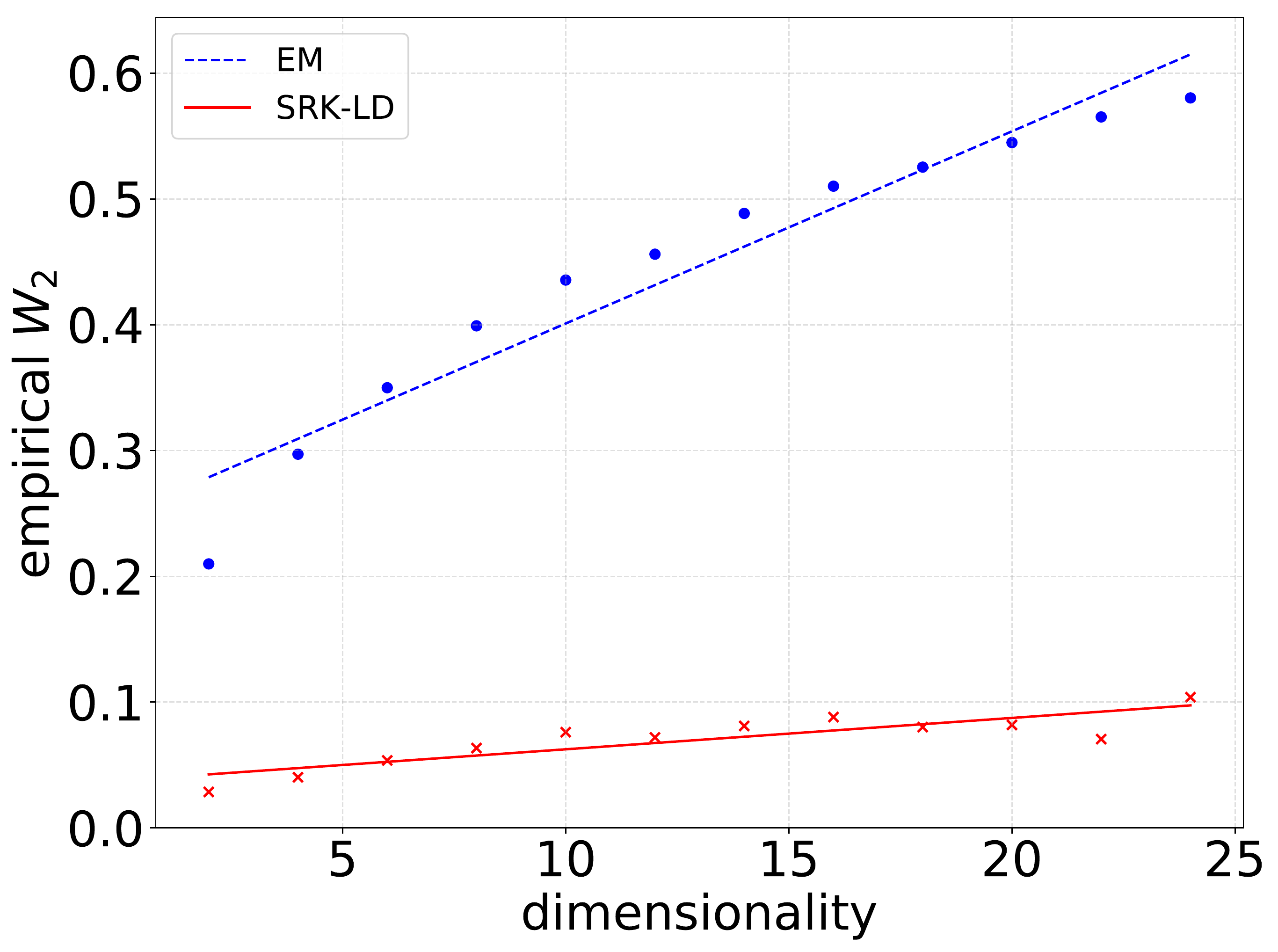}}
(a) dimensionality
\end{minipage}
\begin{minipage}[t]{0.45\linewidth}
\centering
{\includegraphics[width=0.98\textwidth]{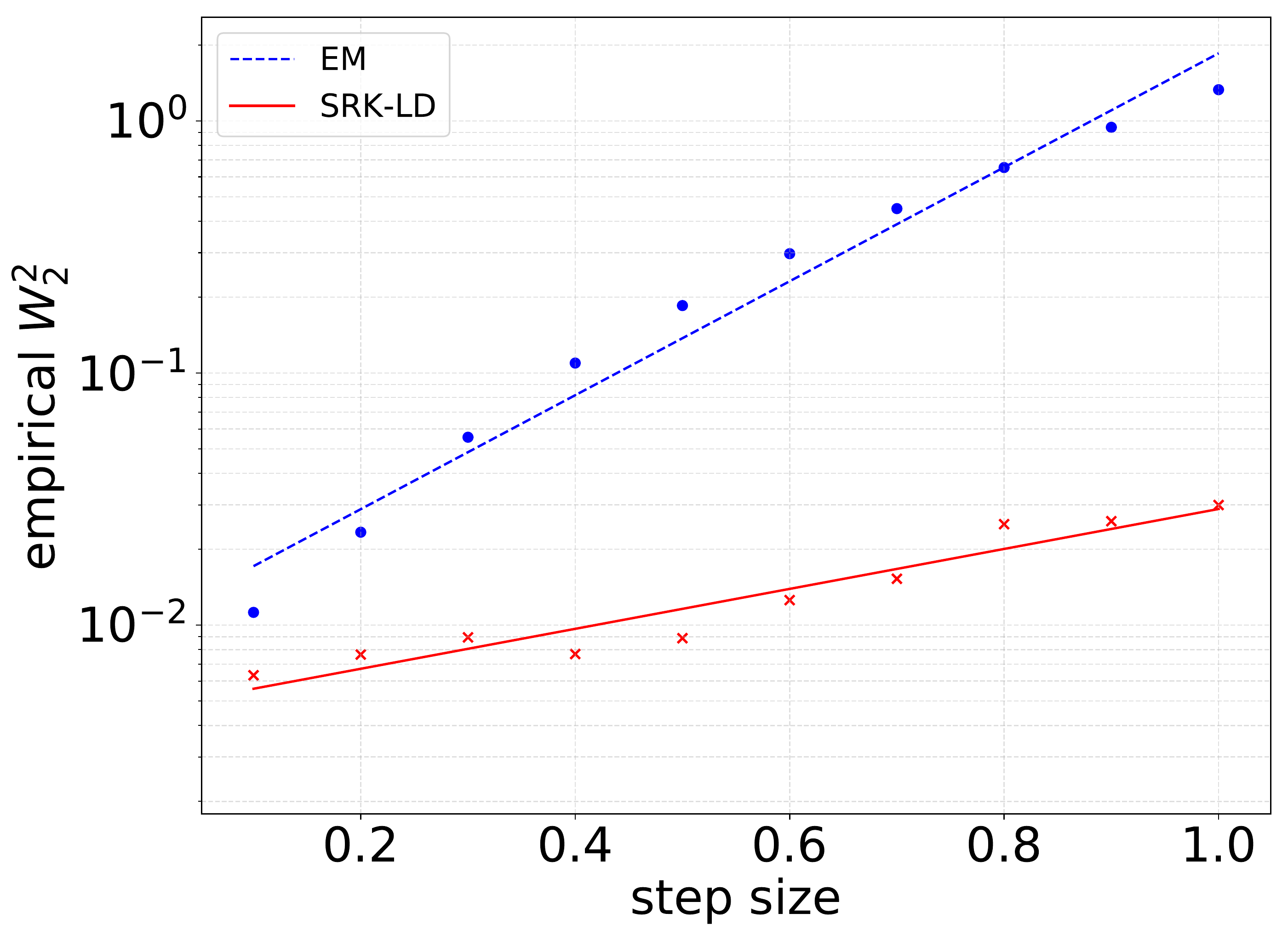}}
(b) step size
\end{minipage}
\caption{Asymptotic error vs dimensionality and step size.}
\label{fig:asymptotic_error}
\end{figure}
Figure~\ref{fig:asymptotic_error} (a) and (b) respectively show the asymptotic error against dimensionality and step size for Gaussian mixture sampling. We perform least squares regression in both plots. 
Plot (a) shows results when a step size of $0.5$ is used. 
Plot (b) is on semi-log scale, where the quantities are estimated for a 10D problem.

\subsubsection{Wall Time}
Figure~\ref{fig:wall_time} shows the wall time against the estimated $W_2^2$ of SRK-LD compared to the EM scheme 
for a 20D Gaussian mixture sampling problem.
On a 6-core CPU with 2 threads per core, we observe that SRK-LD is roughly $\times$ 2.5 times as costly as EM per iteration. 
However, since SRK-LD is more stable for large step sizes, we may choose a step size much larger for SRK-LD compared to 
EM, in which case its iterates converge to a lower error within less time. 
\begin{figure}[ht]
\centering
\begin{minipage}[t]{0.45\linewidth}
\centering
{\includegraphics[width=0.98\textwidth]{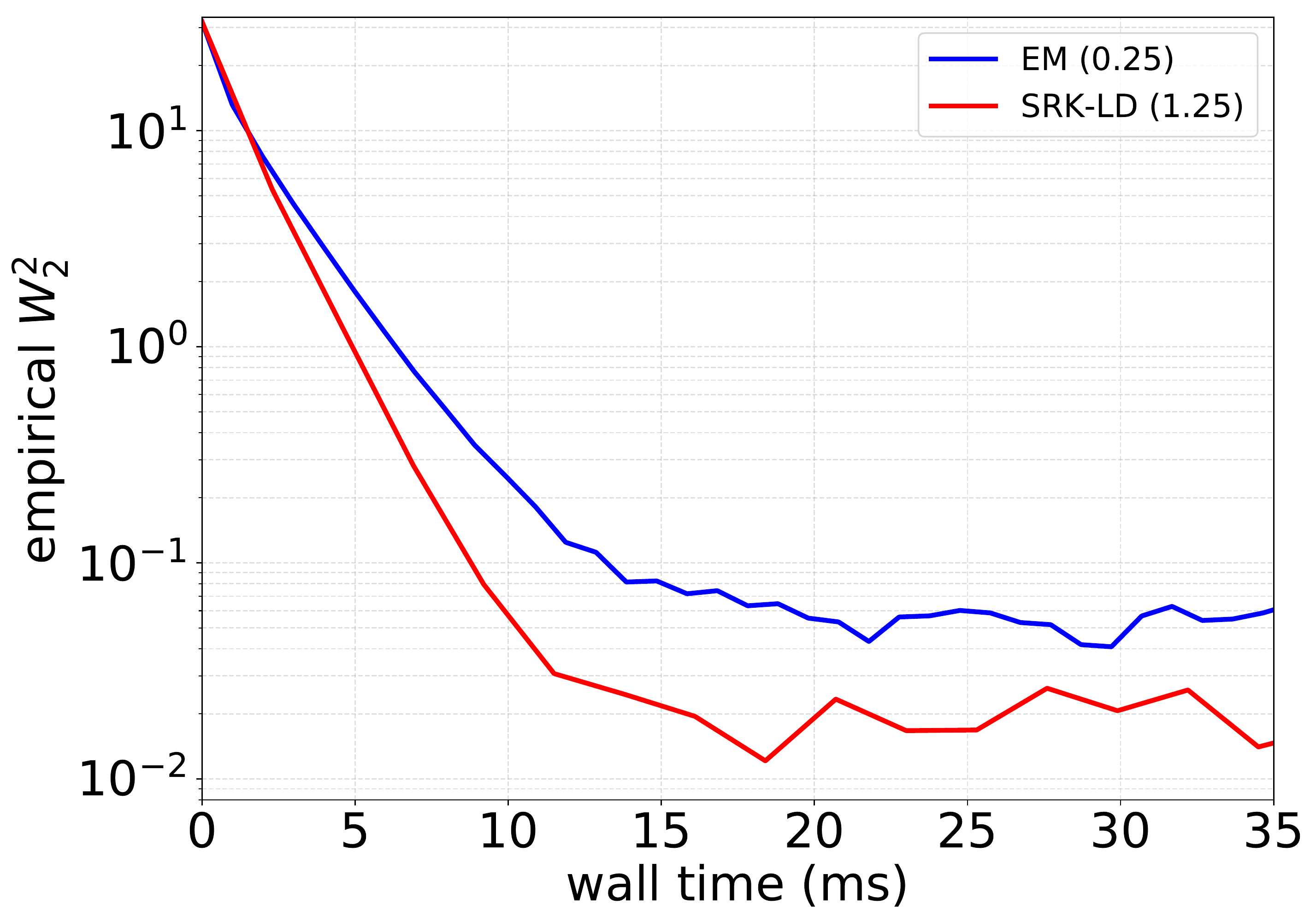}}
\end{minipage}
\caption{Wall time for sampling from a 20D Gaussian mixture.}
\label{fig:wall_time}
\end{figure}

\subsection{Non-Convex Potentials} \label{app:srk-id}
We first discuss how we approximate the iterated It\^o integrals, after which we include additional numerical studies varying the dimensionality of the sampling problem.

\subsubsection{Approximating Iterated It\^o Integrals}
Simulating both the iterated It\^o integrals $I_{(l,i)}$ and the Brownian motion increments $I_{(i)}$ exactly is difficult. We adopt the Kloeden-Platen-Wright approximation, which
has an MSE of order $h^2/n$, where $n$ is the number of terms in the truncation~\cite{kloeden1992approximation}. 
The infinite series can be written as follows:
\eq{
    I_{(l, i)} &= 
        \frac{ I_{(l)} I_{(i)} - h \delta_{li}}{2} + A_{(l, i)}, \\
    A_{(l, i)} &= \frac{h}{2\pi} \sum_{k=1}^\infty \frac{1}{k} \bracks{
        \xi_{l, k} \bracks{ \eta_{i, k} + \sqrt{2/h} \Delta B_h^{(i)} } -
        \xi_{i, k} \bracks{ \eta_{l, k} + \sqrt{2/h} \Delta B_h^{(l)} }
    },
}
where $\xi_{l, k}, \xi_{i, k}, \eta_{i, k}, \eta_{l, k} \overset{\text{i.i.d.}}{\sim} \mathcal{N}(0, 1)$.
$A_{(l,i)}$ is known as the \textit{L\'evy area} and is notoriously hard to simulate~\cite{wiktorsson2001joint}. 

For SDE simulation, in order for the scheme to obtain the same
strong convergence order under the approximation, the MSE in
the approximation of the It\^o integrals must be negligible compared to the local mean-square deviation of the numerical integration scheme. 
For our experiments, we use $n=3000$, following the rule of thumb that $n \propto h^{-1}$~\cite{kloeden1992approximation}.
Although simulating the extra terms can become costly, the computation may be vectorized, branched off from the main update, and parallelized on an additional thread, since it does not require any information of the current iterate.

\citet{wiktorsson2001joint} proposed to add a correction term to the truncated series, which results in an approximation that has an MSE of order $h^2/n^2$. In this case, $n \propto h^{-1/2}$ terms are effectively required. We note that analyzing and comparing between different L\'evy area approximations is beyond the scope of this paper.

\subsubsection{Additional Results} \label{app:srk-id_diminish}
Figure~\ref{fig:srk_id_dims} shows the MSE of simulations starting from a faithful approximation to the target. 
We adopt the same simulation settings as described in Section~\ref{subsec:exp_nonconvex}.
We observe diminishing gains as the dimensionality increases across all settings with differing
$\beta$ and $\gamma$ parameters in which we experimented. 
These empirical findings corroborate our theoretical results.
Note that the corresponding diffusion in all settings are still uniformly dissipativity, yet the potential may become convex when $\beta$ is large. 
Nevertheless, the potential is never strongly convex when $\beta$ is positive due to the linear growth term.
\begin{figure}[ht]
\centering
\begin{minipage}[t]{0.45\linewidth}
\centering
{\includegraphics[width=0.98\textwidth]{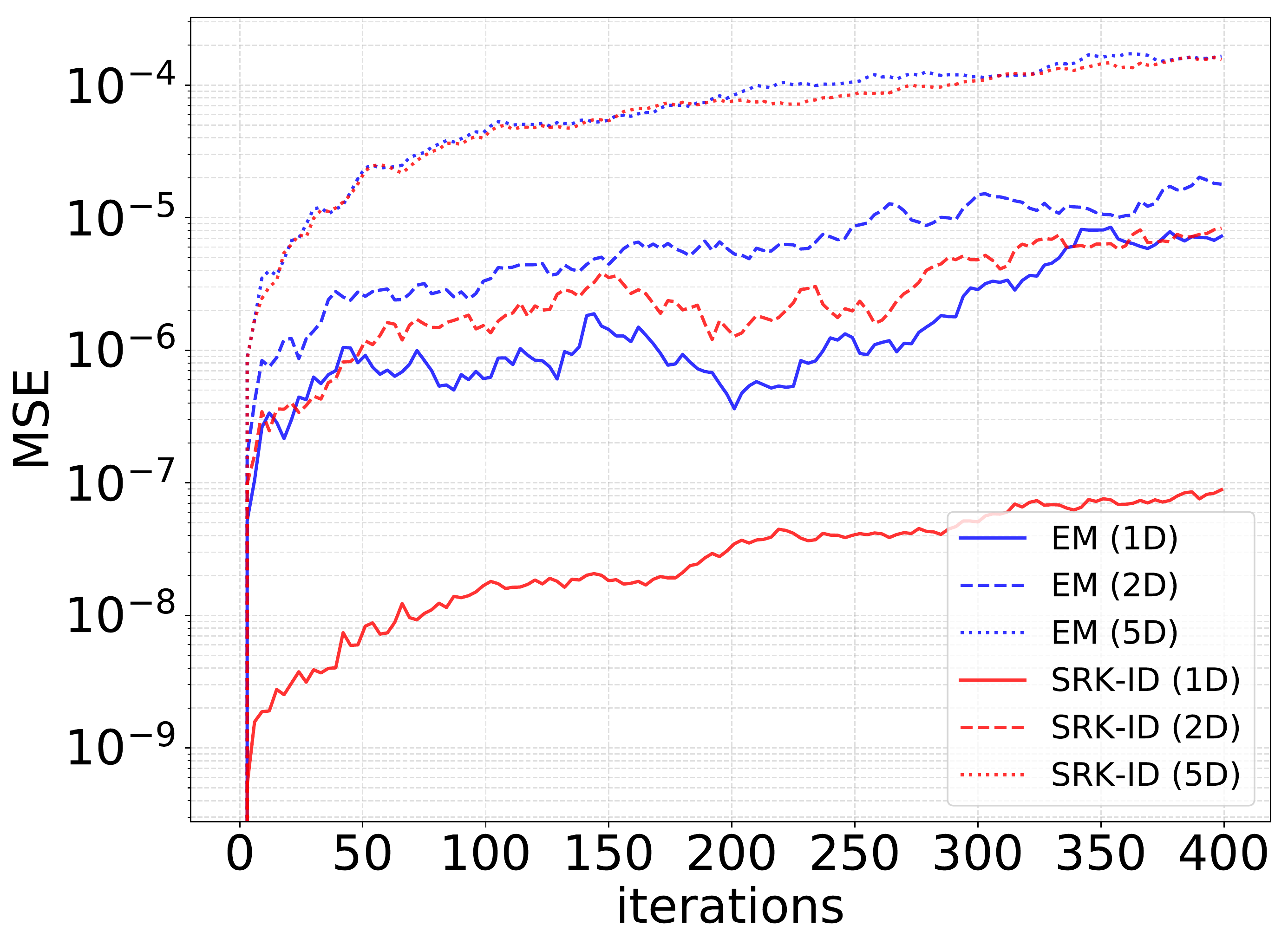}}
(a) $ \beta = 2, \gamma = 0.5$
\end{minipage}
\begin{minipage}[t]{0.45\linewidth}
\centering
{\includegraphics[width=0.98\textwidth]{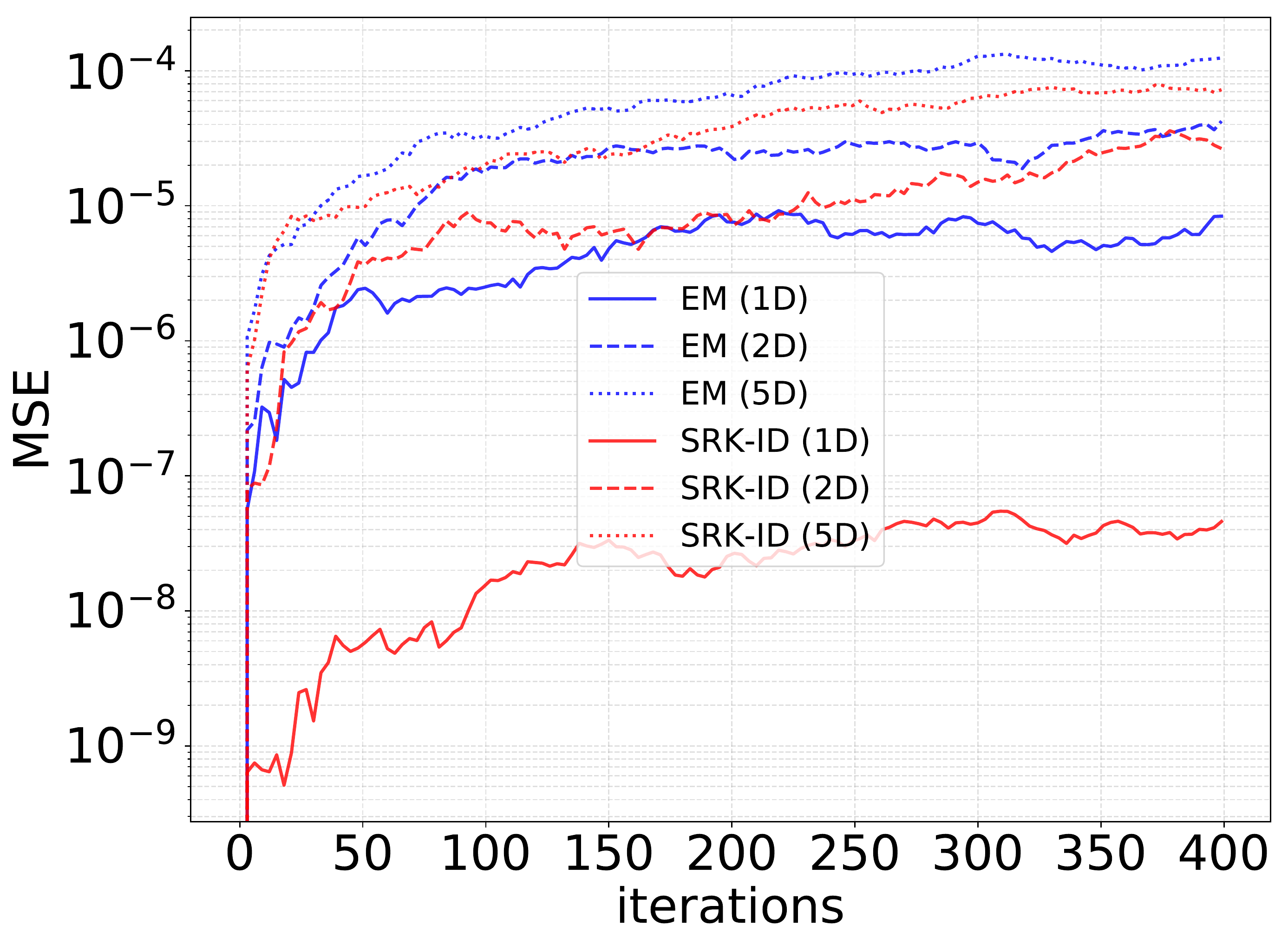}}
(b) $\beta = 3, \gamma = 0.5$
\end{minipage}
\caption{MSE for non-convex sampling.}
\label{fig:srk_id_dims}
\end{figure}

\end{document}